\documentclass{article}

\usepackage{microtype}
\usepackage{graphicx}
\usepackage{subcaption}
\usepackage{booktabs} % for professional tables

\usepackage{hyperref}

\usepackage{placeins}

\usepackage[accepted]{icml2026}

\usepackage{amsmath}
\usepackage{amssymb}
\usepackage{mathtools}
\usepackage{amsthm}

\usepackage[capitalize,noabbrev]{cleveref}

\theoremstyle{plain}
\newtheorem{theorem}{Theorem}[section]

\newtheorem{lemma}[theorem]{Lemma}

\theoremstyle{definition}

\newtheorem{assumption}[theorem]{Assumption}
\theoremstyle{remark}

\usepackage{hyperref}      % For hyperlinks
\usepackage{etoolbox}      % For \whileboolexpr, \ifnumless, etc.
\usepackage{amsmath}       % For math environments
\usepackage{amssymb}       % For math symbols

\usepackage{cleveref}

\makeatletter
\newcommand{\annotatehypertarget}[1]{\Hy@raisedlink{\hypertarget{#1}{}}}
\makeatother

\newcounter{annotatecount}
\newcounter{annotateidx}
\newcounter{annotatejdx}
\newcounter{annotatelabelcount}
\newcounter{annotateglobalindex}

\newcommand{\atran}[2]{%
  \stepcounter{annotatecount}%
  \overset{\text{\annotatehypertarget{\alph{annotatecount}\theannotateglobalindex}%
  {(\hyperlink{desc\alph{annotatecount}\theannotateglobalindex}{\alph{annotatecount}})}}}%
  {#1}%
  \csgdef{annotatedescription\theannotatecount}{#2}%
}

\newcommand{\aeq}[1]{\atran{=}{#1}}
\newcommand{\aleq}[1]{\atran{\leq}{#1}}

\newcommand{\annotateinitused}{%
  \setcounter{annotateidx}{0}%
  \whileboolexpr{test{\ifnumless{\theannotateidx}{\theannotatecount}}}{%
    \stepcounter{annotateidx}%
    \csgdef{aused\theannotateidx}{0}%
  }%
}

\newcommand{\annotategetlabels}{%
  \setcounter{annotatejdx}{0}%
  \setcounter{annotatelabelcount}{0}%
  \whileboolexpr{test{\ifnumless{\theannotatejdx}{\theannotatecount}}}{%
    \stepcounter{annotatejdx}%
    \ifcsequal{annotatedescription\theannotateidx}{annotatedescription\theannotatejdx}{%
      \csgdef{aused\theannotatejdx}{1}%
      \stepcounter{annotatelabelcount}%
      \csedef{annotatelabel\theannotatelabelcount}{\alph{annotatejdx}}%
    }{}%
  }%
}

\newcommand{\annotateprintlabels}{%
  \setcounter{annotatejdx}{0}%
  \whileboolexpr{test{\ifnumless{\theannotatejdx}{\theannotatelabelcount}}}{%
    \stepcounter{annotatejdx}%
    \ifnumequal{\theannotatejdx}{\theannotatelabelcount}{%
      \ifnumequal{\theannotatejdx}{1}{}{~and~}%
    }{}%
    \annotatehypertarget{desc\csuse{annotatelabel\theannotatejdx}\theannotateglobalindex}%
    {(\hyperlink{\csuse{annotatelabel\theannotatejdx}\theannotateglobalindex}%
    {\csuse{annotatelabel\theannotatejdx}})}%
    \ifnumless{\theannotatejdx}{\theannotatelabelcount}{%
      \ifnumless{\theannotatejdx+1}{\theannotatelabelcount}{,~}{}%
    }{}%
  }%
}

\newcommand{\annotate}{%
  \annotateinitused%
  \setcounter{annotateidx}{0}%
  \whileboolexpr{test{\ifnumless{\theannotateidx}{\theannotatecount}}}{%
    \stepcounter{annotateidx}%
    \ifcsstring{aused\theannotateidx}{0}{%
      \ifnumequal{\theannotateidx}{1}{}{;~}%
      \annotategetlabels%
      \annotateprintlabels~\csuse{annotatedescription\theannotateidx}%
    }{}%
  }%
  \setcounter{annotatecount}{0}%
  \stepcounter{annotateglobalindex}%
}

\DeclareMathOperator*{\argmin}{arg\,min}

\usepackage{enumitem}
\usepackage[font=footnotesize, labelfont=bf]{caption}

\usepackage[textsize=tiny]{todonotes}

\icmltitlerunning{One-Step Gradient Delay is Not a Barrier for Large-Scale Asynchronous Pipeline Parallel LLM Pretraining}

\begin{document}

\twocolumn[
  \icmltitle{One-Step Gradient Delay is Not a Barrier for Large-Scale \\ Asynchronous Pipeline Parallel LLM Pretraining}



  \icmlsetsymbol{equal}{*}
  \icmlsetsymbol{prior}{$\dagger$}

  \begin{icmlauthorlist}
    \icmlauthor{Philip Zmushko}{ista,yandex,brain,equal,prior}
    \icmlauthor{Egor Petrov}{yandex,brain,equal}
    \icmlauthor{Nursultan Abdullaev}{yandex,brain,innopolis}
    \icmlauthor{Mikhail Khrushchev}{yandex}
    \icmlauthor{Samuel Horváth}{mbzuai}
  \end{icmlauthorlist}

  \icmlaffiliation{ista}{Institute of Science and Technology Austria (ISTA), Austria}
  \icmlaffiliation{yandex}{Yandex, Russia}
  \icmlaffiliation{brain}{Basic Research of Artificial Intelligence Laboratory (BRAIn Lab), Russia}
  \icmlaffiliation{mbzuai}{Mohamed bin Zayed University of Artificial Intelligence (MBZUAI), UAE}
  \icmlaffiliation{innopolis}{Innopolis University, Russia}

  {\qquad \small \noindent
  \textsuperscript{*}Equal contribution.
  \textsuperscript{$\dagger$}Work completed while Philip Zmushko was at Yandex and BRAIn Lab; current affiliation: ISTA.
  \par}

  \icmlcorrespondingauthor{}{zmushko.ph.a@gmail.com}

  \icmlkeywords{Machine Learning, ICML}

  \vskip 0.3in
]



\printAffiliationsAndNotice{}  % no special notice (required even if empty)

\begin{abstract}
\textcolor{black}
{
Modern large-scale LLM pretraining benefits from utilizing Pipeline Parallelism; however, synchronous implementations leave GPUs idle during pipeline bubbles, wasting computational resources.
Asynchronous Pipeline Parallelism eliminates these bubbles, maximizing throughput at the cost of gradient staleness.
Among asynchronous schedules, PipeDream-\texttt{2BW} is particularly appealing: unlike the original PipeDream schedule, it ensures a constant one-step gradient delay regardless of pipeline depth.
However, its adoption remains limited due to the common belief that optimizing under staleness is fundamentally unstable.
In this work, we challenge this assumption, demonstrating that degradation under one-step delay depends strongly on optimizer choice rather than being an intrinsic limitation.
We provide the first comprehensive empirical analysis showing that while AdamW, the predominant optimizer at the time when PipeDream-\texttt{2BW} was introduced, indeed suffers from severe degradation, recent methods like Muon exhibit strong robustness under a one-step delay.
We introduce an optimizer-agnostic Error Feedback-inspired correction to further mitigate delay effects.
We provide supporting theoretical analysis demonstrating convergence for Muon with and without this correction.
Extensive evaluation on models up to 10B parameters confirms that our strategy bridges the performance gap with synchronous training, highlighting the practical potential of asynchronous pipeline parallelism at scale.
}
\end{abstract}

\section{Introduction} \label{sec:intro}

In the modern era of Large Language Models (LLMs), training on a single GPU is no longer feasible due to memory constraints, necessitating distributed training with model parallelism.
One common approach is Pipeline Parallelism (PP)~\citep{huang2019gpipe}, which partitions the model vertically into stages.
While PP was historically an essential component of large-scale training~\citep{narayanan2021efficient-megatron-1f1b-interleaved}, it became less popular following the introduction of memory-efficient data-parallel approaches like ZeRO~\citep{rajbhandari2020zero-fsdp}, and was primarily used only for models larger than 70B parameters~\citep{grattafiori2024llama3}.
However, the rise of Mixture-of-Experts (MoE) architectures~\citep{shazeer2017outrageously-moe-orig} has made this strategy less effective: MoE layers substantially increase the communication involved in training, without a proportional increase in per-layer computation.
This lower compute-to-communication ratio has led to renewed interest in PP, 
% often combined with Expert Parallelism 
in recent large-scale runs~\citep{liu2024deepseek2,deepseek3,team2025kimi}.

However, synchronous PP suffers from a fundamental limitation: preserving synchronous parameter updates introduces empty slots in the pipeline schedule, known as ``bubbles'', during which some GPUs remain idle, reducing global utilization and efficiency.
Despite extensive efforts to mitigate these bubbles~\citep{narayanan2021efficient-megatron-1f1b-interleaved,qi2023zero-bubble,dualpipe}, they cannot be entirely eliminated in a synchronous setting.
Alternatively, Asynchronous PP (Async PP) avoids synchronization entirely, allowing for the complete removal of pipeline bubbles.
Under standard bubble models, this can translate into substantial schedule-level speedups over synchronous PP\footnote{See~\cref{app:runtime_overhead} for a quantitative estimate.}.
Unfortunately, this comes at a cost: Async PP is no longer semantically equivalent to conventional minibatch training, since gradients may be computed using stale parameters or applied after a delay.

Unlike synchronous PP, Async PP remains significantly less explored in the context of language model pre-training.
To the best of our knowledge, \citet{ajanthan2025nesterov-async-pp} and~\citet{jung2026mitigating} are the only works that study Async PP for pre-training decoder-only language models.
However, their experiments rely on the original PipeDream schedule~\citep{narayanan2019pipedream2019-1f1b}, which has a critical limitation: \textbf{variable gradient delays}.
Because PipeDream updates parameters after each local backward pass, different pipeline stages observe gradients with different amounts of delay.
This staleness heterogeneity leads to severe convergence degradation as the number of stages increases: \citet{ajanthan2025nesterov-async-pp} report an increase of more than $0.2$ in validation loss compared to synchronous training at $16$ stages, a practical scale for real-world training~\citep{deepseek3}.

These findings highlight the need for an asynchronous approach to language modeling that remains robust as the pipeline depth increases.
PipeDream-\texttt{2BW}~\citep{pipedream-2bw} is a natural candidate, as it ensures a constant gradient delay across all stages.
By performing updates once every $M$ backward passes, PipeDream-\texttt{2BW} guarantees a \textbf{uniform staleness of 1 regardless of the pipeline size}\footnote{\textcolor{black}{We also discuss WPipe~\citep{yang2022groupbased} as a potentially better scheduling alternative for practitioners in~\cref{app:wpipe}.
}}.
Rather than dealing with variable delays across stages, this reduces the optimization challenge to the cleaner setting of training with a fixed one-step delay:
    $w_{t+1} = w_{t} - u_{t-1}(g_{t-1}),$
where $u_{t-1}$
denotes the optimizer update function applied to the delayed gradient $g_{t-1}$.
While optimization under staleness is well studied in theory~\citep{mishchenko2022delayed-theory-3,koloskova2022delayed-theory-4}, its practical application to LLM pre-training remains an open challenge.
We address this gap by providing practical guidance on the effects of gradient delay and demonstrating the practical viability of Async PP for LLM pre-training.

\begin{figure}[t]
    \centering
    \includegraphics[width=0.485\textwidth]{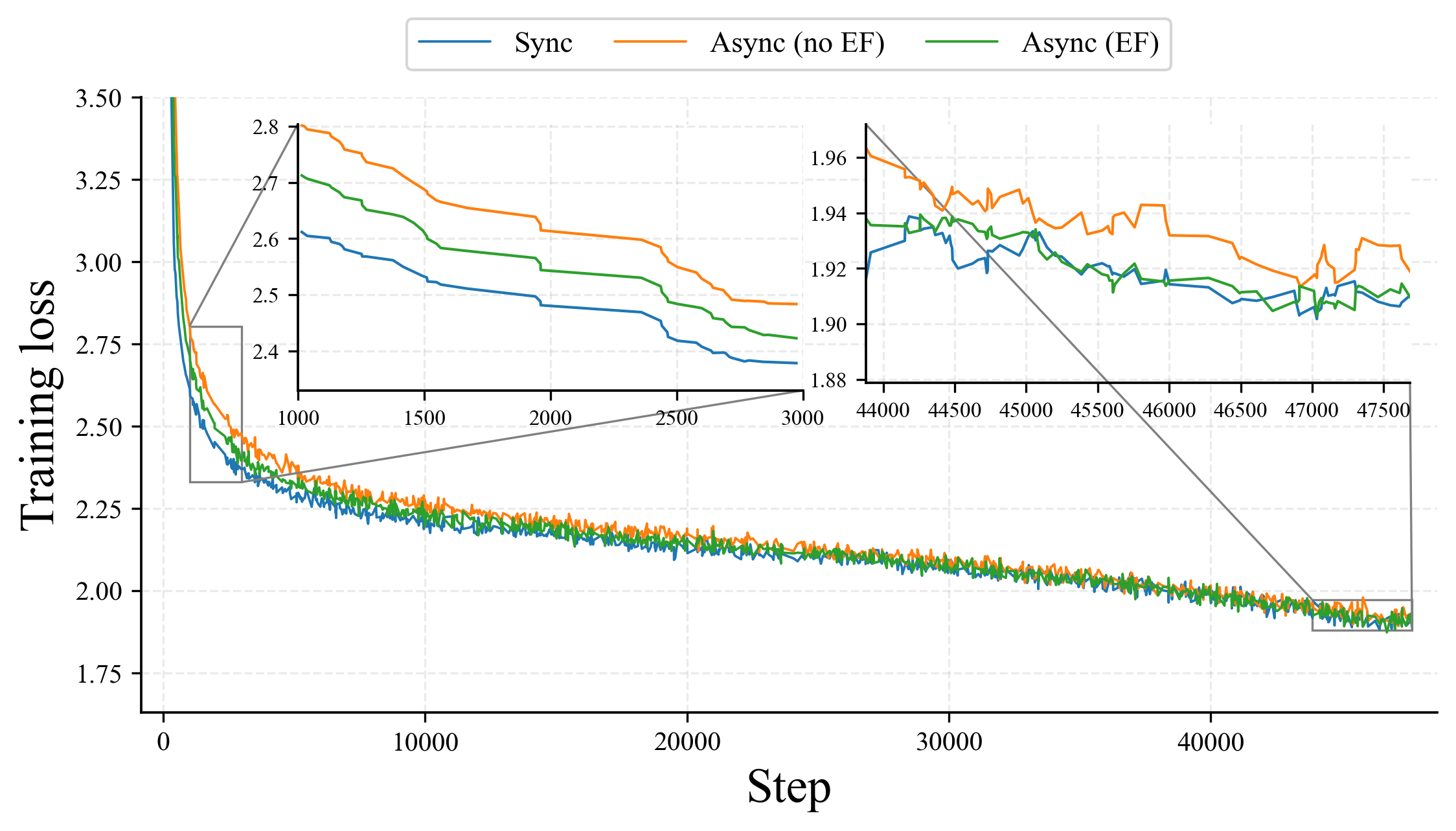}
\caption{Training loss on a 10B MoE model trained for 200B FineWeb-Edu tokens.
Both asynchronous runs remain stable throughout training, with Async PP + Error Feedback closing the final gap to the synchronous baseline entirely.}
    \label{fig:10b}
    % \vspace{-1em}
\end{figure}

Our contributions can be summarized as follows:

\begin{itemize}[leftmargin=*, itemsep=0pt, topsep=0pt]
    \item We conduct the first comprehensive empirical analysis of optimizers and hyperparameters for language model training under gradient staleness, identifying a critical relationship between momentum and loss degradation.
    In particular, we show that AdamW~\citep{loshchilov2017decoupled-adamw}, the historically dominant optimizer during the development of early Async PP methods, suffers substantial quality loss under staleness.
    In contrast, several modern optimizers remain surprisingly robust.
    Among them, Muon~\citep{jordan6muon-keller}, which is rapidly emerging as a leading optimizer for LLM pre-training~\citep{team2025kimi,zeng2025glm}, offers a particularly strong trade-off: it achieves competitive synchronous performance while maintaining a small sync-async gap under default hyperparameters.
    \item We investigate several staleness mitigation strategies and \textcolor{black}{derive} an optimizer-agnostic correction mechanism inspired by Error Feedback~\citep{seide2014first-ef}.
    This technique consistently narrows the performance gap between synchronous and asynchronous and further improves the already small gap observed for Muon.
    \item We empirically demonstrate the superiority of constant gradient delay over variable delay by comparing PipeDream-\texttt{2BW} against the original PipeDream schedule across multiple optimizers.
    These results show that fixed staleness is crucial for stable Async PP training at larger pipeline depths.
    % \item \textcolor{blue}{We provide theoretical convergence guarantees for Muon with Error-Feedback correction under gradient staleness.
    % To the best of our knowledge, this constitutes the first convergence analysis for Linear Minimization Oracle (LMO) algorithms with Error-Feedback under gradient delay.}
    \item \textcolor{black}{We provide the first theoretical convergence analysis of Linear Minimization Oracle (LMO) algorithms under gradient delay, establishing guarantees for both standard delayed updates and our Error-Feedback correction.}
    \item Finally, we validate our findings at scale by training a 10B-parameter Mixture-of-Experts (MoE) model on 200B tokens with Muon.
    \textbf{With Async PP and Error Feedback, we achieve a \textcolor{black}{final loss identical to that of the synchronous baseline} while using the exact same hyperparameters}, marking, to the best of our knowledge, the first successful demonstration of Async PP at this scale without quality degradation.
\end{itemize}

\begin{figure}[t]
    \centering
    \includegraphics[width=0.45\textwidth]{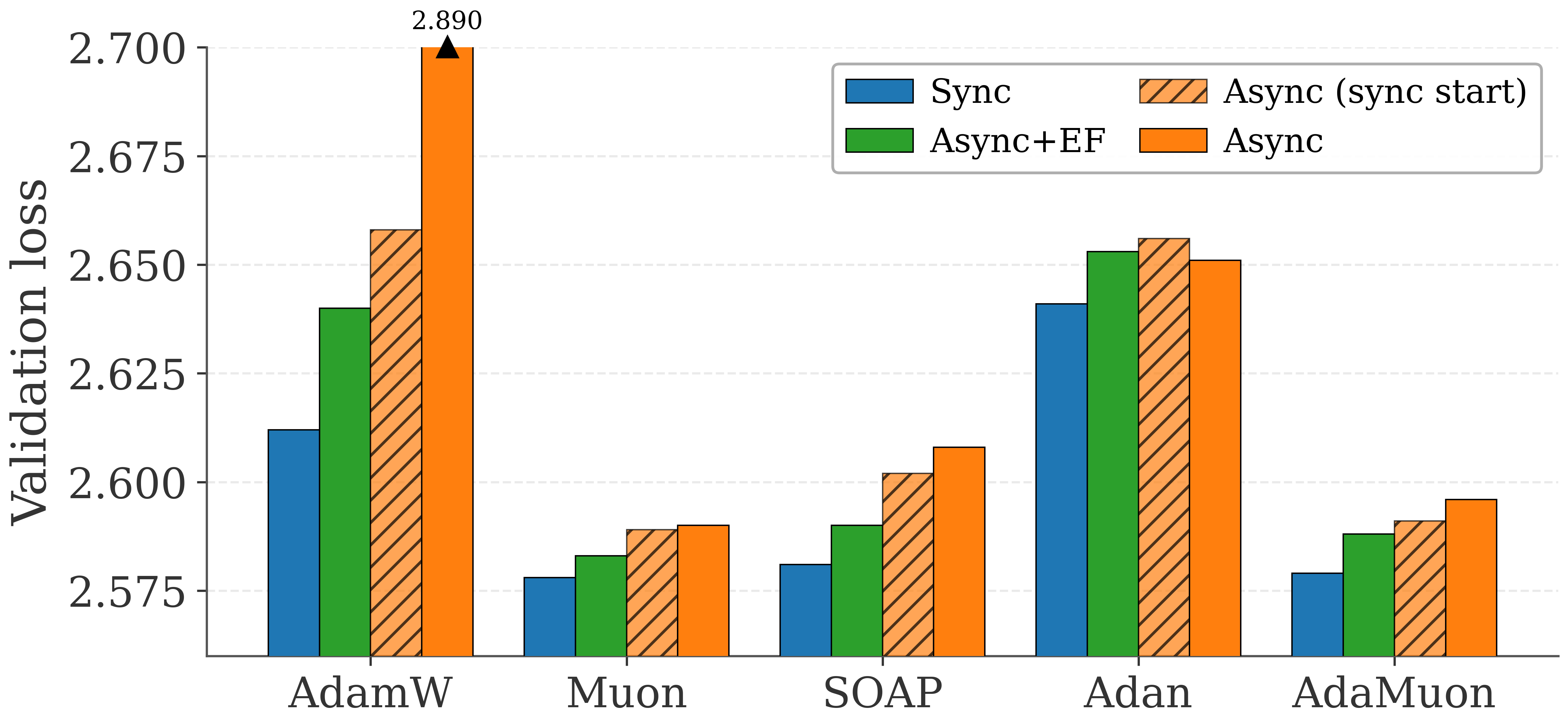}
    \caption{Validation loss of synchronous and one-step delayed optimizers on the 360M model.
For most optimizers, Error Feedback cuts the sync-async gap by more than half compared to standard delayed training and outperforms the synchronous-start baseline.}
    \label{fig:ef_barplot}
    % \vspace{-1em}
\end{figure}

\vspace{-1em}
\section{One-Step Delayed Optimization} \label{sec:opt_delay}

Before evaluating optimizers, we first clarify the delayed-update abstraction used throughout the paper.
\textcolor{black}{The closest prior work on Async PP for language model pre-training,} \citet{ajanthan2025nesterov-async-pp}, relies on the original PipeDream schedule~\citep{narayanan2019pipedream2019-1f1b}, thereby inheriting its variable gradient delays.
In PipeDream, each stage updates its parameters immediately after a local backward pass, so different stages \textcolor{black}{can observe} gradients with different amounts of delay.
Moreover, preserving forward-backward consistency requires weight stashing, with one stashed parameter version per delay level induced by the schedule.
We instead use PipeDream-\texttt{2BW}~\citep{pipedream-2bw}, which avoids variable delays by updating stage parameters only after a full minibatch of $M$ micro-batches has completed backward propagation.
Intuitively, when $M \ge P-1$, where $P$ is the number of pipeline stages, even the last micro-batch has enough time to complete its forward-backward pass through the pipeline before the \textcolor{black}{second next} minibatch triggers an update, yielding a uniform one-step gradient delay across all stages.
PipeDream-\texttt{2BW} also reduces weight stashing to a single additional parameter copy, whose memory cost is negligible in realistic LLM training setups; see~\cref{app:memory_overhead} for a quantitative discussion.
As shown by \citet{pipedream-2bw} for SGD, this constant-delay schedule can be viewed as standard optimization with the previous-step gradient.
Extending this view from SGD to an arbitrary optimizer gives the \textcolor{black}{generic} delayed-update rule in~\cref{alg:delay_gradient}.
\textcolor{black}{
Note, that the time index $t$ in $u_t$ indicates that the update may depend on iteration-dependent quantities, such as the learning rate, momentum, or variance buffers.
}

\begin{algorithm}[h]
\caption{Delayed Gradient Update}
\label{alg:delay_gradient}
\begin{algorithmic}[1]
\REQUIRE Initial point $x_0$, learning rate $\eta$, iterations $T$
\ENSURE Final point $x_T$
\STATE \textcolor{black}{Initialize $g_{-1} = 0$ and $u_{-1} = 0$}
\FOR{$t = 0, 1, \ldots, T-1$}
    \STATE Compute gradient $g_t$
    \STATE \textcolor{black}{Update optimizer statistics with $g_{t-1}$}
    \STATE \textcolor{black}{Calculate update step $u_{t-1}(g_{t-1})$}
    \IF{Standard Async \textcolor{black}{or $t \leq 1$}}
        \STATE $x_{t+1} \leftarrow x_t - u_{t-1}(g_{t-1})$
    \ELSIF{Error-Feedback (\cref{sec:ef})}
        \STATE $x_{t+1} \leftarrow x_t - 2 \cdot u_{t-1} (g_{t - 1}) + u_{t - 2}(g_{t - 2})$
    \ENDIF
\ENDFOR

\end{algorithmic}
\end{algorithm}

\textcolor{black}{We begin our empirical analysis by evaluating existing optimizers under this one-step delayed-update rule.
Throughout the paper, each delayed run is compared against a synchronous baseline trained with the same hyperparameters, so the reported gap isolates the effect of one-step staleness from changes in the training recipe.
The experiments in this section use 135M and 360M models with the same architecture as in~\citet{allal2025smollm2} trained on FineWeb-Edu~\citep{penedo2024fineweb} at a 20:1 token-to-parameter ratio.
We use a common default global recipe across optimizers: weight decay $0.1$, gradient clipping at $1.0$, a cosine learning-rate schedule decaying to $0.1$ of the peak value, and a warmup lasting $10\%$ of the training budget.
Peak learning rate and optimizer-specific coefficients are tuned as described below.
We also fix the global batch size throughout this section, choosing it from existing scaling-law estimates for near-optimal synchronous training; the detailed batch-size calculation, default hyperparameters, and full training setup are given in Appendix~\ref{app:experimental_setup}, and we return to the role of batch size in the hyperparameter sensitivity analysis below.}

% To facilitate rigorous ablation studies, we do not deploy a full distributed pipeline parallel implementation for these experiments.
% Instead, following methodologies established in prior work~\citep{yang2021pipemare}, we train models using standard data parallelism and artificially emulate gradient staleness.
% This approach allows for precise control over the delay mechanism while isolating optimization effects.
% Detailed implementation specifics are discussed in~\cref{app:emulation}.
\vspace{-5pt}
\subsection{Initial Observation: The Staleness Robustness Gap}\label{sec:initial_observation}
\vspace{-3pt}

We start with a simple comparison between two representative optimizers: AdamW~\citep{kingma2014adam,loshchilov2017decoupled-adamw} and Muon~\citep{jordan6muon-keller,liu2025muon-is-scalable}.
AdamW is the standard optimizer used in most LLM pre-training pipelines and was the dominant choice when early asynchronous pipeline-parallel methods were developed.
Muon, in contrast, is a more recent optimizer that has seen rapid adoption in large-scale LLM training reports~\citep{team2025kimi,zeng2025glm}.

For this initial comparison, we train a 360M-parameter model and use commonly adopted default hyperparameters for both optimizers, tuning only the learning rate.
Specifically, we use $\beta=(0.9, 0.95)$ for AdamW~\citep{touvron2023llama1,touvron2023llama2,li2025minimax,deepseek3}, and 
momentum $\mu=0.95$ and update RMS $0.2$ for Muon~\citep{liu2025muon-is-scalable,zeng2025glm}.

\begin{figure}[t]
    \centering
    \includegraphics[width=0.46\textwidth]{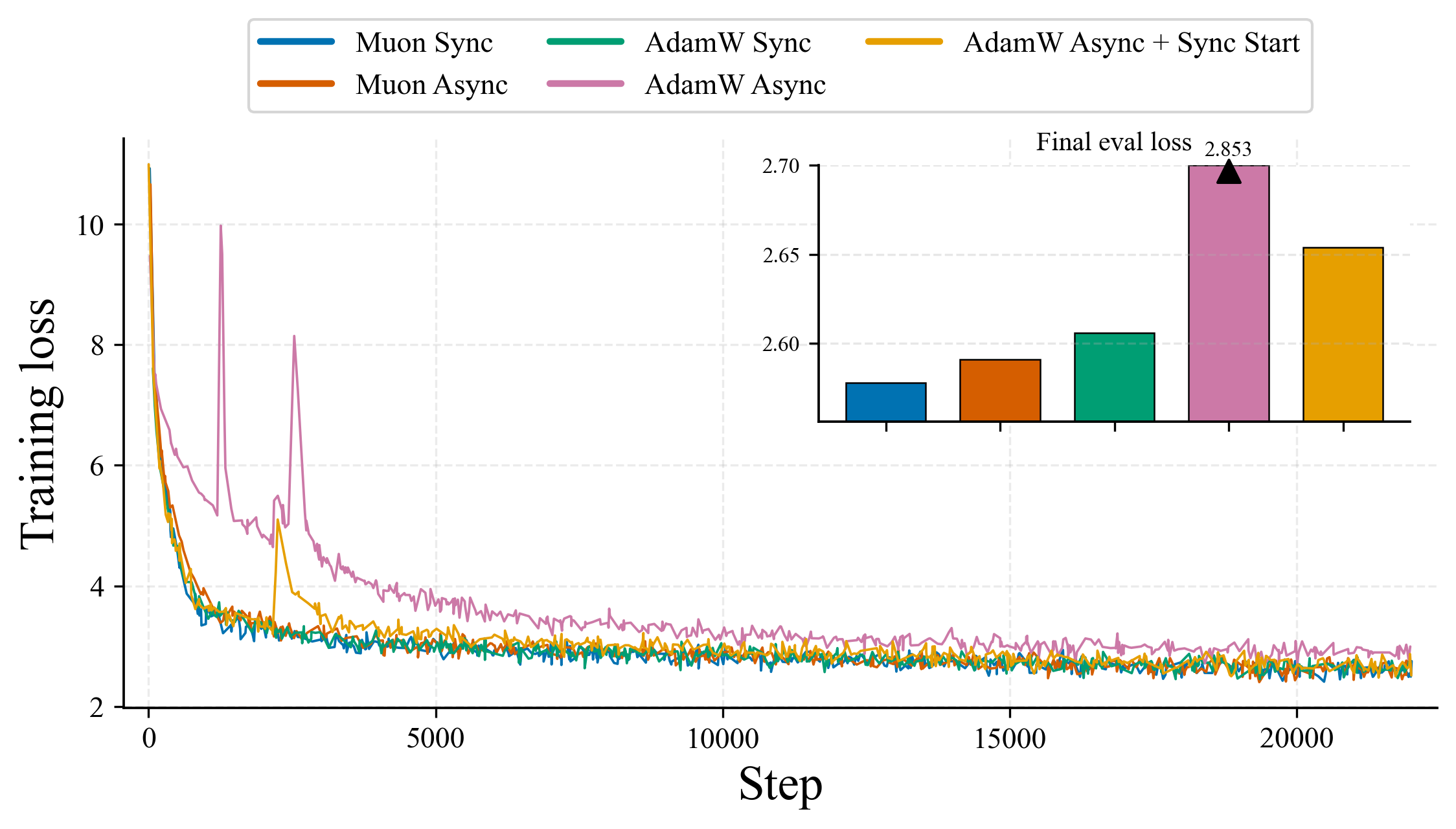}
    \caption{Synchronous and delayed AdamW and Muon training on the 360M model.
AdamW degrades substantially under one-step delay, even with a synchronous start, whereas Muon achieves a much smaller final sync-async gap.}
    \label{fig:optimizer_delay_comparison}
    % \vspace{-1.5em}
\end{figure}

The results in~\cref{fig:optimizer_delay_comparison} show a clear gap in robustness to one-step gradient delay.
Both optimizers train competitively in the synchronous setting, but their delayed variants behave very differently.
AdamW suffers severe quality degradation under one-step delay ($>0.2$), with the training loss diverging early from the synchronous trajectory.
Starting training synchronously before switching to delayed updates, a stabilization heuristic used in prior work~\citep{ren2021zero-offload}, improves AdamW but does not close the gap: the final sync-async loss gap remains $0.046$, and the loss exhibits a sharp spike immediately after the switch to delayed training.
In contrast, Muon achieves a much smaller final gap of only $0.012$ without any synchronous start.

This initial observation suggests that staleness is not uniformly harmful across optimizers.
Rather, robustness to one-step delay appears to depend strongly on the optimizer dynamics.
% The remainder of this section therefore broadens the comparison to a larger set of modern optimizers and studies which hyperparameters control the sync-async performance gap.
\vspace{-5pt}
\subsection{Hyperparameter Sensitivity and Benchmarking} 
\vspace{-3pt}\label{sec:section 2 opt_delay_hyperparameter_sensitivity}

We next broaden the comparison to a larger set of optimizers and study which hyperparameters control robustness to a one-step delay.
We evaluate AdamW~\citep{loshchilov2017decoupled-adamw}, Muon~\citep{liu2025muon-is-scalable}, SOAP~\citep{vyas2024soap}, Nadam~\citep{dozat2016incorporating-nadam}, MARS~\citep{yuan2024mars}, Adan~\citep{xie2024adan}, and Lion~\citep{chen2023symbolic-lion}, alongside several Muon variants including AdaMuon~\citep{si2025adamuon}, NorMuon~\citep{li2025normuon}, and MARS-M~\citep{liu2025mars-m}.
\textcolor{black}{
Starting from the common global recipe described above, we tune the peak learning rate for each optimizer in a local range around the AdamW optimum, following prior evidence that many modern LLM optimizers have optima in similar regions~\citep{semenov2025benchmarking-optimizers-benchmarking-1,wen2025fantastic-optimizers-benchmarking-1}.
We then perform one-dimensional sweeps around these settings on the 135M model to study sensitivity to both global and optimizer-specific hyperparameters: learning rate, weight decay, momentum decay coefficient, the second-moment coefficient $\beta_2$, warmup length, gradient clipping threshold, and the learning-rate scheduler.
}

\begin{table}[t]
\centering
\footnotesize
\caption{Validation loss under synchronous and one-step delayed training for 135M and 360M models.
Most modern optimizers remain robust to one-step delay, while AdamW and MARS suffer severe degradation.
Colors indicate the loss increase relative to the synchronous baseline:
\textcolor{green!60!black}{green} $\Delta \le 0.015$,
\textcolor{cyan!70!black}{cyan} $0.015 < \Delta \le 0.03$,
\textcolor{blue!70!black}{blue} $0.03 < \Delta \le 0.05$,
and \textcolor{red!65!black}{red} $\Delta > 0.05$.}
\label{tab:main_results}

\setlength{\tabcolsep}{4pt}

\begin{tabular}{l|cc|cc}
\toprule
& \multicolumn{2}{c|}{\textbf{SmoLLM-135M}}
& \multicolumn{2}{c}{\textbf{SmoLLM-360M}} \\
\midrule
\textbf{Optimizer}
& \textbf{Sync}
& \textbf{Async}
& \textbf{Sync}
& \textbf{Async} \\
\midrule

Muon
& 2.841
& \textcolor{green!60!black}{2.855}
  {\tiny\textcolor{green!60!black}{(+0.014)}}
& 2.578
& \textcolor{green!60!black}{2.590}
  {\tiny\textcolor{green!60!black}{(+0.012)}} \\

Adan
& 2.896
& \textcolor{green!60!black}{2.902}
  {\tiny\textcolor{green!60!black}{(+0.006)}}
& 2.641
& \textcolor{green!60!black}{2.651}
  {\tiny\textcolor{green!60!black}{(+0.010)}} \\

NorMuon
& 2.837
& \textcolor{cyan!70!black}{2.858}
  {\tiny\textcolor{cyan!70!black}{(+0.019)}}
& 2.574
& \textcolor{green!60!black}{2.588}
  {\tiny\textcolor{green!60!black}{(+0.014)}} \\

AdaMuon
& 2.845
& \textcolor{cyan!70!black}{2.867}
  {\tiny\textcolor{cyan!70!black}{(+0.022)}}
& 2.579
& \textcolor{cyan!70!black}{2.596}
  {\tiny\textcolor{cyan!70!black}{(+0.017)}} \\

SOAP
& 2.850
& \textcolor{cyan!70!black}{2.872}
  {\tiny\textcolor{cyan!70!black}{(+0.022)}}
& 2.581
& \textcolor{cyan!70!black}{2.608}
  {\tiny\textcolor{cyan!70!black}{(+0.027)}} \\

Lion
& 2.870
& \textcolor{cyan!70!black}{2.894}
  {\tiny\textcolor{cyan!70!black}{(+0.024)}}
& 2.624
& \textcolor{cyan!70!black}{2.654}
  {\tiny\textcolor{cyan!70!black}{(+0.030)}} \\

MARS-M
& 2.840
& \textcolor{blue!70!black}{2.875}
  {\tiny\textcolor{blue!70!black}{(+0.035)}}
& 2.578
& \textcolor{cyan!70!black}{2.607}
  {\tiny\textcolor{cyan!70!black}{(+0.029)}} \\

NAdam
& 2.896
& \textcolor{blue!70!black}{2.936}
  {\tiny\textcolor{blue!70!black}{(+0.040)}}
& 2.651
& \textcolor{blue!70!black}{2.694}
  {\tiny\textcolor{blue!70!black}{(+0.043)}} \\

MARS
& 2.874
& \textcolor{red!65!black}{3.343}
  {\tiny\textcolor{red!65!black}{(+0.469)}}
& 2.615
& \textcolor{red!65!black}{2.897}
  {\tiny\textcolor{red!65!black}{(+0.282)}} \\

AdamW
& 2.877
& \textcolor{red!65!black}{3.227}
  {\tiny\textcolor{red!65!black}{(+0.350)}}
& 2.612
& \textcolor{red!65!black}{2.890}
  {\tiny\textcolor{red!65!black}{(+0.278)}} \\

\bottomrule
\end{tabular}

% \vspace{-1em}
\end{table}

\textcolor{black}{
\textbf{Broad hyperparameter trends.}
The corresponding one-dimensional sweeps are reported in Appendix~\ref{app:other_hyperparameters}.
Overall, they suggest that delayed training tends to amplify instabilities already present in the underlying synchronous recipe.
For example, increasing the peak learning rate or shortening warmup often increases the sync-async gap, consistent with the fact that both changes make the early optimization trajectory less stable.
Weight decay shows a similar stability trade-off: values near the default keep the sync-async gap relatively small, while extreme values can worsen delayed training.
In contrast, gradient clipping and the learning-rate schedulers choice have little systematic effect on the gap.
The second-moment coefficient $\beta_2$ is also inconsistent across optimizers: larger values worsen performance for some Adam-like methods but have little effect on SOAP and NorMuon.
}

\textbf{The Role of Momentum.}
\textcolor{black}{Against this mixed picture, one hyperparameter does exhibit a clear and consistent trend across optimizers: the \textit{momentum decay coefficient}, which is typically denoted by $\mu$ in Muon-like optimizers and by $\beta_1$ in Adam-like optimizers.}
Formally, this corresponds to the coefficient $\beta$ in the Exponential Moving Average (EMA) update: $m_t = \beta m_{t-1} + (1-\beta) g_t$.
As shown in~\cref{fig:momentum_sweep}, increasing this coefficient consistently reduces the loss penalty caused by one-step delay.

This observation suggests a broader explanation for the benefits reported by~\citet{ajanthan2025nesterov-async-pp}.
They observed that higher momentum improves delayed training and attributed this effect to the ``look-ahead'' structure of Nesterov Accelerated Gradient, motivating the use of Nadam~\citep{dozat2016incorporating-nadam}.
Our ablations indicate that the effect is not specific to Nesterov-style updates: higher momentum also improves robustness for optimizers without a look-ahead mechanism.
We hypothesize that the mechanism is more fundamental: in the presence of delayed gradients, the optimizer cannot rely as heavily on the instantaneous update as in the synchronous setting.
A higher momentum coefficient effectively dampens the noise introduced by staleness, forcing the optimization trajectory to rely more on the accumulated history rather than the potentially erratic current step.
\begin{figure*}[t]
    \centering
    % Left sub-figure
    \begin{subfigure}[b]{0.49\textwidth}
        \centering
        \includegraphics[width=\textwidth]{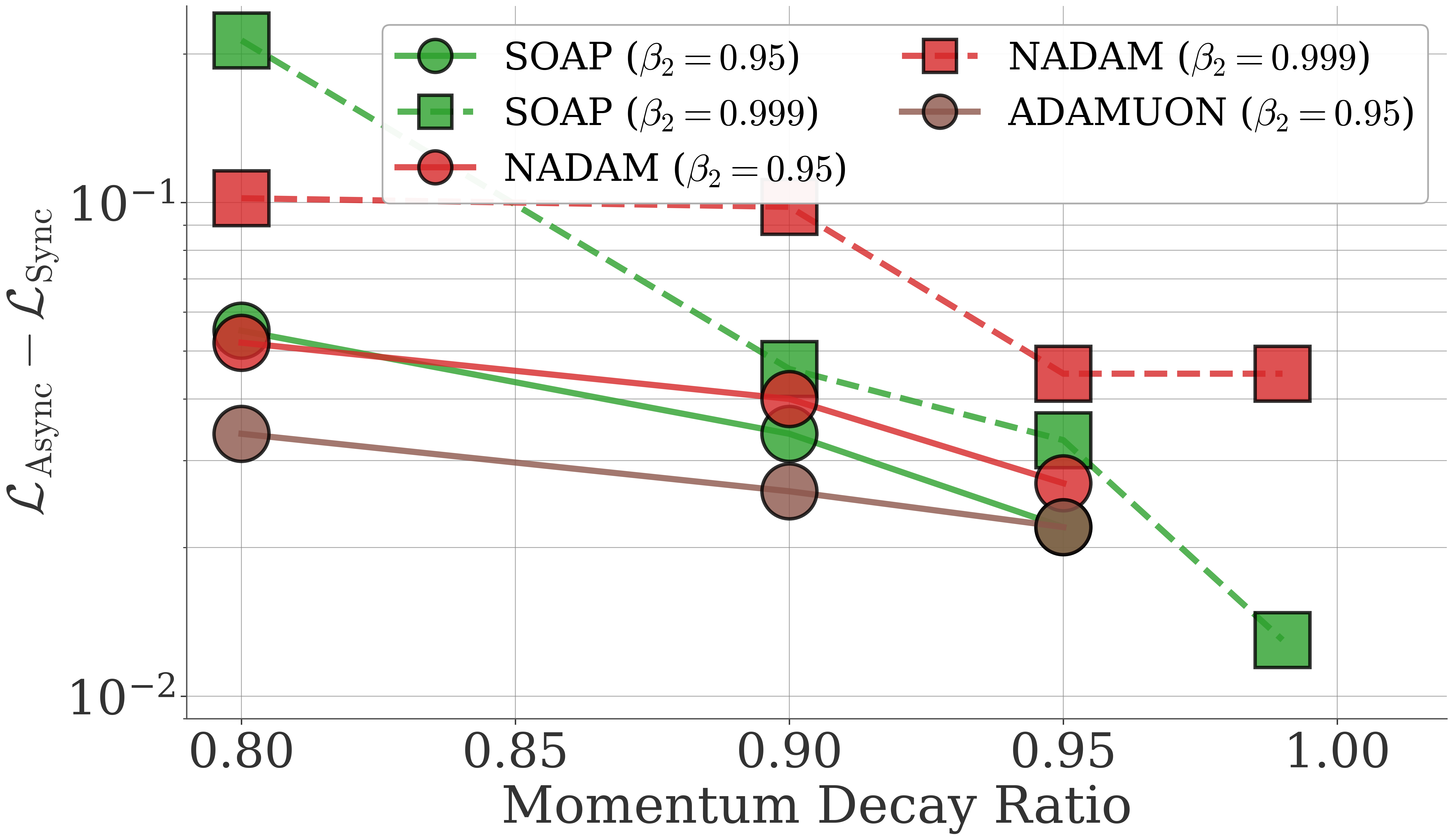}
        % \caption{Performance gap for SOAP, Nadam, AdaMuon.}
        \label{fig:beta1_part1}
    \end{subfigure}
    \hfill
    \begin{subfigure}[b]{0.49\textwidth}
        \centering
        \includegraphics[width=\textwidth]{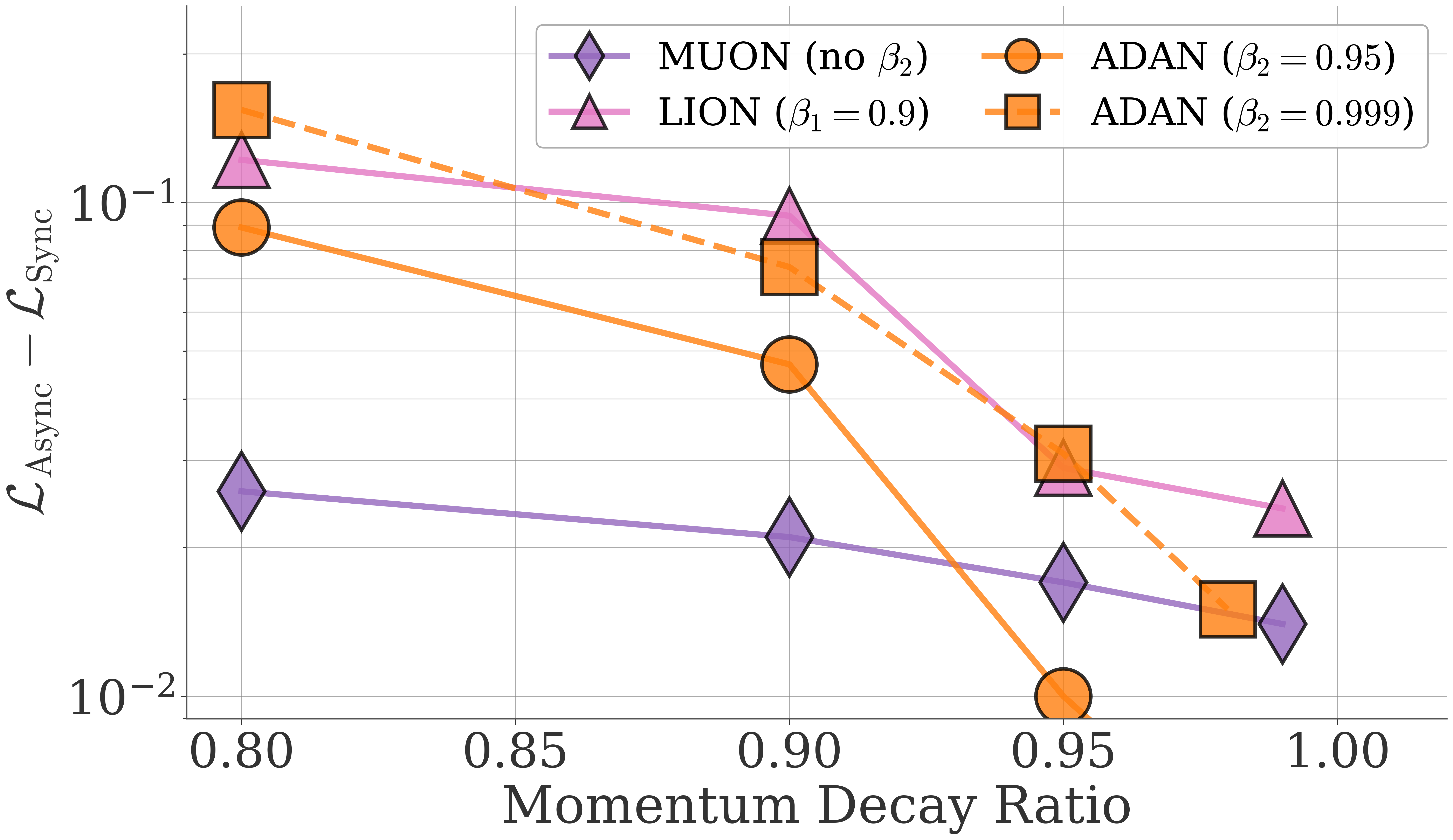}
        % \caption{Performance gap for Muon, Adan, Lion.}
        \label{fig:beta1_part2}
    \end{subfigure}
    \vspace{-10pt}
    \caption{Final loss gap between synchronous and asynchronous training as a function of the momentum decay for various optimizers.
    \textcolor{black}{
    Note that we exclude certain high-momentum configurations in which the synchronous baseline itself diverges due to instability.
    }
    }
    \label{fig:momentum_sweep}
    \vspace{-14pt}
\end{figure*}

\textcolor{black}{
\textbf{Batch size as a special hyperparameter.}
There is one additional hyperparameter not discussed above that has a particularly strong effect on delayed training: the global batch size.
Reducing the batch size can substantially shrink, and in some cases nearly eliminate, the sync-async gap, as shown by the two-dimensional sweeps over the pair \{batch size, learning rate\} in Appendix~\ref{app:batch_size_lr}.
However, batch size is not an ordinary optimizer hyperparameter in large-scale training, because it is often constrained by hardware utilization: reducing it too aggressively may prevent the training system from fully utilizing the available compute resources.
At the same time, the target metric for training quality is the best absolute validation loss attainable by the asynchronous run, not the gap alone.
From this perspective, very small batch sizes are not necessarily attractive: despite the smaller gap, the best asynchronous losses are achieved at batch sizes slightly smaller than the synchronous optimum.
Conversely, increasing the batch size can make the sync-async gap large, but these regimes are also far from optimal for synchronous training itself and would typically be unattractive even without delay.
For this reason, we continue to compare optimizers at the same synchronous near-optimal batch size used throughout this section and treat this fixed-batch-size comparison as a practical proxy for their overall capability under one-step delayed training.
}

\textbf{Benchmarking Results.}
\textcolor{black}{
Having characterized the main hyperparameter trends, we now ask how well each optimizer can perform under one-step delay while still retaining a strong synchronous baseline.
For each optimizer, we start from the common global recipe described above and tune only the peak learning rate and optimizer-specific hyperparameters.
We then report the best delayed-training validation loss, subject to the constraint that the \textcolor{black}{corresponding synchronous baseline remains within $0.01$ of the global synchronous optimum.\footnote{\textcolor{black}{We note that the empirical noise level is on the order of $10^{-3}$; see~\cref{app:noise_level} for details.}}}
}
The results for 135M and 360M models are summarized in~\cref{tab:main_results}.
An important outcome is that severe degradation is concentrated in only a small subset of optimizers: AdamW and MARS degrade substantially under one-step delay, while most other modern optimizers are far more robust.
In particular, Muon, Adan, NorMuon, AdaMuon, SOAP, and Lion keep the final loss gap within $0.03$ across both model sizes.
Among these methods, Adan achieves the smallest sync-async gap, which is also consistent with the role of momentum: its best configuration uses the default high first-moment coefficient $\beta_1=0.98$, substantially larger than the \textcolor{black}{standard} choice of $\beta_1=0.9$ in Adam-like optimizers.
Muon, however, offers the strongest practical trade-off: it is among the best synchronous optimizers, remains highly robust under delay, and is increasingly relevant for modern LLM pre-training.

\textcolor{black}{
Motivated by the unusually poor behavior of AdamW-like optimizers, we include additional diagnostic ablations in Appendix~\ref{app:adamw_ablations}.
One interesting observation is that delaying only the first-moment update $m_t$ closely reproduces the behavior of fully delayed AdamW, further suggesting that first-moment dynamics are central to AdamW's sensitivity.
Concurrent work by~\citet{jung2026mitigating} studies this issue from a different perspective and proposes to mitigate it by rotating the basis in which AdamW applies its adaptive update.
This aligns with our empirical results: SOAP, which similarly uses basis rotations before applying Adam-like updates, is substantially more robust to one-step delay than AdamW in our experiments.
}

Overall, the results in this section show that one-step staleness affects optimizers very differently, with momentum playing a particularly important role.
They also highlight Muon as a particularly strong candidate for large-scale Async PP: it combines strong synchronous performance, robustness to delayed updates, and relevance to modern LLM pre-training.
Before moving to scale, we next ask whether the remaining sync-async gap can be reduced by optimizer-agnostic mitigation strategies.

% \begin{figure*}[t]
%     \centering
%     % Left sub-figure
%     \begin{subfigure}[b]{0.49\textwidth}
%         \centering
%         \includegraphics[width=\textwidth]{figures/beta_part1.png}
%         % \caption{Performance gap for SOAP, Nadam, AdaMuon.}
%         \label{fig:beta1_part1}
%     \end{subfigure}
%     \hfill
%     \begin{subfigure}[b]{0.49\textwidth}
%         \centering
%         \includegraphics[width=\textwidth]{figures/beta_part2.png}
%         % \caption{Performance gap for Muon, Adan, Lion.}
%         \label{fig:beta1_part2}
%     \end{subfigure}
    
%     \caption{Final loss gap between synchronous and asynchronous training as a function of the momentum decay for various optimizers.
%     \textcolor{blue}{
%     Note that we exclude certain high-momentum configurations in which the synchronous baseline itself diverges due to instability.
%     }
%     }
%     \label{fig:momentum_sweep}
% \end{figure*}

% \vspace{-1em}
\vspace{-10pt}
\section{Staleness Mitigation}
\vspace{-5pt}
\label{sec:staleness_mitigation}

We now study generic mitigation strategies that can be applied on top of the delayed-update rule in~\cref{alg:delay_gradient}.
We first evaluate several natural baseline strategies for mitigating one-step staleness, and then introduce an Error-Feedback-inspired correction that provides a more consistent improvement.

% \vspace{-0.6em}
\vspace{-5pt}
\subsection{Baseline Mitigation Strategies}
\vspace{-3pt}
\label{sec:sync_warmup_strategy}

\textbf{Synchronous Start.}
We first revisit the synchronous-start baseline discussed in~\cref{sec:initial_observation}.
This strategy follows ZeRO-Offload~\citep{ren2021zero-offload}, where starting training synchronously before switching to asynchronous updates was reported to improve stability.
Here, we evaluate whether the same strategy provides a consistent mitigation across the broader optimizer set.

The results in~\cref{tab:ef_360m} show a mixed picture.
For many optimizers, synchronous start recovers a nontrivial fraction of the remaining sync-async gap, typically around $20$--$30\%$.
At the same time, as already noted in~\cref{sec:initial_observation}, this strategy introduces an additional sensitivity at the transition point.
In particular, for adaptive optimizers such as SOAP and Adan, larger $\beta_2$ values can make the switch from synchronous to delayed updates highly unstable.
At the switching point, these runs exhibit a sharp loss spike, followed by a long recovery period and, in some cases, divergence (see example in~\cref{fig:beta2loss}).
Beyond these stability issues, synchronous start temporarily reintroduces pipeline bubbles and requires supporting both synchronous and asynchronous execution modes, reducing the practical throughput benefit of Async PP.
% see~\cref{app:runtime_overhead} for a bubble-model estimate of this effect.
Overall, we find synchronous start to be a useful baseline, but not a reliable standalone mitigation strategy.

\textbf{Synchronous Cooldown.}
We also evaluate the opposite schedule-level intervention: switching from delayed updates back to synchronous training near the end of the run.
This tests whether removing staleness in the final phase can recover the remaining loss gap.
As shown in~\cref{tab:delay_end_ablation}, this strategy yields only marginal improvements for both Muon and AdamW.
Thus, the residual gap is not easily removed by making only the final part of training synchronous.

\textbf{DC-ASGD / Taylor-based Delay Compensation.}
Finally, we test Delay-Compensated ASGD (DC-ASGD)~\citep{zheng2017asynchronous-delay-compensation-taylor}, which modifies the stale gradient using a Taylor-style correction term proportional to $\lambda \odot g^2 \odot \Delta w$.
A simple scale estimate suggests that this correction is extremely small for LLM training unless $\lambda$ is very large: in our runs, gradients are typically around $10^{-5}$, while parameter updates are proportional to the learning rate, around $10^{-3}$.
We therefore sweep $\lambda$ from $10^4$ to $10^8$.
As shown in~\cref{fig:taylor_lambda}, values up to $10^6$ produce losses indistinguishable from the standard delayed baseline up to the third decimal place, indicating that the correction remains too small to matter.
Larger values make the correction visible, but only degrade training rather than improving it.
We therefore find this gradient-level correction ineffective in our setting.

Taken together, these baselines suggest that simple schedule-level or gradient-level fixes are insufficient.
Synchronous start can help some optimizers but may introduce a sharp switching spike.
Synchronous cooldown has little effect, while Taylor-based correction only destabilizes training once scaled up.
We therefore turn to a correction mechanism that operates directly at the optimizer-update level.

\begin{table}[t]
\centering
\footnotesize
\caption{
% \textbf{SmoLLM-360M on FineWeb-Edu.}
Validation loss under synchronous and delayed training with different staleness mitigation techniques for 360M model.
The \textbf{Standard} column is delayed training without mitigation; percentages in parentheses show the recovered fraction of the Standard sync-async gap.
Error Feedback recovers more than half of the gap for most optimizers.}
\label{tab:ef_360m}

\setlength{\tabcolsep}{4pt}
\begin{tabular}{l|c|c|c|c}
\toprule
\textbf{Optimizer} & \textbf{Sync} & \multicolumn{3}{c}{\textbf{Async}} \\
\cmidrule{3-5}
                   &               & \textbf{EF} & \textbf{Sync Start} & \textbf{Standard} \\
\midrule
Muon    & 2.578 & \textbf{2.583}~{\tiny\textbf{(-71\%)}} & 2.589~{\tiny(-8\%)}   & 2.590 \\
Adan    & 2.641 & 2.653~{\tiny(+20\%)}          & 2.656~{\tiny(+50\%)}  & \textbf{2.651} \\
NorMuon & 2.574 & \textbf{2.579}~{\tiny\textbf{(-64\%)}} & 2.584~{\tiny(-29\%)} & 2.588 \\
AdaMuon & 2.579 & \textbf{2.588}~{\tiny\textbf{(-53\%)}} & 2.591~{\tiny(-29\%)}  & 2.596 \\
SOAP    & 2.581 & \textbf{2.590}~{\tiny\textbf{(-67\%)}} & 2.602~{\tiny(-22\%)}  & 2.608 \\
Lion    & 2.624 & 2.642~{\tiny(-40\%)}          & \textbf{2.639}~{\tiny\textbf{(-50\%)}} & 2.654 \\
NAdam   & 2.651 & 2.703~{\tiny(+21\%)}          & 2.695~{\tiny(+2\%)}   & \textbf{2.694} \\
MARS    & 2.615 & \textbf{2.657}~{\tiny\textbf{(-85\%)}} & 2.820~{\tiny(-27\%)}  & 2.897 \\
AdamW   & 2.612 & \textbf{2.640}~{\tiny\textbf{(-90\%)}} & 2.658~{\tiny(-84\%)} & 2.890 \\
\bottomrule
\end{tabular}

% \vspace{-1em}
\end{table}

\vspace{-5pt}
\subsection{Error Feedback}
\vspace{-3pt}
\label{sec:ef}

To further reduce the effect of stale gradients, we derive a lightweight update-level correction inspired by Error Feedback~\citep{seide2014first-ef,stich2019ef-delayed-theory1}.
\textcolor{black}{
At step $t$, standard delayed training applies the update $-u_{t-1}(g_{t-1})$.
Rather than viewing $g_{t-1}$ solely as the stale gradient available at the current step, we can also view it as the fresh gradient that was missing from the previous iteration.
Indeed, at step $t-1$, the algorithm \textit{actually applied} $-u_{t-2}(g_{t-2})$, while after $g_{t-1}$ becomes available we can see that the update we \textit{would have applied} with fresh information is $-u_{t-1}(g_{t-1})$.
The discrepancy between the desired and actual previous updates is therefore $u_{t-2}(g_{t-2}) - u_{t-1}(g_{t-1})$.
}
Adding this correction to the current delayed update gives
\begin{align}
\label{eq:ef_update}
x_{t+1} &= x_t - \underbrace{u_{t-1}(g_{t-1})}_{\text{Async Update}} + \underbrace{(u_{t-2}(g_{t-2}) - u_{t-1}(g_{t-1}))}_{\text{Error Correction}} \nonumber \\
        &= x_t - 2 \cdot u_{t-1}(g_{t-1}) + u_{t-2}(g_{t-2})
\end{align}
\textcolor{black}{
This update-level formulation has a close connection to SAPipe-WP~\citep{chen2022sapipe}, which arrives at a similar correction from a different data-parallel motivation.
We became aware of this connection after the main experimental study was completed.
In the one-step delayed abstraction, SAPipe-WP induces the same displacement as~\cref{eq:ef_update}.
We discuss this relation, including optimizer-state handling and empirical comparisons, in Appendix~\ref{app:sapipe}.
}

\cref{tab:ef_360m} and~\cref{fig:ef_barplot} show that this correction provides a more consistent benefit than the baseline strategies above.
For several robust optimizers, including Muon, AdaMuon, SOAP, and NorMuon, Error Feedback recovers roughly $50$--$70\%$ of the degradation introduced by delayed training.
It also substantially improves the most degraded AdamW-like runs, recovering $85$--$90\%$ of the gap for MARS and AdamW in our 360M experiments.
The method is not universally beneficial: it slightly degrades Adan and NAdam in this benchmark.
Nevertheless, across the full set of optimizers, it is the most reliable mitigation strategy we evaluate.
\textcolor{black}{
We additionally ablate the strength of the Error Correction term in~\cref{eq:ef_update} by multiplying it by a scalar coefficient $\lambda$.
The results suggest a U-shaped dependence on $\lambda$, with the optimum close to the default value $\lambda=1$, so we keep this default in all main experiments; see Appendix~\ref{app:ef_lambda_ablation}.
}

Two practical details are worth noting.
The update-level Error Feedback approach stores one additional model-sized buffer, but this adds only a small constant memory overhead in realistic LLM training setups; see Appendix~\ref{app:memory_overhead}.
A similar correction can also be applied to raw gradients before passing them to the optimizer, but this variant diverges in our experiments; see~\cref{fig:grad_ef}.

Overall, Error Feedback is the most consistent mitigation strategy we evaluate, as it reduces the sync-async gap for most optimizers and has only a small constant memory overhead.
These mitigation results also determine the recipe we use at scale: Section~\ref{sec:opt_delay} identifies Muon as a strong optimizer for one-step delayed training, and Error Feedback further reduces its remaining gap.
We therefore use Muon, with and without Error Feedback, in the large-scale validation.
Before scaling up, we first provide a theoretical analysis of delayed Muon with and without Error Feedback.

\vspace{-1em}
\section{Theoretical analysis} \label{sec:theory}

\begin{algorithm}[t]
  \caption{Delayed Muon}
  \label{alg:simple_delayed_muon}
  \begin{algorithmic}[1]
    \STATE {\bf input:} $\mathbf{X}_0, \mathbf{M}_0\in \mathbb{R}^{m \times n}$
    \STATE {\bf parameters:} stepsize $\eta > 0$, momentum $\mu \in (0,1)$, weight decay $\lambda \in (0,1)$, number of iterations $T$
    \FOR{$t=0,1,\ldots, T-1$}
    \STATE Compute gradient: $\mathbf{G}_{t-1} \gets \nabla f(\mathbf{X}_{t-1}, \xi_{t-1})$
    \STATE $\mathbf{M}_{t-1} \gets (1-\mu)\mathbf{M}_{t - 2} + \mu \mathbf{G}_{t - 1}$
    \STATE 
    $\mathbf{O}_{t-1} \gets \text{Newton-Schulz}(\mathbf{M}_{t-1})$
    \STATE 
    $\mathbf{U}_{t-1} \gets \eta \bigl(\mathbf{O}_{t-1} + \lambda \mathbf{X}_t \bigr)$
    \IF{Standard Async}
        \STATE $\mathbf{X}_{t+1} \gets \mathbf{X}_t - \mathbf{U}_{t-1}$
    \ELSIF{Error-Feedback (\cref{sec:ef})}
    \STATE
    $\begin{aligned}[t]
    \mathbf{X}_{t+1} \gets \mathbf{X}_t - 2\mathbf{U}_{t-1} + \mathbf{U}_{t - 2}\end{aligned}$
    \ENDIF
    \ENDFOR
    \STATE {\bf output:} $\mathbf{X}_T$
  \end{algorithmic}
\end{algorithm}

% The empirical results above suggest that Muon is a promising optimizer for one-step delayed training.
% We analyze Muon-style Linear Minimization Oracle (LMO) updates under gradient staleness.
% While convergence under delayed gradients has been studied for other algorithm families~\citep{mishchenko2022delayed-theory-3,koloskova2022delayed-theory-4}, theoretical guarantees for LMO-based methods remain largely unexplored in asynchronous settings.
% To the best of our knowledge, our result provides \textbf{the first convergence analysis of LMO-based optimization under gradient delay.}
% In the main text, we focus on Muon as the most relevant instance for our experiments, while Appendix~\ref{app:general_theory} gives the general formulation for arbitrary norms and the full convergence proofs for arbitrary delay $\tau \geq 0$.
The empirical results above suggest that Muon is a promising optimizer for one-step delayed training.
We analyze Muon-style Linear Minimization Oracle (LMO) updates under gradient staleness.
While convergence under delayed gradients has been studied for other algorithm families~\citep{mishchenko2022delayed-theory-3,koloskova2022delayed-theory-4}, theoretical guarantees for LMO-based methods under the fixed-delay updates considered here remain limited.
% \textcolor{blue}{To the best of our knowledge, this is the first convergence analysis of Muon with Error Feedback under gradient delay.}\footnote{\textcolor{blue}{Concurrent  subsequent work~\citet{sadiev2026ringmaster-lmo} studies asynchronous LMO methods with delay thresholding in heterogeneous server-worker settings; see~\cref{sec:rw} for a discussion.}}
\textcolor{black}{To the best of our knowledge, this is the first convergence analysis of LMO algorithms under gradient delay.}\footnote{\textcolor{black}{Concurrent  subsequent work~\citet{sadiev2026ringmaster-lmo} studies asynchronous LMO methods with delay thresholding in heterogeneous server-worker settings; see~\cref{sec:rw} for a discussion.}}
In the main text, we focus on Muon as the most relevant instance for our experiments, while Appendix~\ref{app:general_theory} gives the general formulation for arbitrary norms and the full convergence proofs for arbitrary delay $\tau \geq 0$.

% \textcolor{blue}{Change the formulation a bit, it's copied from Kovalev's paper}.

First, we theoretically formulate the optimization problem for Muon, following the setting covered in~\citep{kovalev2025understanding-muon-theory}:
\begin{equation} \label{eq:problem}
    \textstyle
    \min_{\mathbf{X}\in \mathbb{R}^{m \times n}} f(\mathbf{X})
\end{equation}
where $f(\cdot)\colon \mathbb{R}^{m \times n}\to\mathbb{R}$ is a bounded from below and differentiable objective function. 
% We equip $\mathcal{X}$ with the operator norm $\|\cdot\| = \|\cdot\|_{op}$ (spectral norm), whose dual norm is $\|\cdot\|_* = \|\cdot\|_{nuc}$ (nuclear norm).

\begin{assumption} \label{ass:main_assumptions}
For further theoretical analysis, we consider the following:
\begin{enumerate}[leftmargin=*, itemsep=0pt, topsep=2pt, parsep=0pt] 
    \item \textbf{Stochastic gradient estimator.} \\
    $\mathbb{E}_{\xi}[\nabla f(\mathbf{X};\xi)] = \nabla f(\mathbf{X})$, \\
    $\mathbb{E}_{\xi}[\|\nabla f(\mathbf{X};\xi) - \nabla f(\mathbf{X})\|_2^2] \leq \sigma^2$.
    \item \textbf{Smoothness.} $\|\nabla f(\mathbf{X}) - \nabla f(\mathbf{X}')\|_{\mathrm{nuc}} \leq L \|\mathbf{X}-\mathbf{X}'\|_{op}$.
    \item \textbf{Star convexity.} \\
    $f(\alpha\mathbf{X}^* + (1 - \alpha)\mathbf{X}) \leq \alpha f(\mathbf{X}^*) + (1-\alpha)f(\mathbf{X})$.
\end{enumerate}
for any $\mathbf{X}, \mathbf{X}' \in \mathbb{R}^{m,n}, \text{ where }\mathbf{X}^* \in \mathbb{R}^{m \times n}, \alpha \in (0, 1) \text{ and }\sigma \geq 0. $
\end{assumption}

% \begin{assumption} [Smoothness] \label{ass:muon_smoothness}
% We assume that function $f(\cdot)$ has Lipschitz gradient:
% \begin{align*}
%  \|\nabla f(\mathbf{X}) - \nabla f(\mathbf{X}')\|_{\mathrm{nuc}} \leq L \|\mathbf{X}-\mathbf{X}'\|_{op}
% \end{align*}
% for all  $\mathbf{X},\mathbf{X}' \in \mathbb{R}^{m \times n}$
% % where we denote dual norm $\|x\|_* = \sup\limits_{\|x'\| \leq 1}(\langle x, x' \rangle)$.
% \end{assumption}

% \begin{assumption} [Star Convexity] \label{ass:muon_star_conv}
% We assume $f(\mathbf{X})$ to be star-convex:
% \begin{align*}
%  f(\lambda\mathbf{X}^* + (1 - \lambda)\mathbf{X}) \leq \lambda f(\mathbf{X}^*) + (1-\lambda)f(\mathbf{X})
% \end{align*}
% for all $\mathbf{X} \in \mathbb{R}^{m \times n}$ and $\lambda \in (0,1).$
% \end{assumption}

These assumptions have been widely adopted for the analysis of many stochastic gradient optimization algorithms \citep{gower2019sgd, horvath2023stochastic, kovalev2025understanding-muon-theory}

Then, we introduce Muon version with gradient delay, formulated in \cref{alg:simple_delayed_muon}.

\begin{theorem} [Delayed Muon with Weight Decay] \label{thm:simple_delayed_muon_wd}
Let Assumption ~\ref{ass:main_assumptions} hold, and let $\mathbf{M}_0 = \mathbf{G}(\mathbf{X}_0)$. Then the iterations of Algorithm \ref{alg:simple_delayed_muon} with Weight Decay $\lambda > 0$ satisfy:
\begin{align*}
    &\mathbb{E}[f(\mathbf{X}_T) - f(\mathbf{X}^*)] \leq (1-\lambda)^K (f(\mathbf{X}_0) - f(\mathbf{X}^*)) \\
    &+ 2\eta\left(\frac{\rho\sigma}{\mu} + \frac{\sqrt{2\mu} \sqrt{\rho^2\sigma^2 + {8 (L\eta)^2}}}{\lambda}\right) + \frac{4L\eta^2}{\lambda}(1 + \frac{1}{\mu}),
\end{align*}
% \begin{equation}
% \begin{gathered}
% \notag
%     \mathbb{E}[f(\mathbf{X}_T) - f(\mathbf{X}^*)] \leq (1-\lambda)^K (f(\mathbf{X}_0) - f(\mathbf{X}^*))  + \\
% + 2\eta\left(\frac{\rho\sigma}{\mu} + \frac{\sqrt{2\mu} \sqrt{\rho^2\sigma^2 + {8 (L\eta)^2}}}{\lambda}\right) + \frac{4L\eta^2}{\lambda}\left(1 + \frac{1}{\mu}\right),
% \end{gathered}
% \end{equation}
where $\rho = \sqrt{\min(m, n)}$, and $\eta$, $\lambda$ satisfy the following:
\begin{equation}\label{eq:eta_beta}
    \textstyle
    \eta \geq \lambda \max\left\{\|\mathbf{X}_0\|, \|\mathbf{X}^*\|\right\}
\end{equation}
\end{theorem}

\begin{proof}
    This result is a direct corollary of our general convergence guarantee for delayed LMO algorithms, presented in \cref{thm:delayed_muon_wd}. The proof follows by instantiating the general theorem with the specific choices for the Muon optimizer: setting the regularizer $R(\mathbf{X}) \equiv 0$, using the operator norm $\|\cdot\|_{\mathrm{op}}$ and its dual, the nuclear norm $\|\cdot\|_{\mathrm{nuc}}$, and substituting the concrete norm equivalence constant $\rho = \sqrt{\min(m, n)}$. For a complete derivation of the general case, we refer the reader to Appendix~\ref{app:general_theory}.
\end{proof}

\vspace{-10pt}
\textbf{Discussion}. The main difference between the obtained estimation for the delayed setup and the synchronous one lies in the noise bound, specifically $\sqrt{2\mu}\sqrt{\rho^2\sigma^2 + 8(L\eta)^2}$ for the delayed setup versus $\sqrt{\mu}\rho\sigma$ for the standard one. Following Corollary 2 from \citep{kovalev2025understanding-muon-theory}, where $\eta = \mathcal{O}\left(\mathrm{min}\{\frac{\epsilon}{L}, \frac{\epsilon^2}{\rho^2\sigma^2L} \}\right)$ and $\mu = \mathcal{O}\left(\mathrm{min}\{1, \frac{\epsilon^2}{\rho^2\sigma^2} \}\right)$, one can show that the additional term caused by delayed gradients is generally small.
% , which previously occurred only from the stochastic oracle noise term and now it's enlarged due to gradient delay.
% \textbf{\textcolor{red}{SAM: MAYBE ADD DISCUSSION WHY THE ADDITIONAL TERM IS SMALL, SO THE CONVERGENCE IS NOT AFFECTED.}}

% We further investigate the reasons for the optimizers behavior in Section \ref{} \textcolor{blue}{THE EXPLANATION PART}.

\vspace{-1em}
\section{Large Scale Experiments} \label{sec:exps}

Having studied optimizer robustness, mitigation strategies, and theoretical guarantees under one-step delay, we now ask whether these findings transfer to realistic pre-training runs.
This question is particularly important because the throughput benefits of Async PP are most relevant in large distributed training regimes, where pipeline bubbles translate into substantial wasted accelerator time (see~\cref{app:runtime_overhead}).
Since benchmarking all optimizers at this scale is prohibitively expensive, we focus on \textbf{Muon}: across the previous sections, it combines strong synchronous performance, robustness to one-step delay, compatibility with Error Feedback, and convergence guarantees under staleness.
We scale our evaluation to 2B- and 10B-parameter MoE~\citep{shazeer2017outrageously-moe-orig} models to test whether Async PP can match synchronous training quality in realistic training scenarios.
\vspace{-5pt}
\subsection{2B MoE Experiments}
\vspace{-3pt}

We train a MoE model with 2B total parameters and 500M active parameters for training horizons ranging from $50$B to $200$B tokens.
To keep the global batch size close to the optimum as the training horizon increases, we scale it according to $B \propto D^{0.58}$ following~\citet{li2025predictable-scale-part1}, using $B=1$M tokens at $D=50$B as the anchor point.

\textcolor{black}{
\textbf{Learning Rate Robustness.}
We first check whether the large-scale async comparison is sensitive to the changes in the learning rate in the near-optimal range.
For each training horizon $D$, we sweep five peak learning-rate values around the expected optimum and report the full results in~\cref{fig:lrs_at_scale}.
Within this local range, synchronous and asynchronous losses move similarly as the learning rate changes, and the sync-async gap remains comparable across the tested values.
These sweeps therefore verify that the 2B results are not an artifact of a single learning-rate choice, consistent with the learning-rate sensitivity observed on smaller models in~\cref{fig:lr_sweep}.
}

\begin{table}[t]
    \centering
    \caption{\textcolor{black}{Large-scale pretraining results: final validation loss for 10B MoE model trained for 200B tokens on the Fine-Web dataset.}}
    \label{tab:large_scale}
    \begin{tabular}{lccc}
    \toprule
    \textbf{Optimizer} & \textbf{Sync} & \textbf{Async} & \textbf{Async + EF} \\
    \midrule
    Muon & $\mathbf{1.906}$ & 1.911 & $\mathbf{1.906}$ \\
    \bottomrule
    \end{tabular}
    % \vspace{-1.3em}
\end{table}
\textbf{Scaling with Training Horizon.}
We next test whether the effect of staleness grows as training progresses to longer horizons and lower losses.
A natural concern is that delayed gradients may become increasingly harmful as training approaches convergence, where the optimization trajectory may require more accurate gradient information to continue reducing the loss.
Using the learning-rate sweeps described above, we take the best validation loss at each training horizon and observe nearly parallel synchronous and asynchronous scaling curves in~\cref{fig:scaling_d}.
This indicates that one-step staleness does not introduce a growing optimization barrier between $50$B and $200$B tokens.
Error Feedback also consistently recovers a substantial fraction of the remaining gap across all tested scales.
While verifying this behavior at trillion-token scale remains an important direction for future work, the absence of gap growth up to $200$B tokens supports the scalability of Async PP in realistic pre-training runs.
\begin{figure}[t]
    \centering
        \includegraphics[width=0.5\textwidth]{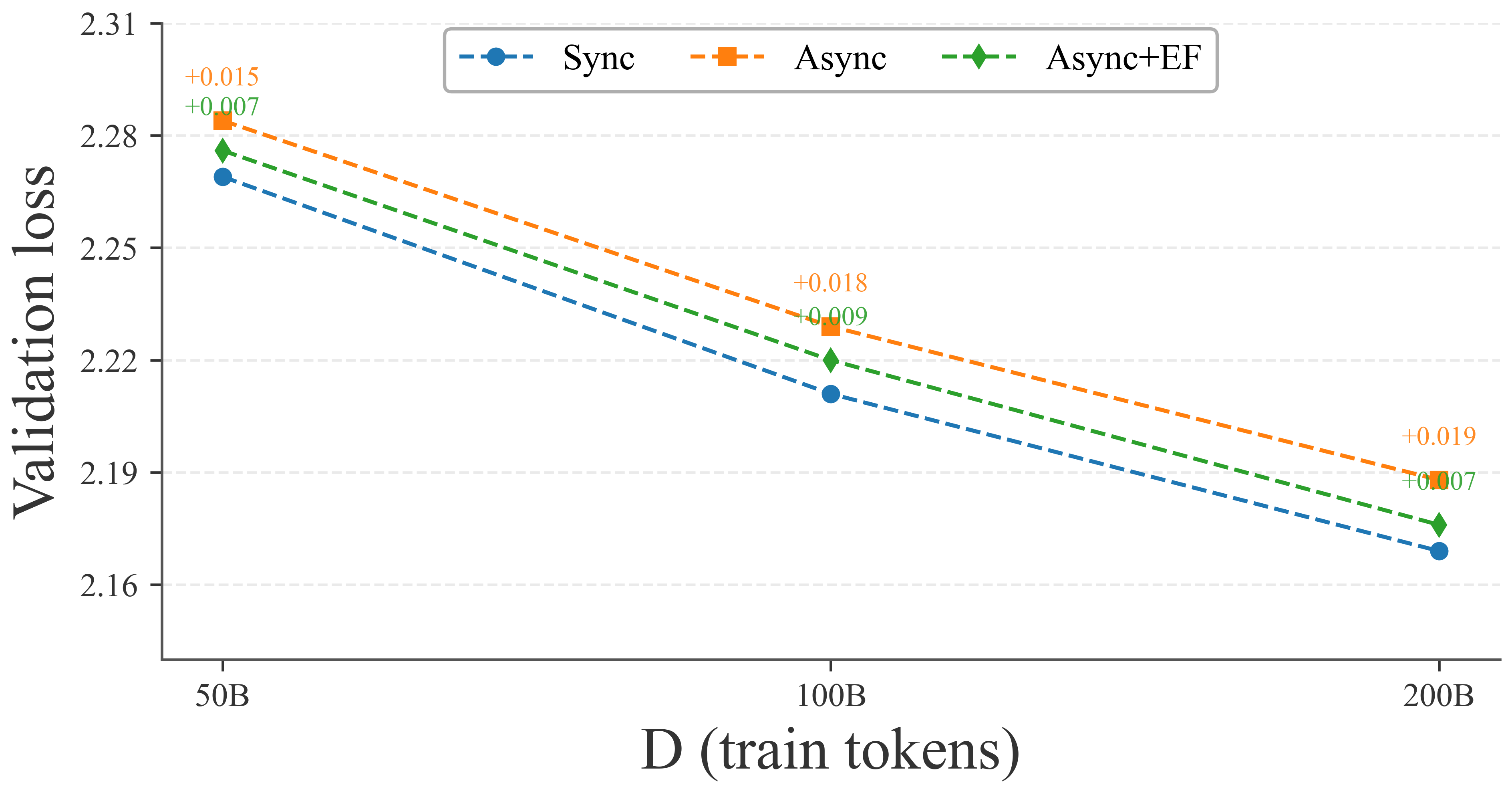}
    \hfill
    \vspace{-10pt}
    \caption{Validation loss of the 2B MoE model across training horizons.
The synchronous and asynchronous scaling curves remain nearly parallel from $50$B to $200$B tokens, indicating no growth of the sync-async gap with longer training; Error Feedback consistently reduces the remaining gap.}
    \label{fig:scaling_d}
    % \vspace{-1.5em}
    \vspace{-2pt}
\end{figure}
% \vspace{-5pt}
\subsection{10B MoE Experiments}
\vspace{-3pt}
To test whether our findings hold at the largest scale available in our experiments, we train a 10B-parameter Mixture-of-Experts model.
The model uses a Qwen3-Next-like architecture~\citep{qwen3next} with Gated Delta Net layers~\citep{yang2024gated-delta-net}.
We train for 200B tokens with a global batch size of 4M tokens and a peak learning rate of 0.00225, comparing the synchronous baseline against standard Async PP and Async PP with Error Feedback.

The training loss trajectories are shown in~\cref{fig:10b}, and the final validation losses are reported in~\cref{tab:large_scale}.
Standard Async PP remains highly competitive at this scale, incurring only a small final loss gap relative to the synchronous baseline ($1.911$ vs. $1.906$).
\textbf{With Error Feedback, Async PP closes this gap entirely, matching the synchronous final loss} of $1.906$ while using the exact same hyperparameters.
Notably, the relative degradation at 10B scale is smaller than in our smaller dense-model experiments, suggesting that one-step delayed optimization remains robust in realistic large-scale MoE pre-training.
Both asynchronous runs remain stable throughout training, despite a small lag during the early phase. \textcolor{black}{Furthermore, downstream benchmarking across a diverse suite of tasks confirms that this identical validation loss translates to equivalent downstream task performance (see~\cref{app:benchmarks}).}

To the best of our knowledge, this is the first successful demonstration of Async PP on a model of this scale without quality degradation.
Together with the throughput motivation of Async PP, these results highlight the practical potential of asynchronous pipeline parallelism for large-scale LLM pre-training.

% \begin{figure}[t]
%     \centering
%         \includegraphics[width=0.5\textwidth]{figures/d_figure_redacted.png}
%     \hfill
%     \caption{Validation loss of the 2B MoE model across training horizons.
% The synchronous and asynchronous scaling curves remain nearly parallel from $50$B to $200$B tokens, indicating no growth of the sync-async gap with longer training; Error Feedback consistently reduces the remaining gap.}
%     \label{fig:scaling_d}
%     % \vspace{-1.5em}
% \end{figure}

\begin{figure*}[t]
\centering
\includegraphics[width=0.9\textwidth]{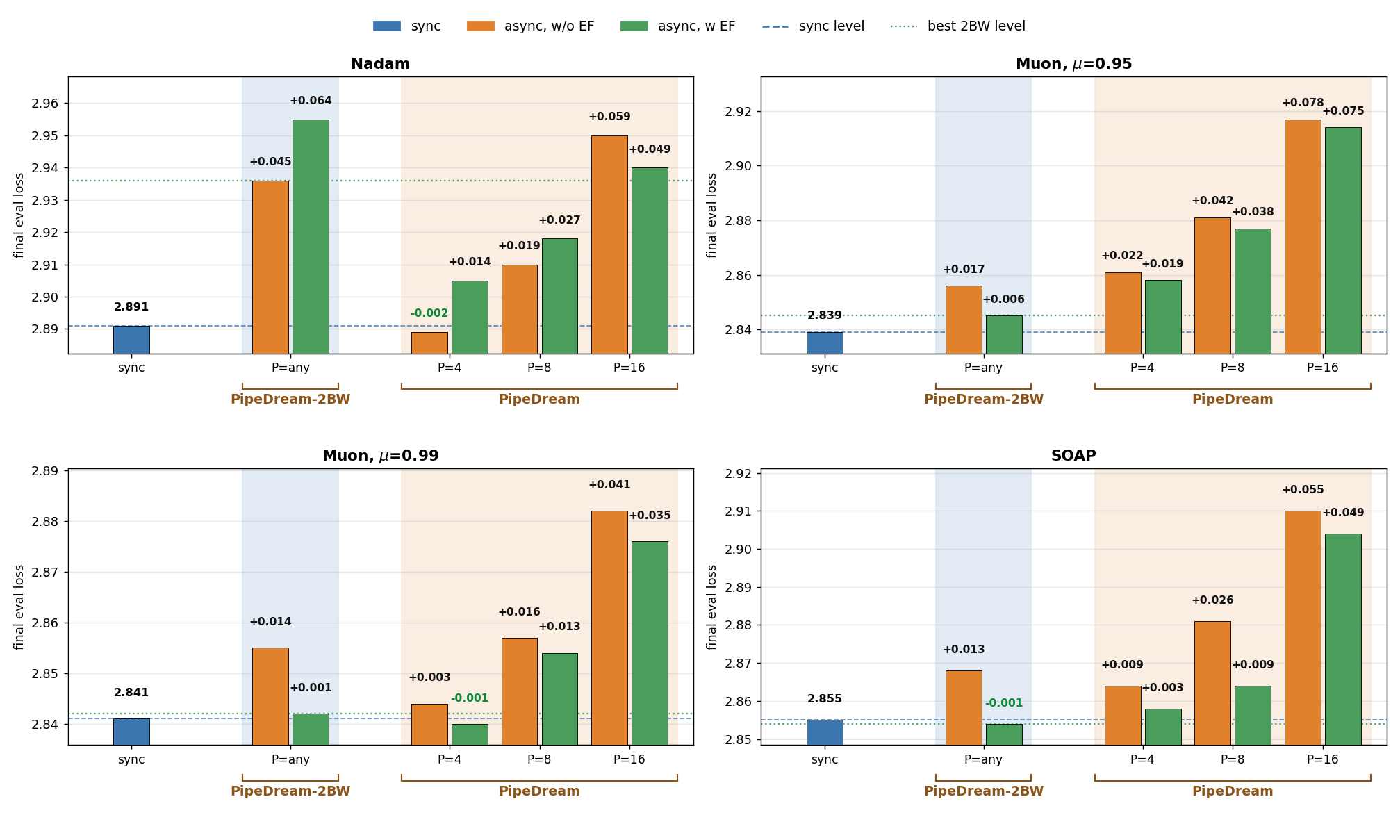}
\vspace{-10pt}
\caption{
Comparison of synchronous training, PipeDream-\texttt{2BW} with constant one-step delay, and the original PipeDream schedule with variable delay on the 135M model.
Here $P$ denotes the number of pipeline stages in the original PipeDream schedule; PipeDream-\texttt{2BW} has a constant one-step delay independent of $P$.
As $P$ increases, original PipeDream progressively degrades relative to the corresponding PipeDream-\texttt{2BW} runs.
This illustrates that variable delay remains a major source of degradation at larger pipeline depths.
}
\label{tab:pipedream_comparison}
\vspace{-10pt}
\end{figure*}

\vspace{-7pt}
\section{Comparison with PipeDream} \label{sec:exp_pipedream}
\vspace{-3pt}

While the main experiments use PipeDream-\texttt{2BW}~\citep{pipedream-2bw}, prior work on Async PP for language model pre-training~\citep{ajanthan2025nesterov-async-pp} used the original PipeDream schedule~\citep{narayanan2019pipedream2019-1f1b}.
In that setting, \citet{ajanthan2025nesterov-async-pp} found that Nadam can perform reasonably at small pipeline depths but degrades substantially as the number of stages increases.
This leaves open question whether the degradation is primarily due to the optimizer choice or to the PipeDream-style variable-delay schedule itself.
In particular, if the instability is mostly optimizer-driven, the more robust optimizers identified in~\cref{sec:opt_delay} could make the original PipeDream schedule a practical alternative.
We therefore compare against the original PipeDream schedule using these optimizers and also evaluate whether EF can further reduce the degradation.
\vspace{-5pt}
\subsection{Experimental Results}
\vspace{-3pt}
We evaluate the original PipeDream schedule with $P \in \{4,8,16\}$ stages using Muon, SOAP, and Nadam with the best hyperparameter configurations from~\cref{sec:opt_delay}.
To account for the schedule's mechanics, we set the effective batch size per update to $B_{\text{sync}} / P$.
This smaller per-update batch arises because the original PipeDream schedule performs an optimizer step after every backward pass, whereas PipeDream-\texttt{2BW} accumulates gradients over a full minibatch before applying an update.
Thus, unlike PipeDream-\texttt{2BW}, original PipeDream does not preserve the same effective global batch size per weight update as the synchronous baseline.
\textcolor{black}{
This makes the comparison intentionally faithful to the original schedule, but also highlights a practical complication: original PipeDream may require separate batch-size and learning-rate calibration.
PipeDream-\texttt{2BW}, in contrast, preserves the same effective per-update batch size as the synchronous baseline, making synchronous scaling-law estimates for optimal and critical batch sizes~\citep{zhang2024does-cbs1,merrill2025critical-cbs2} a more natural starting point rather than changing the batch-size regime by construction.
}

\textcolor{black}{The results in~\cref{tab:pipedream_comparison} show that the main trends from~\cref{sec:opt_delay,sec:ef} carry over to the original PipeDream schedule.}
Muon and SOAP are generally more robust than Nadam, especially at larger pipeline depths.
For Muon, increasing the momentum from $\mu=0.95$ to $\mu=0.99$ improves performance at every pipeline depth, both with and without EF, further supporting the role of momentum observed in~\cref{sec:section 2 opt_delay_hyperparameter_sensitivity}.
Error Feedback provides small but consistent improvements for Muon and SOAP, while leading to slight degradation for Nadam, in line with the pattern observed in~\cref{tab:ef_360m}.
At shallow pipeline depth, these improvements can nearly close the gap: for $P=4$, Muon with $\mu=0.99$ and EF reaches $2.840$, matching the synchronous baseline within noise ($2.841$), while SOAP reaches $2.858$ compared to its synchronous baseline of $2.855$.

However, these gains do not remove the scaling issue of the original PipeDream schedule.
As the number of stages increases, all methods degrade substantially: at $P=16$, even the best configuration, Muon with $\mu=0.99$ and EF, loses more than $0.03$ relative to its synchronous baseline.
These results suggest that robust optimizers and EF can make the original PipeDream schedule viable for shallow pipelines, but are insufficient at larger pipeline depths.
\textcolor{black}{
Taken together, these results reinforce the central role of PipeDream-\texttt{2BW}: robust optimizers can partially compensate for the original PipeDream schedule at small pipeline depths, but scalable Async PP appears to benefit substantially from the constant-delay guarantees provided by PipeDream-\texttt{2BW}.
}

% \vspace{-0.6em}
\vspace{-7pt}
\section{Related work} \label{sec:rw}
\vspace{-3pt}
\textbf{Asynchronous Pipeline Parallelism.}
The domain of Asynchronous Pipeline Parallelism was established by PipeDream~\citep{narayanan2019pipedream2019-1f1b}, which utilized weight stashing to ensure consistent weights for forward and backward passes, albeit yielding variable gradient staleness.
Subsequent approaches like PipeMare~\citep{yang2021pipemare}, SpecTrain~\citep{spectrain}, XPipe~\citep{guan2019xpipe} and PipeOptim~\citep{guan2025pipeoptim} prioritized memory efficiency by removing stashing; however, these methods fundamentally compromise optimization integrity by allowing forward and backward passes to execute on different model versions.
\textcolor{black}{
In the context of language modeling,~\citet{ajanthan2025nesterov-async-pp} were the first to demonstrate the viability of Async PP for decoder-only LLM pre-training, proposing to use the Nesterov look-ahead as a delay correction in weight space and instantiating this idea with NAdam and large first-moment momentum.
More recently,~\citet{jung2026mitigating} studied the sensitivity of AdamW to asynchronous pipeline delay and proposed basis rotation as a mitigation.
However, both works build on the original PipeDream schedule and therefore inherit its variable-delay behavior, which itself is a significant source of degradation.
% In contrast, we adopt the PipeDream-\texttt{2BW}~\citep{pipedream-2bw} scheme, which accumulates gradients over $M$ micro-batches rather than performing parameter updates after every step.
% This design achieves a constant delay of $1$ across all pipeline stages at the cost of maintaining a single additional copy of model parameters in memory.
In addition, rather than focusing on a single optimizer, our work provides a broader comparison of optimizer robustness under delay.
}

\textbf{Optimizer Benchmarking.}
\textcolor{black}{Recently, the community has placed increased emphasis on the empirical evaluation of optimization algorithms for LLMs.
Studies like~\citet{semenov2025benchmarking-optimizers-benchmarking-1} and~\citet{wen2025fantastic-optimizers-benchmarking-1} provide extensive benchmarking regarding convergence and performance, while~\citet{vlassis2025beyond-optimizers-benchmarking-3} explores optimizer interactions with quantization.
We complement this line of work by conducting a comprehensive benchmark of optimizers specifically under the constraints of asynchronous gradient delay.}

\textbf{Error Feedback.}
Error Feedback was originally introduced by~\citet{seide2014first-ef} to compensate for quantization errors.
Since then, it has been extensively utilized in the context of gradient compression~\citep{stich2018sparsified-ef1,alistarh2018convergence-ef2,karimireddy2019-ef3}.
Recently,~\citet{gruntkowska2025-ef-muon} investigated EF with the Muon optimizer under compression constraints, while~\citet{stich2019ef-delayed-theory1} considered the interplay between EF and gradient delays.
Our work uniquely synthesizes these directions by applying EF specifically to address the staleness in gradients.

% \textbf{Optimization with Delayed Gradients.}
% The theoretical foundations of optimization under gradient delays are well-established~\citep{butnariu2001delayed-theory-1,agarwal2011delayed-theory-2,mishchenko2022delayed-theory-3,koloskova2022delayed-theory-4}, with research in this domain continuing to evolve~\citep{maranjyan2025delayed-theory-5}.
\textbf{Optimization with Delayed Gradients.}
The theoretical foundations of optimization under gradient delays are well-established~\citep{agarwal2011delayed-theory-2,mishchenko2022delayed-theory-3,koloskova2022delayed-theory-4}, with research in this domain continuing to evolve~\citep{maranjyan2025delayed-theory-5}.
Stale updates have also been studied in systems-motivated distributed optimization frameworks, such as Pipe-SGD~\citep{li2018pipe}, which pipelines AllReduce-based data-parallel training and provides convergence guarantees for convex and strongly convex objectives.
\textcolor{black}{
More closely related to our mitigation study, SAPipe~\citep{chen2022sapipe} introduces staleness in data-parallel training to overlap gradient synchronization with computation, and uses weight prediction and delay compensation to mitigate the resulting delay.
Although the system's settings differ from Async PP, its weight-prediction variant is closely related to our update-level EF correction.
We discuss this connection, including optimizer-state handling and empirical comparisons, in Appendix~\ref{app:sapipe}.}
\textcolor{black}{
Concurrent subsequent work~\citet{sadiev2026ringmaster-lmo} studies asynchronous LMO optimization in heterogeneous server-worker systems.
Their setting is complementary to ours: Ringmaster LMO handles variable delays caused by heterogeneous worker runtimes via delay thresholding, whereas we focus on the fixed-delay regime induced by PipeDream-\texttt{2BW} in Async PP and cover the EF correction used in our experiments.
% Experimentally, their evaluation focuses on stochastic quadratic problems and simulated-worker NanoChat pre-training, whereas ours studies large-scale LLM pre-training under Async PP, including optimizer benchmarking, staleness mitigation, and validation up to 10B parameters.
}
However, while these studies provide rigorous convergence guarantees, they often lack large-scale experimental validation.
We distinguish our work by conducting the first extensive empirical investigation of delayed optimization across a diverse set of modern optimizers specifically in the context of LLM training.

% \vspace{-0.5em}
% \section{Discussion} \label{sec:discussion}

% This work presents a comprehensive analysis of optimizer dynamics under gradient staleness, highlighting the sensitivity of standard algorithms like AdamW contrasted with the stability of modern alternatives such as Muon.
% By utilizing PipeDream-\texttt{2BW} alongside a proposed Error-Feedback correction, we establish a framework that effectively mitigates convergence degradation.
% Our theoretical analysis provides convergence guarantees for these methods, while \textit{empirical validation on a 10B parameter Mixture-of-Experts model} demonstrates performance parity with synchronous baselines.
% These results indicate that with appropriate algorithmic choices, asynchronous pipeline parallelism offers a viable and efficient pathway for large-scale model training.

\color{black}
\vspace{-10pt}
\section{Discussion and Limitations}
\vspace{-3pt}
This work revisits the role of one-step gradient delay in asynchronous pipeline-parallel LLM pre-training.
Our results suggest that the degradation commonly associated with Async PP is not an unavoidable consequence of staleness itself, but depends strongly on the optimizer and schedule.
In particular, using a constant-delay  PipeDream-\texttt{2BW} schedule with the robust optimizers such as Muon and lightweight update-level correction, can make Async PP training closely match synchronous baselines.
The \textit{10B MoE experiment} provides evidence that this conclusion can hold at realistic scale: Async PP with EF matches the synchronous final validation loss while using the same hyperparameters.

% Several limitations remain.
% First, although our sweeps show that larger primary momentum consistently improves robustness to one-step delay, we do not yet have a complete mechanistic explanation for this effect.
% Second, our largest run is limited to a 10B MoE model trained for 200B tokens; larger models, trillion-token horizons, or later stages closer to convergence may reveal additional effects.
% Third, our \{batch size, learning rate\} grids are limited to several optimizers on the 135M model due to computational cost.
% Finally, WPipe-style schedules appear promising, but we only study them in limited appendix-scale experiments.
% Overall, our results show that one-step delay is not a fundamental barrier for large-scale Async PP, while leaving optimizer dynamics, batch-size choice, and schedule design as important directions for future work.
% \color{black}

However, several limitations still remain.
For example, we lack a complete mechanistic explanation for why exactly higher momentum improves robustness.
Additionally, our \{batch size, learning rate\} grids are restricted to the 135M model.
Finally, WPipe-style schedules appear promising, but we only study them in limited appendix-scale experiments.
Overall, our results show that one-step delay is not a fundamental barrier for large-scale Async PP, while leaving optimizer dynamics and batch-size choice as important directions for future work.

\color{black}

\section*{Acknowledgements}

This work was conducted while Philip Zmushko was affiliated with Yandex and BRAIn Lab; he is currently affiliated with ISTA.
The work of Egor Petrov was supported by the Ministry of Economic Development of the Russian Federation (agreement No.~139-15-2025-013, dated June 20, 2025, IGK 000000C313925P4B0002).

We thank Aleksandr Beznosikov, Alexander Mazitov and our colleagues from Yandex Research, Yandex, and BRAIn Lab for fruitful discussions.
\section*{Impact Statement}

This paper presents work whose goal is to advance the field of Machine Learning. There are many potential societal consequences of our work, none which we feel must be specifically highlighted here.

\bibliography{icml2026}
\bibliographystyle{icml2026}

\newpage
\onecolumn
\appendix

\section{Hyperparameter Sensitivity Results and Other Additional Experiments}
\label{app:hyperparameter_sensitivity}
\color{black}
This section provides the detailed hyperparameter sweeps supporting the summary in~\cref{sec:opt_delay}, separating ordinary one-dimensional sweeps from the special \{batch-size, learning-rate\} interaction discussed in the main text.

\subsection{One-Dimensional Hyperparameter Sweeps}
\label{app:other_hyperparameters}

In~\cref{sec:section 2 opt_delay_hyperparameter_sensitivity}, we identify the primary momentum decay coefficient as the most consistent hyperparameter controlling robustness to one-step delay.
Here, we provide additional sweeps for the other hyperparameters considered in our study.
All experiments in this subsection are performed on the 135M model.
Unless stated otherwise, each sweep varies a single hyperparameter while keeping the rest of the training recipe fixed, and synchronous and one-step delayed runs always use the same hyperparameter value.
% The goal of these experiments is not to fully retune every optimizer, but to identify whether a given hyperparameter produces a systematic change in the sync-async gap.

% Overall, we find that several hyperparameters affect delayed training through the general stability of the training recipe.
% For example, overly aggressive learning rates, insufficient warmup, or extreme weight decay values can increase the gap, especially when the corresponding synchronous run is also less stable.
% However, these effects are either optimizer-dependent or relatively small for robust optimizers.
% The clearest optimizer-independent trend remains the one reported in the main text: increasing the primary momentum decay coefficient consistently improves robustness to one-step delay.

\textbf{Learning rate.}
We first sweep the peak learning rate while keeping all other hyperparameters fixed.
The results in~\cref{fig:lr_sweep} show a mild trend: smaller learning rates tend to slightly reduce the sync-async gap, while overly large learning rates can increase it, consistent with the broader pattern that delayed updates amplify instability in already fragile regimes.
At the same time, lowering the learning rate also changes absolute synchronous quality, so this should be interpreted as a stability--quality trade-off rather than a universal recommendation.

\begin{figure*}[h]
    \centering
    \includegraphics[width=0.8\textwidth]{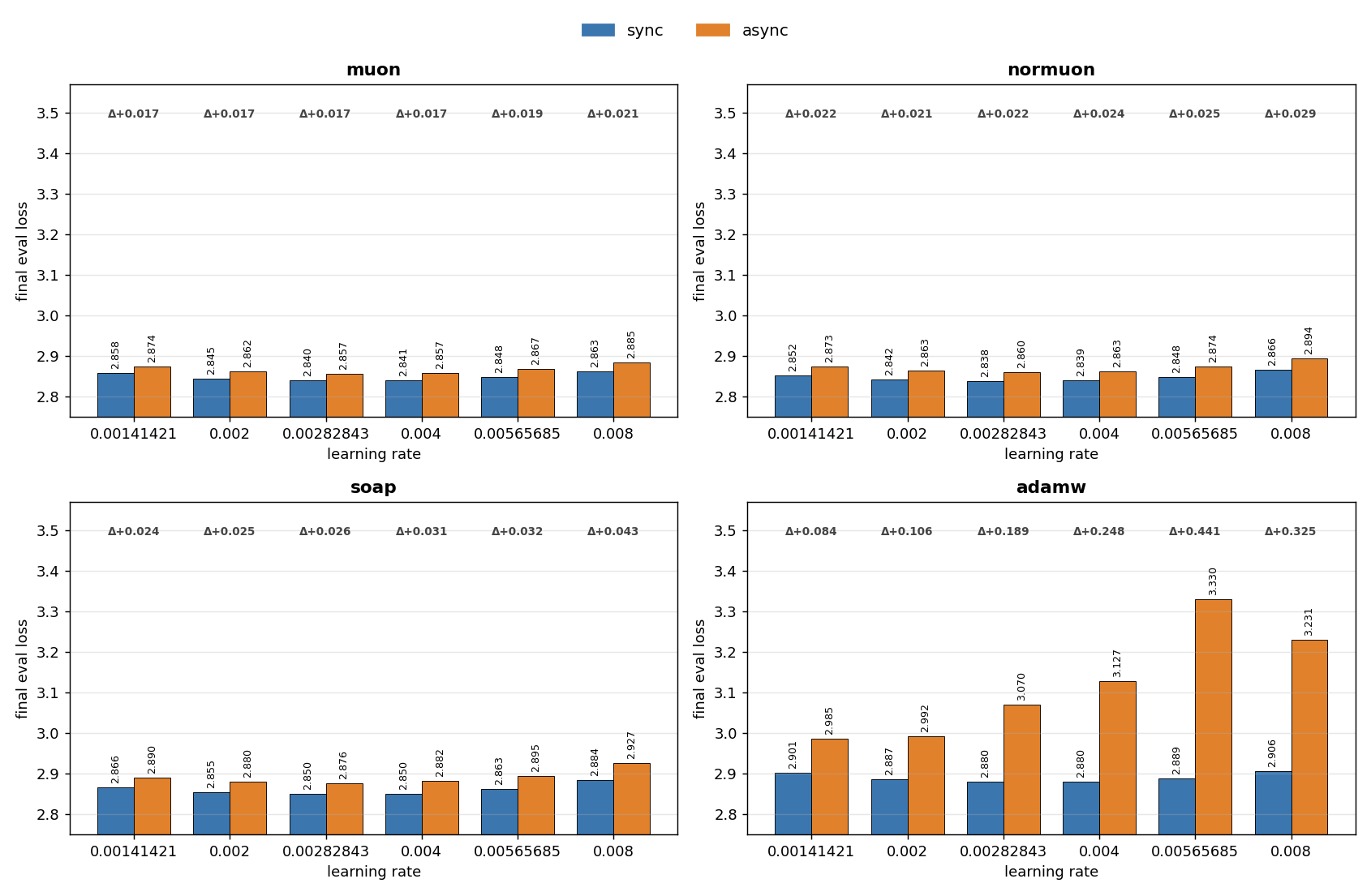}
    \caption{
    Effect of peak learning rate on synchronous and one-step delayed training on the 135M model.
    Bars show final validation loss, and annotations indicate the sync-async gap.
    Lower learning rates mildly reduce the gap for robust optimizers, while large learning rates can substantially worsen delayed training for AdamW.
    }
    \label{fig:lr_sweep}
\end{figure*}

\textbf{Weight decay.}
We next vary weight decay.
As shown in~\cref{fig:weight_decay_sweep}, moderate changes around the default have a limited effect for most robust optimizers, but extreme values can be harmful.
For Muon, NorMuon, and SOAP, large weight decay worsens both absolute loss and, to a lesser extent, the sync-async gap.
Interestingly, for NorMuon and AdamW, small weight decay values lead to divergence.
These results suggest that the effect is optimizer-dependent.

\begin{figure*}[h]
    \centering
    \includegraphics[width=0.8\textwidth]{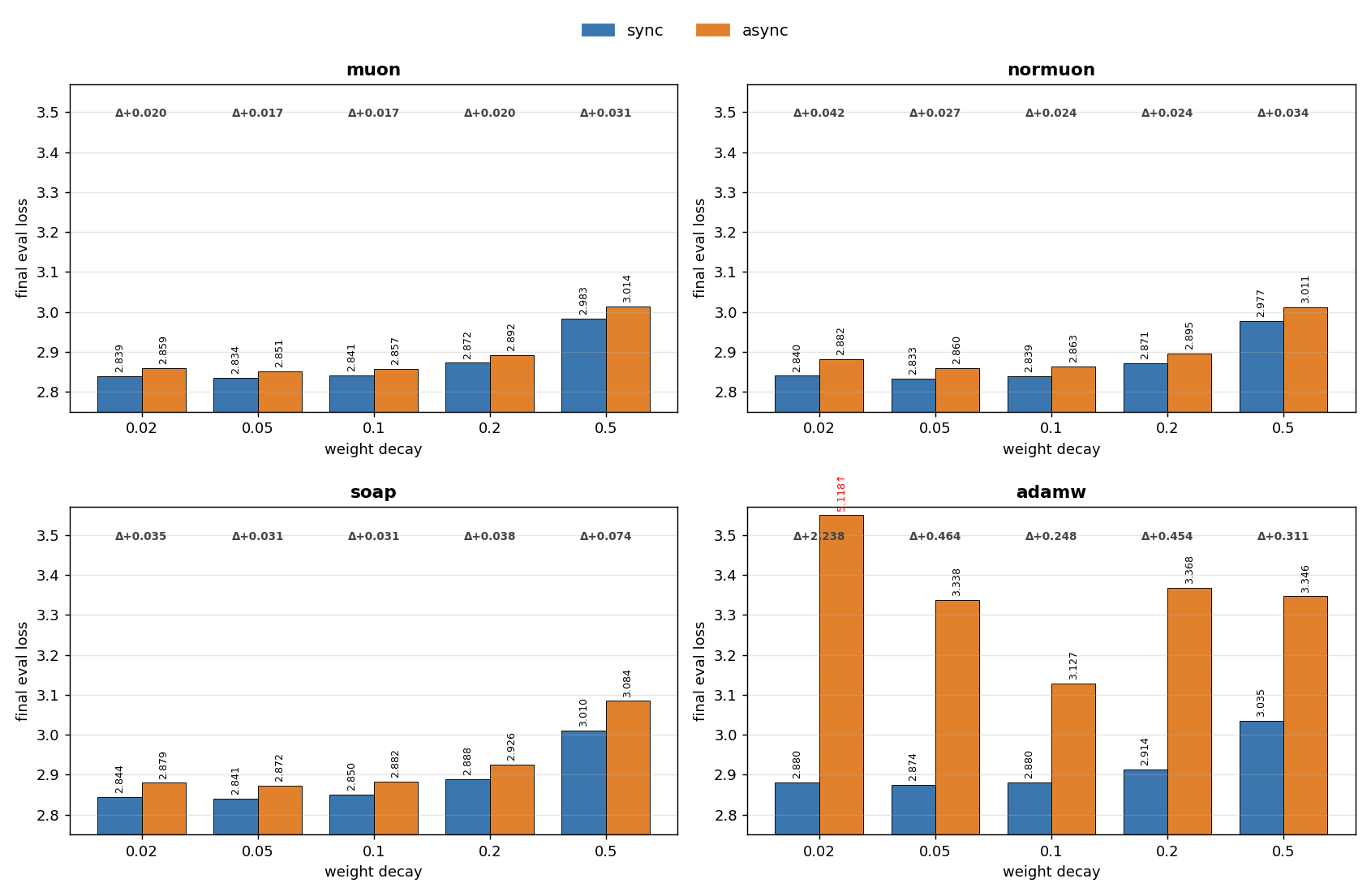}
    \caption{
    Effect of weight decay on synchronous and one-step delayed training on the 135M model.
    Moderate values have limited effect for robust optimizers, while extreme values can substantially worsen delayed training.
    AdamW is especially sensitive to small weight decay in this sweep.
    }
    \label{fig:weight_decay_sweep}
\end{figure*}

\FloatBarrier
\textbf{Warmup length.}
We also vary the number of learning-rate warmup steps.
Importantly, this is the standard learning-rate warmup: the one-step delay is enabled from the first training step in all runs, and only the learning-rate schedule is changed.
As shown in~\cref{fig:warmup_sweep}, increasing warmup length mildly reduces the sync-async gap.
The effect is small for Muon, NorMuon, and SOAP, but more visible for AdamW.
This is consistent with the interpretation that smoother early optimization helps delayed training.

\begin{figure*}[!h]
    \centering
    \includegraphics[width=0.8\textwidth]{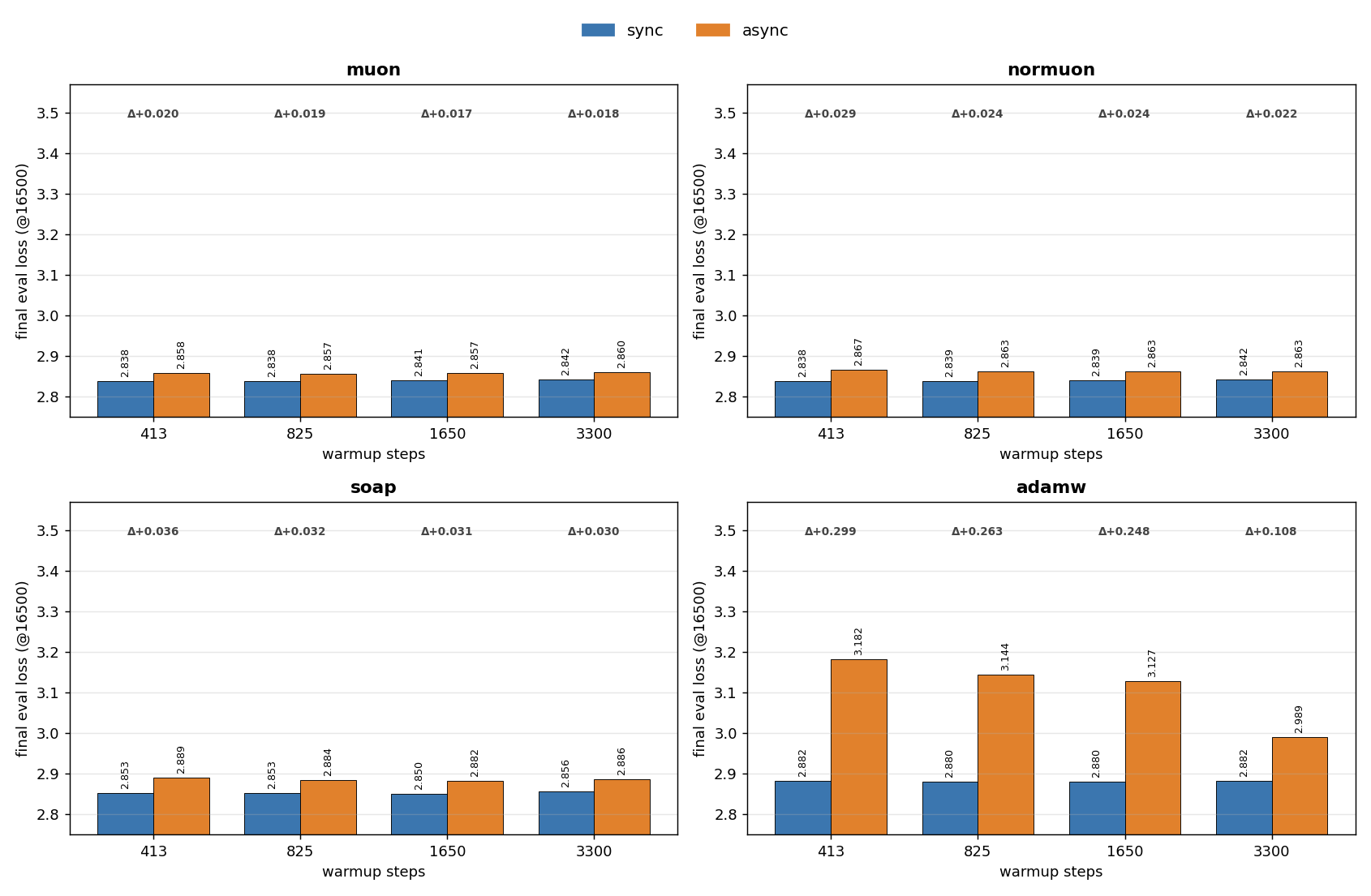}
    \caption{
    Effect of learning-rate warmup length on one-step delayed training on the 135M model.
    The delay is enabled from the first step in all runs.
    Longer warmup mildly reduces the sync-async gap, with the largest visible effect for AdamW.
    }
    \label{fig:warmup_sweep}
\end{figure*}

\textbf{Gradient clipping.}
We sweep the global gradient clipping threshold in~\cref{fig:grad_clip_sweep}.
Across the tested values, gradient clipping has little systematic effect on the sync-async gap.

\begin{figure*}[!htbp]
    \centering
    \includegraphics[width=0.8\textwidth]{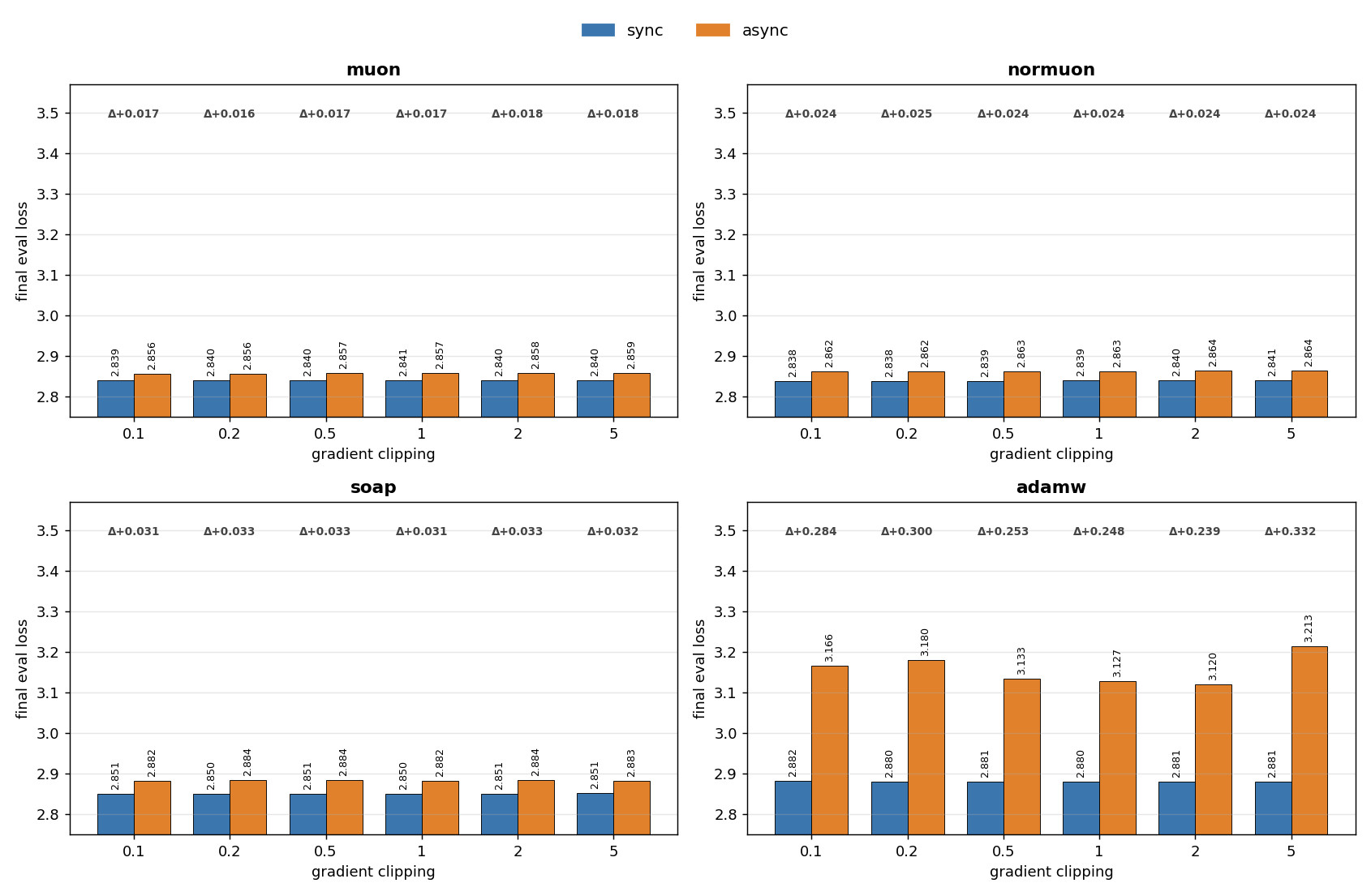}
    \caption{
    Effect of gradient clipping threshold on synchronous and one-step delayed training on the 135M model.
    The sync-async gap is largely insensitive to the clipping threshold for the robust optimizers tested here.
    }
    \label{fig:grad_clip_sweep}
\end{figure*}

\FloatBarrier
\textbf{Learning-rate scheduler.}
We compare several learning-rate schedules in~\cref{fig:scheduler_sweep}, including cosine decay with final learning rate $0.1$ times the peak value, cosine decay to zero, linear decay, and WSD.
For Muon, NorMuon, and SOAP, the sync-async gap remains very similar across schedules.
AdamW remains substantially more sensitive to delayed updates than the other optimizers for all tested schedules, although WSD noticeably reduces the gap relative to the other scheduler choices in this sweep.
Overall, scheduler choice does not appear to be a primary factor controlling robustness to one-step delay for the robust optimizers.

\begin{figure*}[!htbp]
    \centering
    \includegraphics[width=0.8\textwidth]{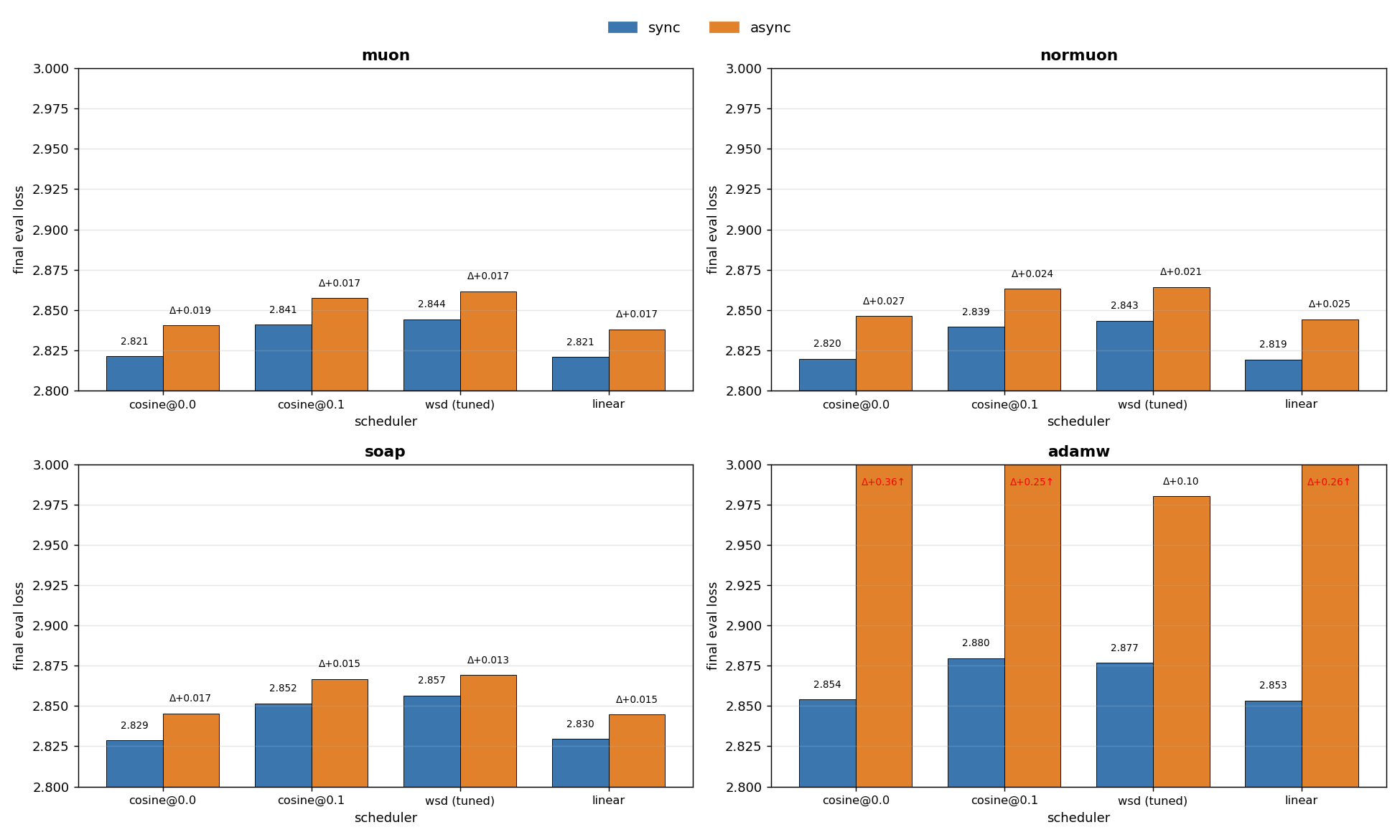}
    \caption{
    Effect of learning-rate scheduler on synchronous and one-step delayed training on the 135M model.
    The gap is similar across the tested schedules for Muon, NorMuon, and SOAP.
    Some AdamW bars are clipped for readability; annotations show the corresponding sync-async gap.
    }
    \label{fig:scheduler_sweep}
\end{figure*}

\textbf{Second-moment decay.}
Finally, we sweep the second-moment or variance decay coefficient $\beta_2$ for optimizers where this parameter is applicable.
The results in~\cref{fig:beta2_sweep} do not show a universal trend.
For AdamW, Adan with lower first-moment momentum, and Nadam, larger $\beta_2$ values tend to increase the sync-async gap.
In contrast, SOAP, NorMuon and AdaMuon are largely insensitive to $\beta_2$ over the tested range.
This behavior differs from the primary momentum coefficient $\beta_1$ or $\mu$, whose effect is consistent across optimizers in~\cref{fig:momentum_sweep}.

\begin{figure*}[!htbp]
    \centering
    \includegraphics[width=0.8\textwidth]{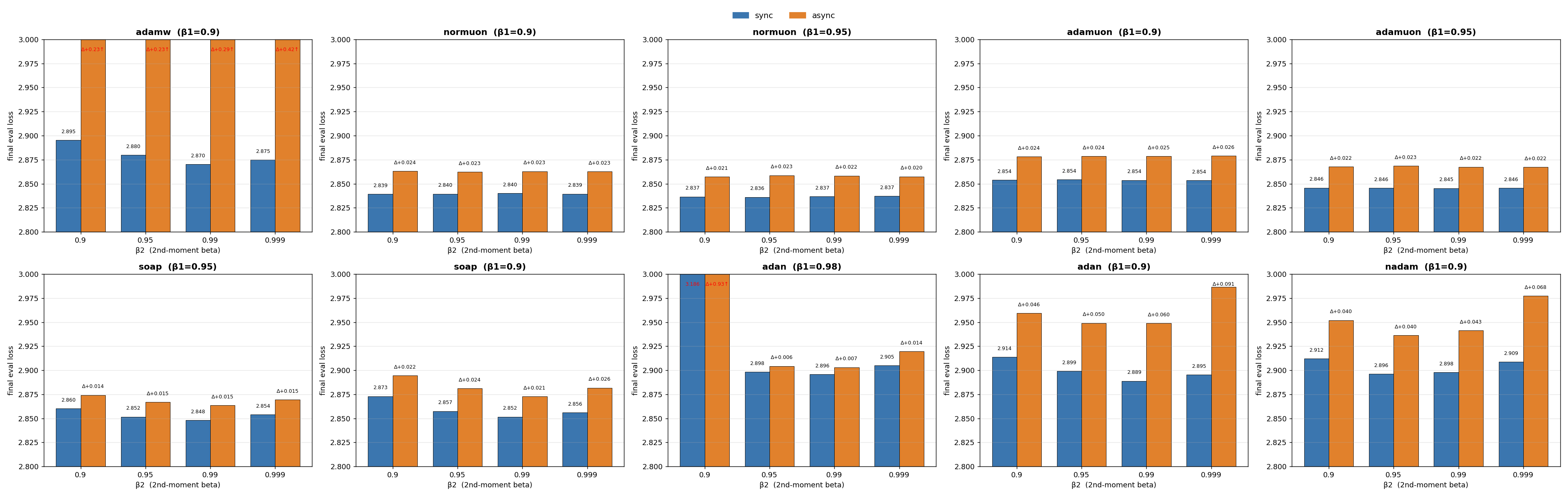}
    \caption{
    Effect of the second-moment or variance decay coefficient $\beta_2$ on synchronous and one-step delayed training on the 135M model.
    Unlike the primary momentum coefficient, $\beta_2$ does not produce a consistent optimizer-independent trend.
    Some bars are clipped for readability; annotations show the corresponding sync-async gap.
    }
    \label{fig:beta2_sweep}
\end{figure*}

\textbf{Optimizer-specific knobs.}
We also observed several effects from optimizer-specific hyperparameters.
For NorMuon, using a Nesterov-style momentum update substantially reduces the sync-async gap, whereas the same modification has little effect for standard Muon.
For SOAP, increasing the interval between preconditioner updates mildly increases the sync-async gap.
This is consistent with the broader trend above: delayed training becomes more sensitive when the underlying optimizer trajectory is made less stable or less frequently refreshed.

\newpage

\subsection{Batch Size Impact}
\label{app:batch_size_lr}

Global batch size has the strongest effect on the sync-async gap among the hyperparameters we tested.
We therefore study it separately from the one-dimensional sweeps above.
For each optimizer, we sweep the pair {batch size, peak learning rate}, since the optimal learning rate depends on the batch size~\citep{li2025predictable-scale-part1,zhang2024does-cbs1}.
We report three heatmaps: synchronous validation loss, one-step delayed validation loss, and the corresponding sync-async gap.
Because a full batch-size retuning is expensive, we restrict this analysis to several representative optimizers on the 135M model.

The results in~\cref{fig:bs_lr_heatmaps_abs} show that decreasing the batch size substantially reduces the sync-async gap.
In fact, for all three optimizers, the gap can become very small at sufficiently small batch sizes.
However, minimizing the gap alone is not the right objective: at very small batch sizes, the synchronous run also becomes worse, so the best absolute asynchronous loss does not necessarily improve.
This is especially clear for AdamW.
Although its sync-async gap can be almost eliminated by reducing the batch size, the best asynchronous loss remains more than $0.06$ worse than the best synchronous loss, consistent with the severe AdamW degradation observed in~\cref{fig:optimizer_delay_comparison}.
In contrast, Muon retains a much stronger absolute optimum under delay, with the best async loss within roughly $0.01$ of the best sync loss.

The opposite regime is also informative.
Increasing the batch size can make the sync-async gap exceed $0.1$ even for optimizers that are otherwise robust.
However, these large-batch regimes are already far from optimal for synchronous training itself, so they are unlikely to be attractive choices even without delay.
Thus, while batch size can strongly change the measured gap, the fixed-batch-size comparison in the main text remains a useful proxy: it evaluates how much quality is lost when applying one-step delayed training at a synchronous near-optimal batch size, rather than at batch sizes that are chosen only to hide or amplify the staleness gap.

\begin{figure*}[h]
\centering
\begin{minipage}{0.98\textwidth}
\centering
\includegraphics[width=\textwidth]{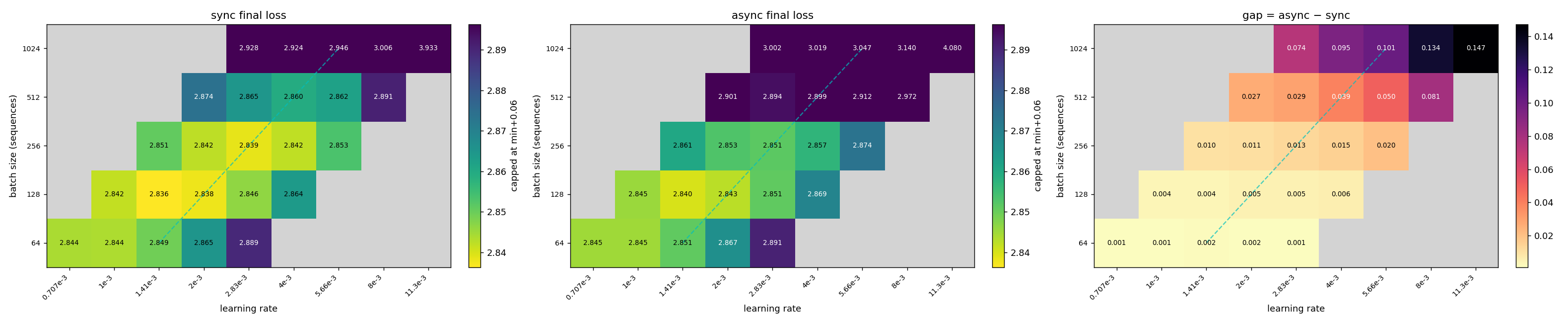}\\[-2mm]
\small (a) Muon
\end{minipage}

\vspace{2mm}

\begin{minipage}{0.98\textwidth}
    \centering
    \includegraphics[width=\textwidth]{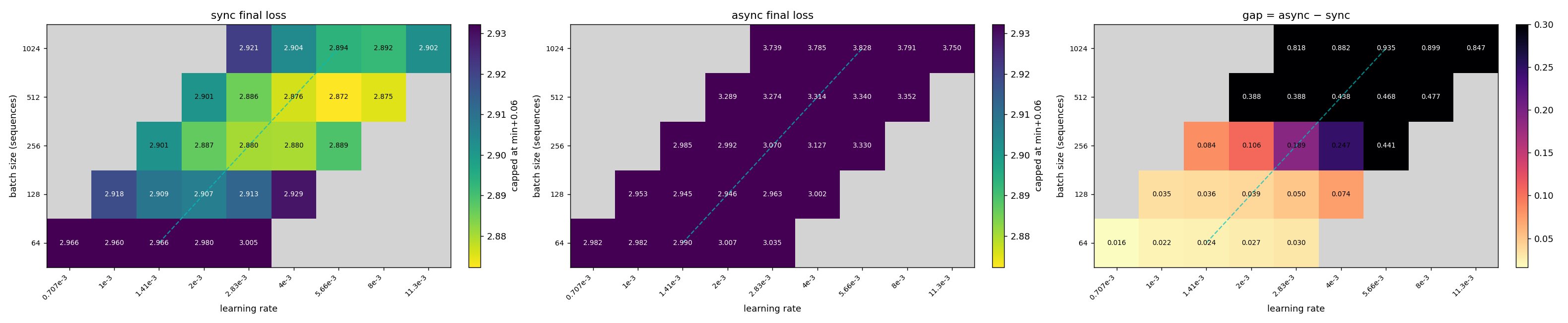}\\[-2mm]
    \small (b) AdamW
\end{minipage}

\vspace{2mm}

\begin{minipage}{0.98\textwidth}
    \centering
    \includegraphics[width=\textwidth]{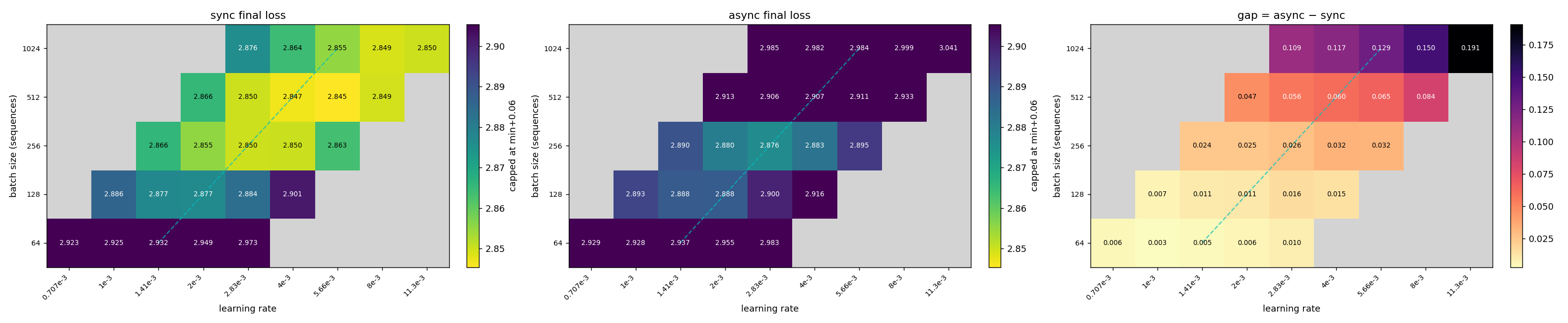}\\[-2mm]
    \small (c) SOAP
\end{minipage}

\caption{
Two-dimensional batch-size versus peak-learning-rate sweeps on the 135M model.
For each optimizer, we show synchronous validation loss, one-step delayed validation loss, and the sync-async gap.
Smaller batch sizes substantially reduce the gap, but do not necessarily yield the best absolute asynchronous loss; larger batch sizes can produce large gaps, but also degrade the synchronous baseline.
}
\label{fig:bs_lr_heatmaps_abs}

\end{figure*}

\color{black}

\FloatBarrier

\subsection{DC-ASGD Delay Compensation}
\label{app:dc_asgd}

We also evaluate Delay-Compensated ASGD (DC-ASGD)~\citep{zheng2017asynchronous-delay-compensation-taylor}, a gradient-level staleness correction based on a Taylor-style compensation term.
Because the correction magnitude is controlled by the coefficient $\lambda$, we sweep $\lambda$ from $10^4$ to $10^8$ on SmoLLM-135M with Muon.
As shown in~\cref{fig:taylor_lambda}, none of the tested coefficients improves over the standard delayed baseline.
Small values of $\lambda$ leave the result essentially unchanged, while larger values degrade training.

\begin{figure}[h]
\centering
\includegraphics[width=0.55\linewidth]{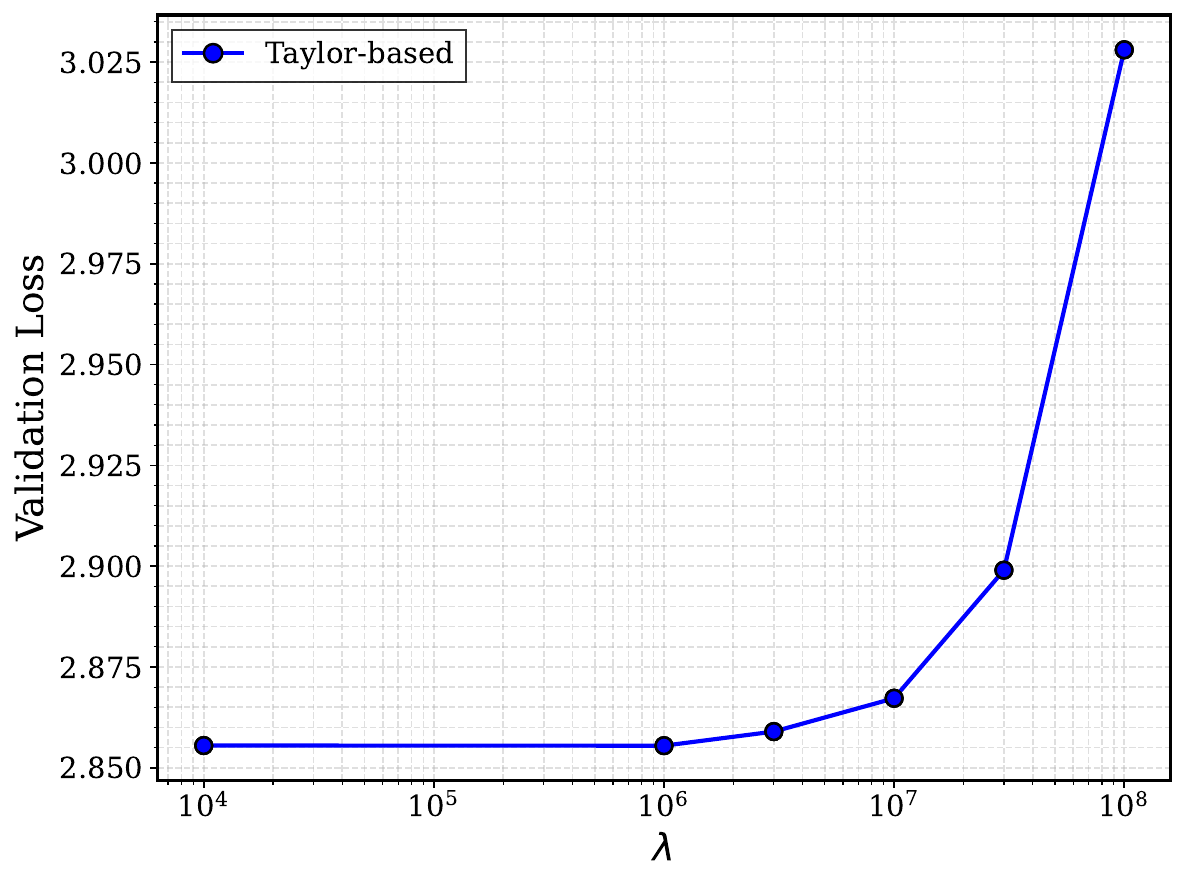}
\caption{Final validation loss versus the Taylor correction coefficient $\lambda$ for DC-ASGD-style compensation on SmoLLM-135M with Muon.
None of the tested coefficients improves over the standard delayed baseline.}
\label{fig:taylor_lambda}
\end{figure}

\color{black}
\subsection{\textcolor{black}{Error Feedback Coefficient Ablation}} \label{app:ef_lambda_ablation}

The Error-Feedback correction in~\cref{eq:ef_update} can be generalized by introducing a scaling coefficient $\lambda$:
\begin{equation}
x_{t+1} = x_t - u_{t-1}(g_{t-1}) - \lambda \cdot \left( u_{t-1}(g_{t-1}) - u_{t-2}(g_{t-2}) \right).
\end{equation}
Here $\lambda=0$ recovers standard delayed training without Error Feedback, while $\lambda=1$ corresponds to the default correction used in the main experiments.
We sweep $\lambda \in [0,3]$ on the 135M model for several representative optimizers; the results are shown in~\cref{fig:ef_lambda}.
\begin{figure}[h]
    \centering
    \includegraphics[width=0.95\linewidth]{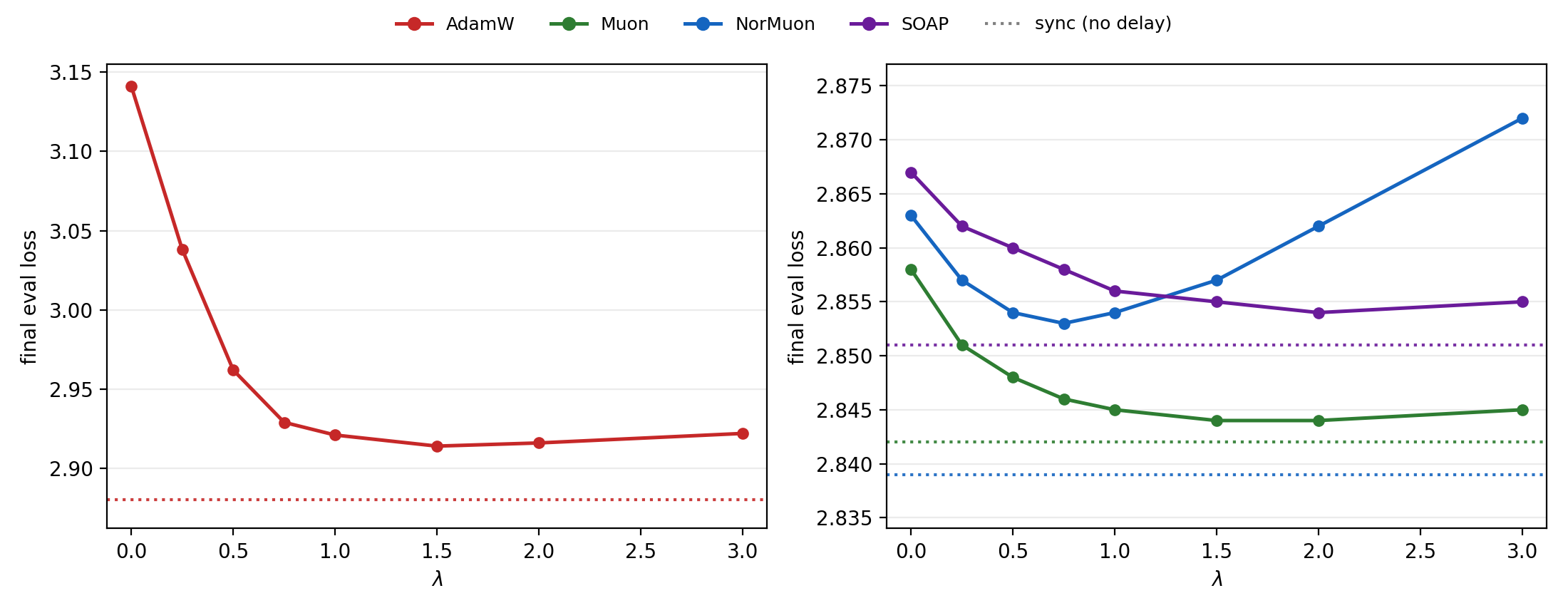}
    \caption{
    Final validation loss as a function of the Error-Feedback coefficient $\lambda$ on the 135M model.
    The value $\lambda=0$ corresponds to standard delayed training without Error Feedback, while $\lambda=1$ is the default correction used in the main experiments.
    The dependence on $\lambda$ is optimizer-specific: NorMuon shows a clear U-shape with an optimum below $1$, while AdamW, Muon, and SOAP perform best at somewhat larger values in this sweep.
    Dotted lines show the corresponding synchronous baselines.
    }
    \label{fig:ef_lambda}
\end{figure}
The results suggest an optimizer-dependent U-shaped dependence on the correction strength.
For NorMuon, this pattern is especially clear, with the best value around $\lambda=0.75$.
For AdamW, Muon, and SOAP, the optimum is shifted above the default value, roughly toward $\lambda=1.5$--$2.0$ in this sweep, although the shape is less sharp for Muon and SOAP.
Importantly, $\lambda=0$ is consistently worse than values near the standard EF setting, confirming that the correction itself is beneficial.
Because the best coefficient varies across optimizers, we keep $\lambda=1$ in all main experiments as a simple default that improves over standard delayed training without introducing another optimizer-specific hyperparameter.

\FloatBarrier
\subsection{Learning Rate Robustness for the 2B Model} \label{app:2b_learning_rate}
To verify the stability of our method at the 2B scale, we present additional training runs varying the peak learning rate across training horizons in~\cref{fig:lrs_at_scale}.

\begin{figure}[h]
    \centering
    \begin{minipage}[t]{0.45\textwidth}
        \centering
        \includegraphics[width=\textwidth]{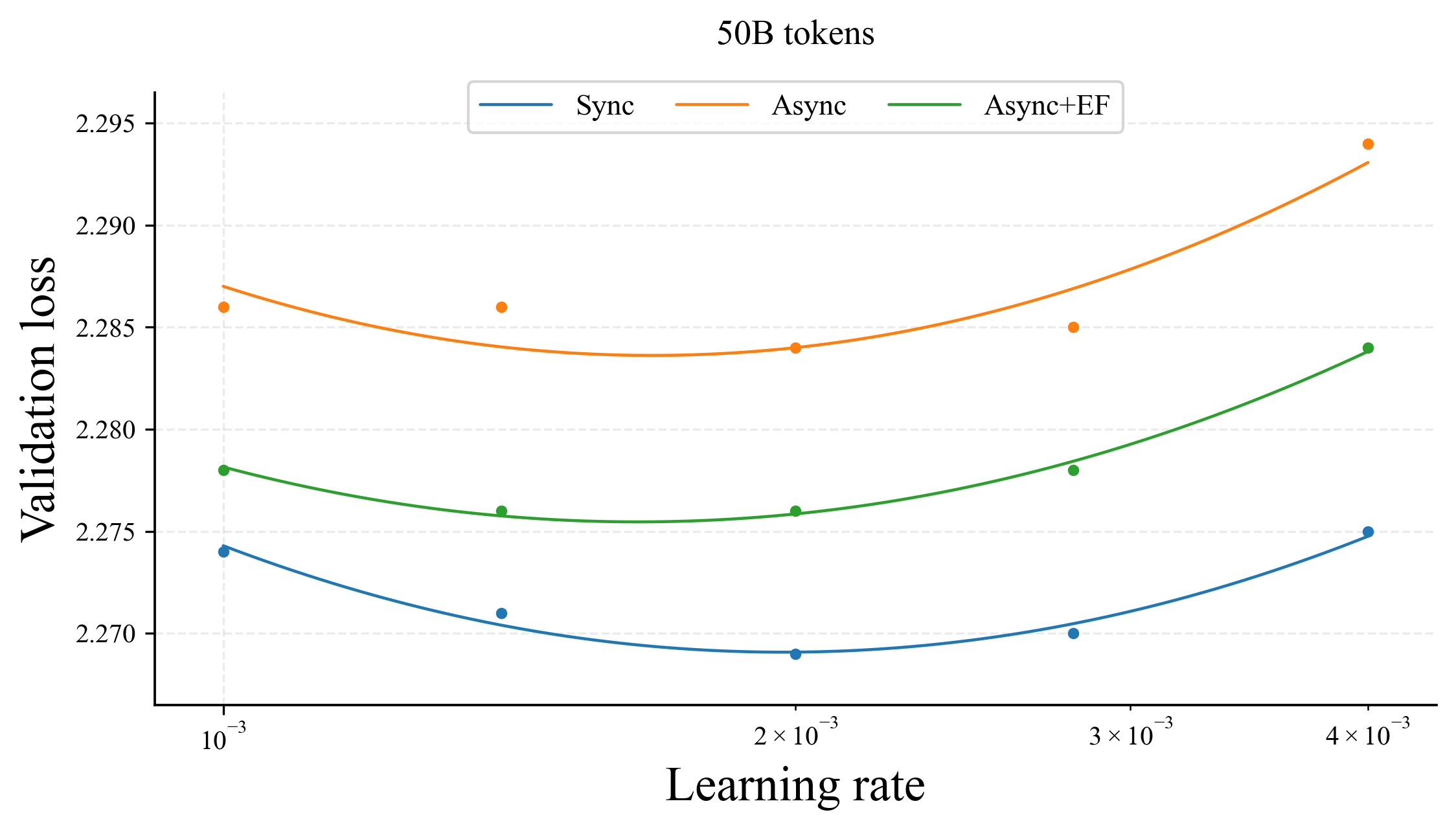}
        {\small (a) 50B tokens}
    \end{minipage}\hfill
    \begin{minipage}[t]{0.45\textwidth}
        \centering
        \includegraphics[width=\textwidth]{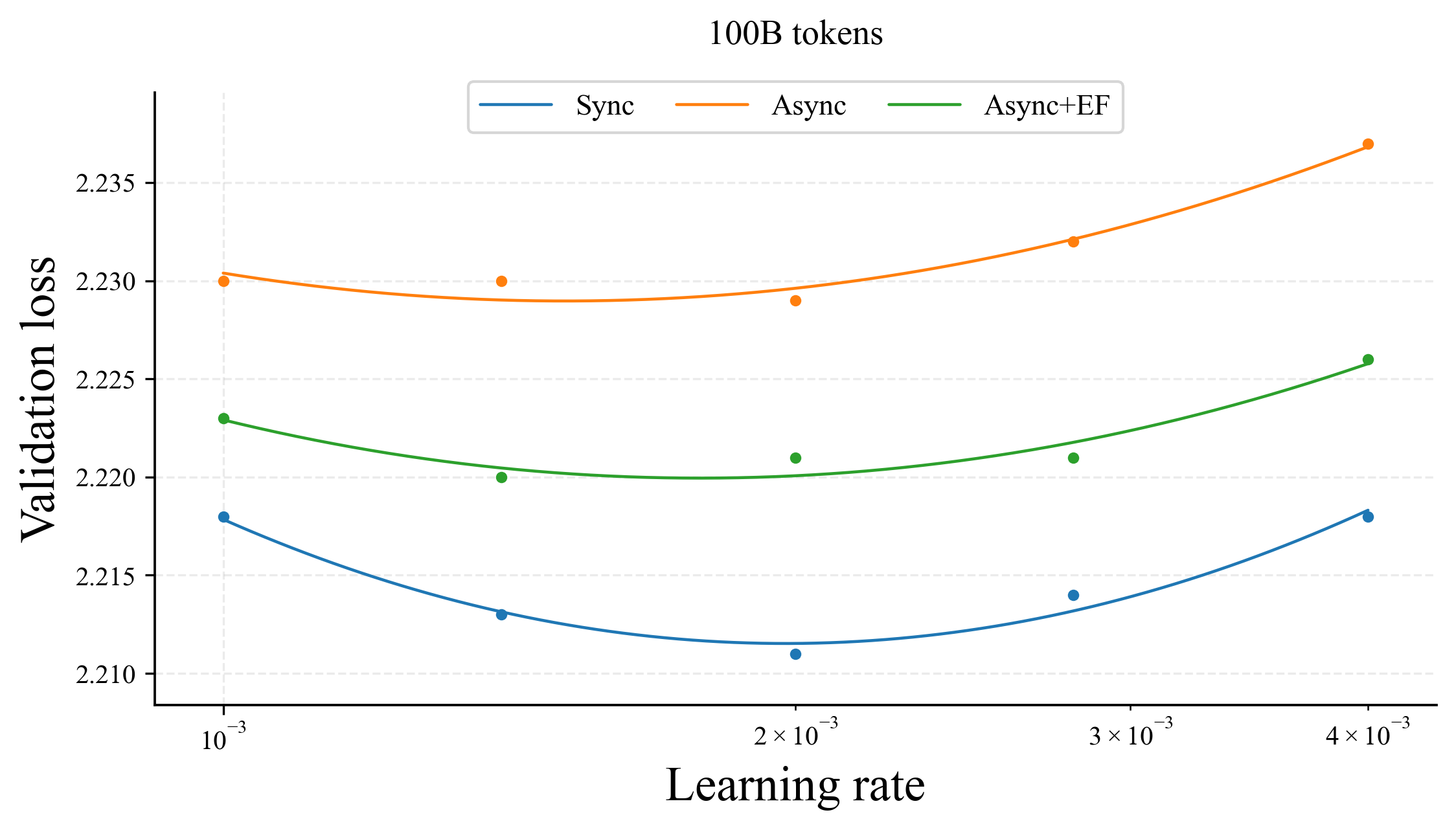}
        {\small (b) 100B tokens}
    \end{minipage}\hfill
    \begin{minipage}[t]{0.45\textwidth}
        \centering
        \includegraphics[width=\textwidth]{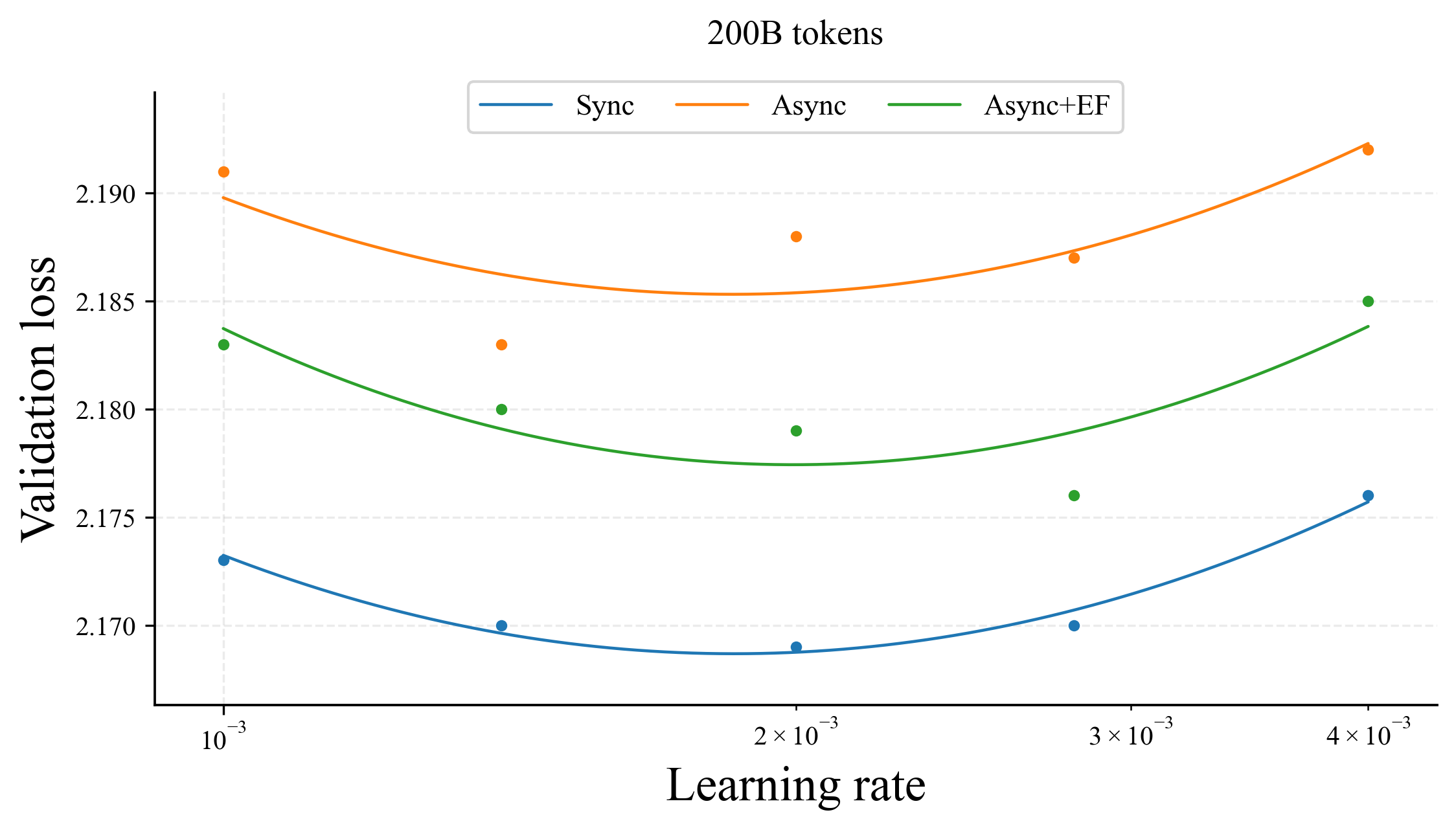}
        {\small (c) 200B tokens}
    \end{minipage}

    \caption{\textcolor{black}{Loss as a function of learning rate for synchronous and delayed 2B MoE training at different scales.}}
    \label{fig:lrs_at_scale}
\end{figure}

\subsection{\textcolor{black}{10B Benchmarking Results}} \label{app:benchmarks}

\textcolor{black}{To ensure that the identical validation losses of the synchronous and Async with Error Feedback setups reflect genuine equivalence in downstream capabilities, we evaluate the 10B MoE model checkpoints on a diverse suite of established benchmarks. These benchmarks assess a wide spectrum of abilities, ranging from broad multitask knowledge~\citep{hendrycks2021mmlu} to physical and commonsense reasoning~\citep{bisk2019piqa, sakaguchi2019winogrande, zellers2019hellaswag, roemmele2011copa}, as well as science question answering~\citep{mihaylov2018openbookqa, clark2018arc}.}

\textcolor{black}{As shown in~\cref{tab:benchmarks}, the performance of the Async + EF model closely matches the synchronous baseline across these diverse tasks. The benchmarking results confirm that the matching validation losses accurately reflect downstream benchmark behavior and validate the effectiveness of our proposed delay mitigation strategy.}

\begin{table}[h]
    \centering
    \small
    \setlength{\tabcolsep}{3pt}
    \caption{\textcolor{black}{Downstream benchmark evaluation and final validation loss for the 10B MoE model. The Async + EF setup exactly matches the synchronous validation loss, and its downstream performance variations average out to equivalent overall quality. Bold values indicate the best result for each metric (lowest for Val. Loss, highest for benchmarks). Abbreviations: ARC-Ch. (ARC-Challenge), ARC-Ea. (ARC-Easy), HellaSw. (HellaSwag), OBQA (OpenBookQA), WinoG. (WinoGrande).}}
    \label{tab:benchmarks}
    \vspace{0.5em}
    \begin{tabular}{lccccccccc}
    \toprule
    \textbf{Setup} & \textbf{Val. Loss} & \textbf{ARC-Ch.} & \textbf{ARC-Ea.} & \textbf{COPA} & \textbf{HellaSw.} & \textbf{MMLU} & \textbf{OBQA} & \textbf{PIQA} & \textbf{WinoG.} \\
    \midrule
    Sync       & \textbf{1.906} & \textbf{0.418} & 0.640 & 0.695 & 0.670 & \textbf{0.411} & \textbf{0.537} & 0.775 & \textbf{0.590} \\
    Async      & 1.911 & 0.414 & \textbf{0.643} & \textbf{0.699} & 0.665 & 0.410 & 0.532 & 0.771 & 0.584 \\
    Async + EF & \textbf{1.906} & 0.415 & 0.637 & 0.691 & \textbf{0.673} & 0.407 & \textbf{0.537} & \textbf{0.778} & 0.567 \\
    \bottomrule
    \end{tabular}
\end{table}

\subsection{\textcolor{black}{Empirical Noise Level Estimation}} \label{app:noise_level}

\textcolor{black}{To quantify the inherent stochastic noise in our training setup, we evaluate the variance of the final validation loss for the 135M model trained with Muon across multiple random seeds. As summarized in Table~\ref{tab:noise_level}, the observed standard deviation is approximately $10^{-3}$.}

\begin{table}[h]
    \centering
    \caption{\textcolor{black}{Empirical noise level of the final validation loss for the 135M model trained with Muon. We report the mean, standard deviation, and range across multiple independent runs.}}
    \label{tab:noise_level}
    \vspace{0.5em}
    \begin{tabular}{lccc}
    \toprule
    \textbf{Setup} & \textbf{Mean} & \textbf{Std Dev} & \textbf{Range} \\
    \midrule
    Muon (Sync)       & 2.8421 & $3.9 \times 10^{-4}$ & 2.8417 -- 2.8426 \\
    Muon (Async)      & 2.8565 & $9.7 \times 10^{-4}$ & 2.8552 -- 2.8576 \\
    \midrule
    Gap (Async $-$ Sync) & 0.0144 & $8.5 \times 10^{-4}$ & 0.0134 -- 0.0151 \\
    \bottomrule
    \end{tabular}
\end{table}

\textcolor{black}{To assess whether this noise level is representative, we additionally measure the synchronous variance for other optimizers and the extra variance introduced by Error Feedback.}

\begin{table}[h]
    \centering
    \caption{\textcolor{black}{Synchronous validation loss variance across different optimizers on the 135M model. The standard deviation remains below $2 \times 10^{-3}$ in all cases.}}
    \label{tab:noise_sync_opts}
    \vspace{0.5em}
    \begin{tabular}{lccc}
    \toprule
    \textbf{Optimizer} & \textbf{Sync Mean} & \textbf{Sync Std} & \textbf{Seeds} \\
    \midrule
    Muon ($\mu=0.99$)   & 2.842 & $3.9 \times 10^{-4}$ & 5 \\
    SOAP                & 2.853 & $4.0 \times 10^{-4}$ & 3 \\
    AdamW               & 2.881 & $9.3 \times 10^{-4}$ & 3 \\
    NorMuon             & 2.836 & $1.0 \times 10^{-3}$ & 3 \\
    Lion                & 2.874 & $2.0 \times 10^{-3}$ & 3 \\
    \bottomrule
    \end{tabular}
\end{table}

\begin{table}[!h]
    \centering
    \caption{\textcolor{black}{Async + Error Feedback noise on the 135M model. Error Feedback introduces additional variance, which remains small for robust optimizers but is more pronounced for unstable ones.}}
    \label{tab:noise_ef}
    \vspace{0.5em}
    \begin{tabular}{lcccc}
    \toprule
    \textbf{Optimizer} & \textbf{Sync Std} & \textbf{EF Std} & \textbf{EF Gap (mean)} & \textbf{Gap Std} \\
    \midrule
    SOAP                & $4.0 \times 10^{-4}$ & $1.3 \times 10^{-3}$ & $+0.005$ & $1.5 \times 10^{-3}$ \\
    NorMuon             & $1.0 \times 10^{-3}$ & $8.3 \times 10^{-4}$ & $+0.009$ & $3.7 \times 10^{-4}$ \\
    AdamW               & $9.3 \times 10^{-4}$ & $6.5 \times 10^{-3}$ & $+0.046$ & $6.1 \times 10^{-3}$ \\
    Lion                & $2.0 \times 10^{-3}$ & $6.0 \times 10^{-3}$ & $+0.040$ & $4.2 \times 10^{-3}$ \\
    \bottomrule
    \end{tabular}
\end{table}

\textcolor{black}{These measurements confirm that the $0.01$ constraint is sufficiently strict: the noise and Error Feedback-induced variance remain well below this threshold.}

\color{black}

\subsection{\textcolor{black}{AdamW Ablations}}
\label{app:adamw_ablations}

To better understand why AdamW degrades more severely than other optimizers under one-step delay, we perform several targeted diagnostic experiments on the 135M model.
These experiments do not fully isolate a single cause of AdamW's degradation, but they help rule out several simple explanations and provide additional evidence for the importance of first-moment dynamics.

We first compare stale updates with the corresponding fresh updates that would have been applied in a non-delayed run.
Specifically, we measure update cosine similarity and relative update error between these two updates; see~\cref{fig:cosine_sim_full,fig:cosine_sim_zoom,fig:rel_mse_full,fig:rel_mse_zoom}.
Somewhat surprisingly, these discrepancy metrics are not worse for AdamW than for Muon.
Thus, AdamW's poor delayed performance cannot be explained simply by its stale updates being more different from their fresh counterparts according to these direct update-level metrics.

We next test whether the degradation is driven by the final language-model head, which has been identified as a sensitive component in optimizer studies~\citep{zhao2025deconstructing}.
To do so, we keep the LM head synchronous while applying delayed updates to the rest of the model.
As shown in~\cref{tab:adamw_ablations}, this modification provides little improvement: AdamW still remains far worse than its synchronous baseline, with final loss above $3.0$ in the 135M setting.
This suggests that the instability is not localized to the LM head.

Finally, we isolate the effect of delaying different AdamW state variables.
When the delay is applied only to the first-moment update $m_t$, while the second-moment update $v_t$ remains synchronous, the resulting loss is almost identical to fully delayed AdamW (see~\cref{tab:adamw_ablations}).
This supports the interpretation in~\cref{sec:section 2 opt_delay_hyperparameter_sensitivity} that first-moment dynamics play a central role in robustness to one-step delay.

A likely reason is that the value of $\beta_1$ required for delay robustness may fall outside the stable region for AdamW itself.
In our experiments, increasing $\beta_1$ improves delayed robustness only up to a point, while very large values destabilize or degrade AdamW even in the synchronous setting.
For example, synchronous AdamW with $\beta_1=0.99$ reaches a substantially worse final loss than the standard $\beta_1=0.9$ configuration ($2.939$ vs. $2.877$).
This contrasts with optimizers such as SOAP or Adan, whose stable operating regimes include substantially larger first-moment coefficients. 
% \textcolor{blue}{This observation is consistent with recent findings by \citet{jung2026mitigating}, who show that accumulating momentum in a static basis causes severe directional misalignment under staleness. 
% They demonstrate that mitigating this effect requires transforming the gradient space, a mechanism that standard AdamW inherently lacks.}
\textcolor{black}{This observation aligns with recent findings by \citet{jung2026mitigating}, who show that staleness causes severe momentum misalignment unless mitigated by matrix-level basis transformations.}
% akin to those in SOAP or Shampoo. 
% Crucially, AdamW inherently lacks such spatial transformations or orthogonalization, leaving its unrotated exponential moving average highly vulnerable to directional drift. 
% This explains our empirical results: optimizers like SOAP and Muon (via Newton-Schulz orthogonalization) possess intrinsic geometric stability against staleness, whereas AdamW's purely diagonal scaling offers no such protection.}

\begin{table}[h]
\centering
\small
\begin{tabular}{p{0.56\linewidth}ccc}
\toprule
\textbf{Setup} & \textbf{Sync} & \textbf{EF} & \textbf{Async} \\
\midrule
AdamW, $\beta=(0.9, 0.95)$ & 2.879 & 2.920 & 3.158 \\
AdamW, \texttt{no\_delay\_lmhead}, $\beta=(0.9, 0.95)$ & 2.879 & 2.917 & \textbf{3.100} \\

AdamW, $\beta=(0.95, 0.95)$ & 2.877 & 2.901 & 3.227 \\

AdamW, $\beta=(0.9, 0.99)$ & 2.876 & 2.897 & -- \\
Adam delay $m$, $\beta=(0.9, 0.99)$ & 2.875 & 2.898 & 3.190 \\

AdamW, $\beta=(0.95, 0.99)$ & 2.875 & \textbf{2.895} & -- \\
Adam delay $m$, $\beta=(0.95, 0.99)$ & 2.876 & 2.896 & 3.450 \\

AdamW, $\beta=(0.95, 0.999)$ & \textbf{2.873} & 2.908 & 3.241 \\
AdamW, \texttt{no\_delay\_lmhead}, $\beta=(0.95, 0.999)$ & \textbf{2.873} & 2.922 & 3.289 \\
\bottomrule
\end{tabular}
\vspace{4.0pt}
\caption{Diagnostic AdamW ablations on the 135M model.
We compare fully delayed AdamW, AdamW with only the first-moment update delayed, and AdamW with a synchronous LM head.
The results suggest that delaying the first-moment dynamics closely reproduces the behavior of fully delayed AdamW, while keeping the LM head synchronous does not remove the degradation.}
\label{tab:adamw_ablations}
\end{table}
\begin{figure}[h]
    \centering
    \begin{subfigure}[b]{0.48\linewidth}
        \centering
        \includegraphics[width=\linewidth]{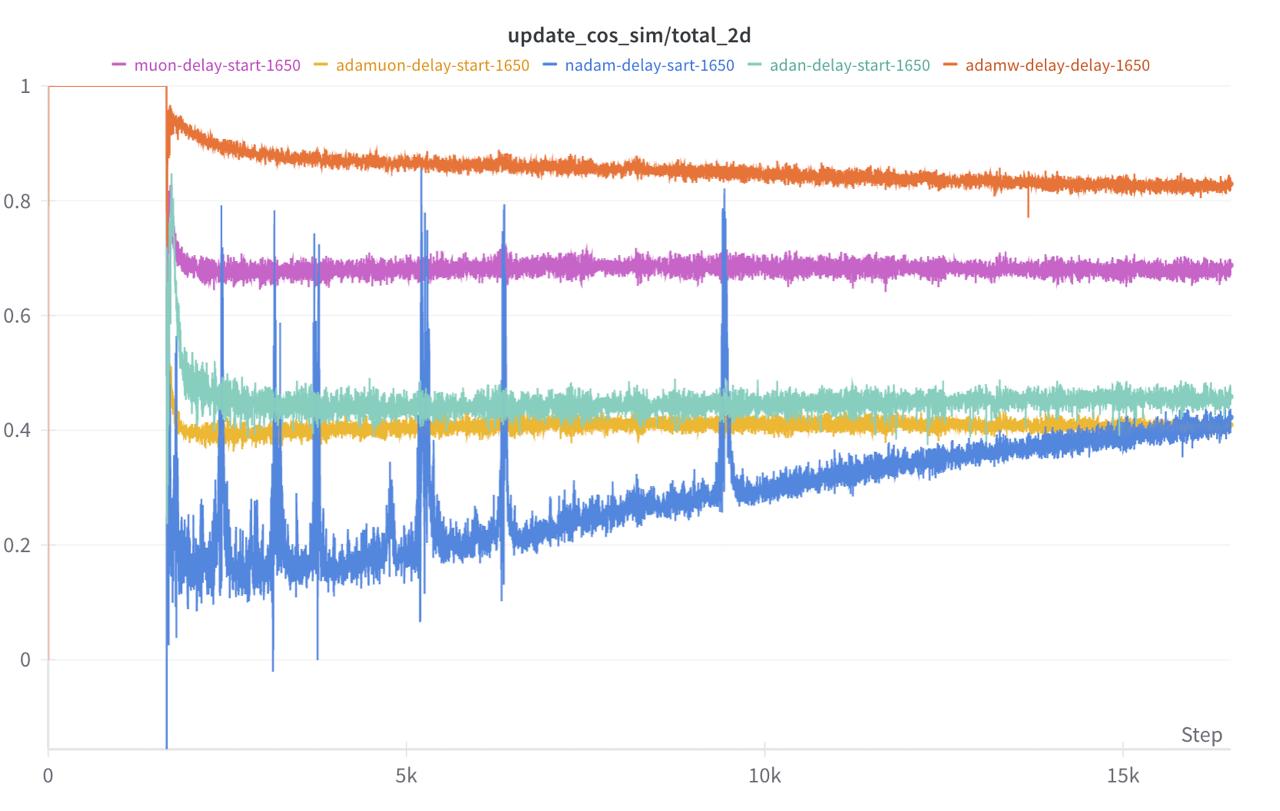}
        \caption{Full training dynamics.}
        \label{fig:cosine_sim_full}
    \end{subfigure}
    \hfill
    \begin{subfigure}[b]{0.48\linewidth}
        \centering
        \includegraphics[width=\linewidth]{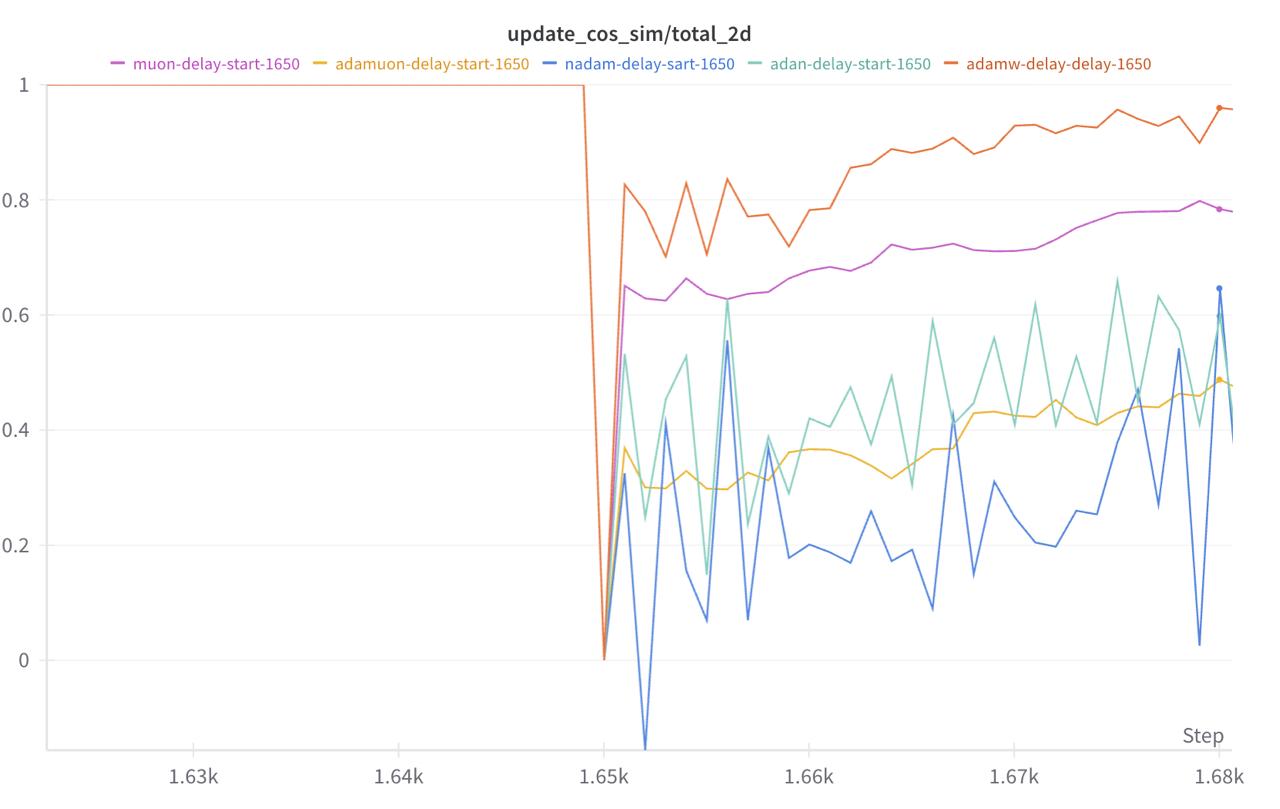}
        \caption{Zoomed-in view around delay start.}
        \label{fig:cosine_sim_zoom}
    \end{subfigure}
    
    \caption{Cosine similarity between the delayed optimizer update and the corresponding fresh update on the 135M model.
    The fresh update is defined as the update that would have been applied using the non-delayed gradient at the same step.
    The right panel zooms in on the iterations around the transition to one-step delayed training.}
    \label{fig:update_cosine_sim}
\end{figure}

\begin{figure}[h]
    \centering
    \begin{subfigure}[b]{0.48\linewidth}
        \centering
        \includegraphics[width=\linewidth]{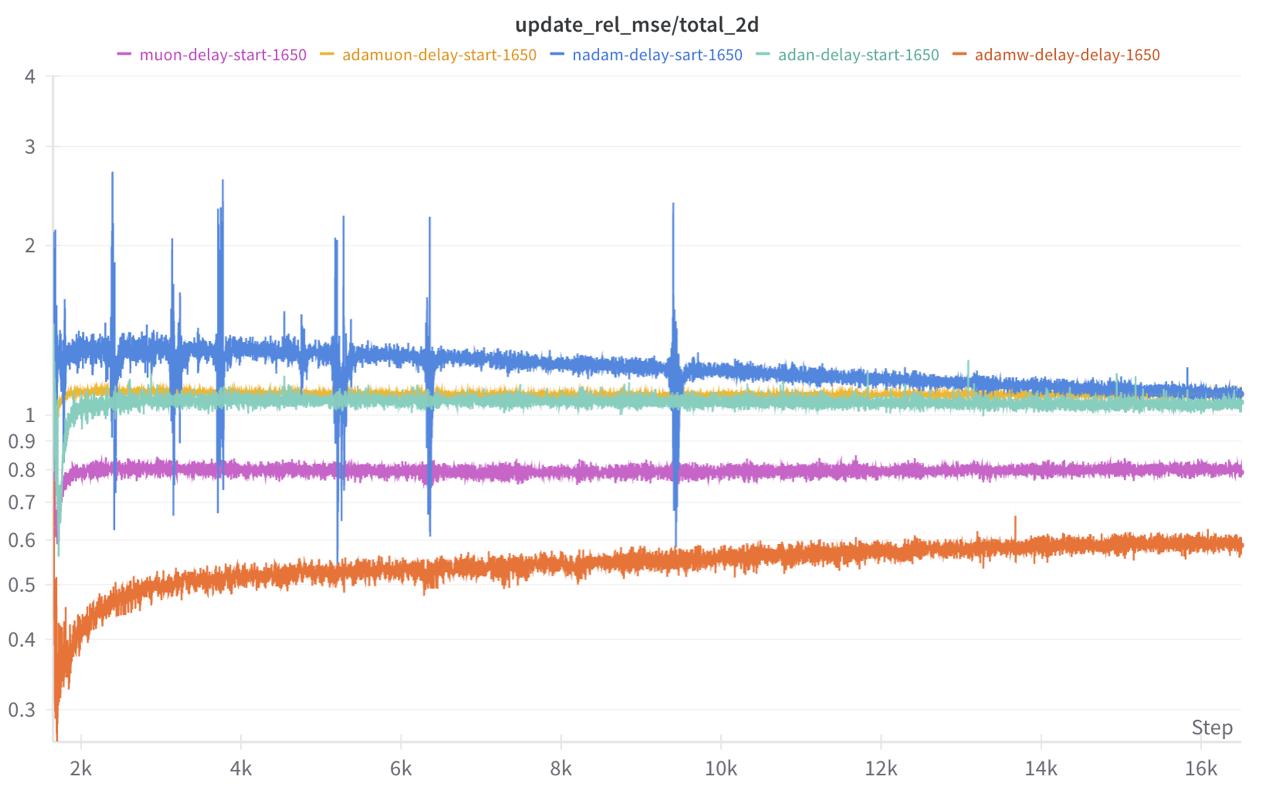}
        \caption{Full training dynamics.}
        \label{fig:rel_mse_full}
    \end{subfigure}
    \hfill
    \begin{subfigure}[b]{0.48\linewidth}
        \centering
        \includegraphics[width=\linewidth]{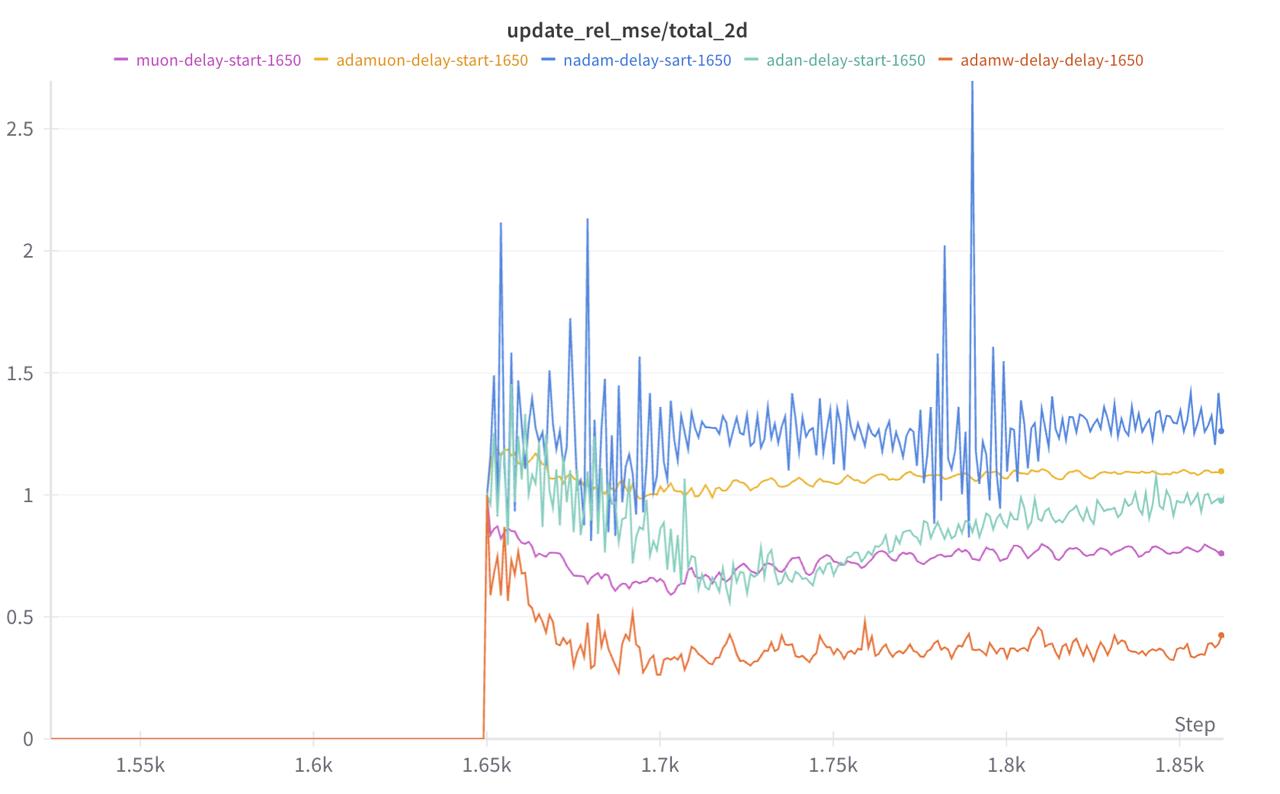}
        \caption{Zoomed-in view around delay start.}
        \label{fig:rel_mse_zoom}
    \end{subfigure}
    
    \caption{Relative update error between the delayed optimizer update and the corresponding fresh update on the 135M model.
    The fresh update is defined as the update that would have been applied using the non-delayed gradient at the same step.
    The right panel zooms in on the iterations around the transition to one-step delayed training.}
    \label{fig:update_rel_mse}
\end{figure}

\FloatBarrier

\subsection{Ablation on Synchronous Cooldown}
\label{app:delay_end}

As discussed in~\cref{sec:sync_warmup_strategy}, we investigated a "synchronous cooldown" strategy, where the training process switches from asynchronous to synchronous mode towards the end of training. 
The hypothesis was that removing stale gradients in the final convergence phase might recover the remaining performance gap.
We conducted ablation studies on the 135M model for both Muon and AdamW.
For AdamW, we utilized $\beta_1=0.95$, synchronous warmup of $1W$, and Error-Feedback enabled, while for Muon we used the standard configuration with async start at step 0, no EF.
The switch-over point was defined relative to the warmup duration $W$ (e.g., $-1.5W$ indicates switching to synchronous mode $1.5 \times W$ steps before the end of training).

The results are summarized in Table~\ref{tab:delay_end_ablation}. 
We observe that switching back to synchronous training yields only marginal improvements.

\begin{table}[h]
\centering
\caption{Ablation study on switching to synchronous training near the end of the schedule (135M model). Cutoff times are expressed relative to the warmup duration $W$.}
\label{tab:delay_end_ablation}
\resizebox{\textwidth}{!}{%
\begin{tabular}{lcccccc}
\toprule
 & \textbf{Sync} & \textbf{No Switch} & \multicolumn{4}{c}{\textbf{Switch to Sync (Time before end)}} \\
\cmidrule(lr){4-7}
\textbf{Configuration} & \textbf{Baseline} & (Async throughout) & $-1.5W$ & $-1.0W$ & $-0.5W$ & $-0.25W$ \\ \midrule
Muon (start=$0$, no EF) & 2.839 & 2.856 & 2.853 & 2.853 & 2.853 & 2.853 \\
AdamW ($\beta_1=0.95$, start=$1W$, EF) & 2.879 & 2.935 & 2.928 & 2.928 & 2.930 & 2.929 \\ \bottomrule
\end{tabular}%
}
\end{table}

\FloatBarrier
\subsection{Gradient-Based Error Feedback}
\label{app:gradient_ef}

We also evaluate a gradient-level variant of Error Feedback, where the correction is applied to raw gradients before they are passed to the optimizer.
As shown in~\cref{fig:grad_ef}, this variant is unstable and diverges in our experiments.
We do not investigate the cause of this instability in detail, and use the update-level correction throughout the main experiments.

\begin{figure}[h]
    \centering
    \includegraphics[width=0.7\linewidth]{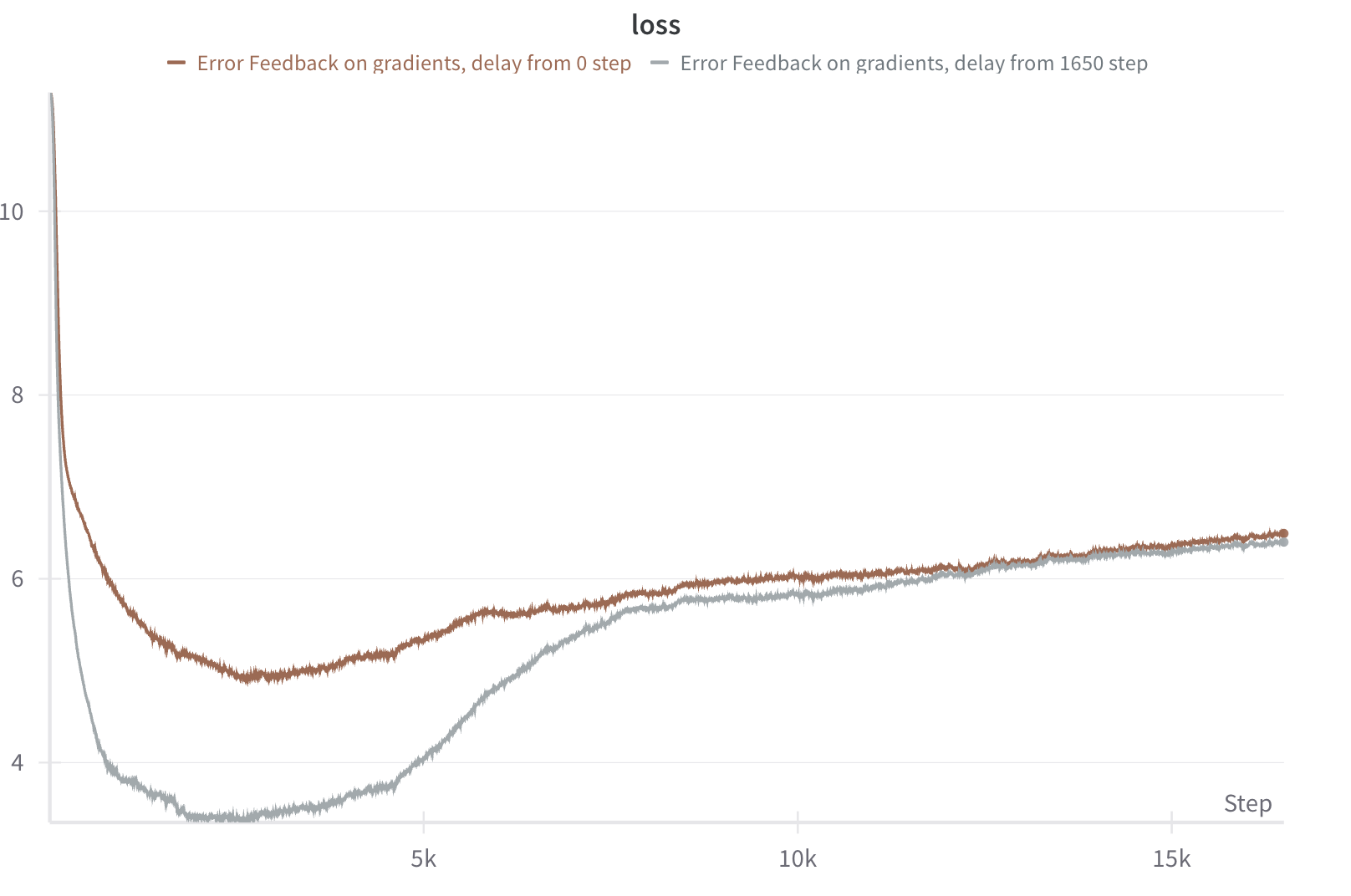}
    \caption{Gradient-level Error Feedback on the 135M model.
    Unlike the update-level correction used in the main experiments, applying the correction directly to raw gradients leads to divergence in our setting.}
    \label{fig:grad_ef}
\end{figure}

\FloatBarrier
\subsection{Effect of \texorpdfstring{$\beta_2$}{beta2} on the Synchronous-Start Loss Spike}

We show in~\cref{fig:beta2loss} that the spike in train loss is notably larger for $\beta_2=0.999$ in comparison to $\beta_2=0.95$.

\begin{figure}[!htbp]
    \centering
    \includegraphics[width=0.7\textwidth]{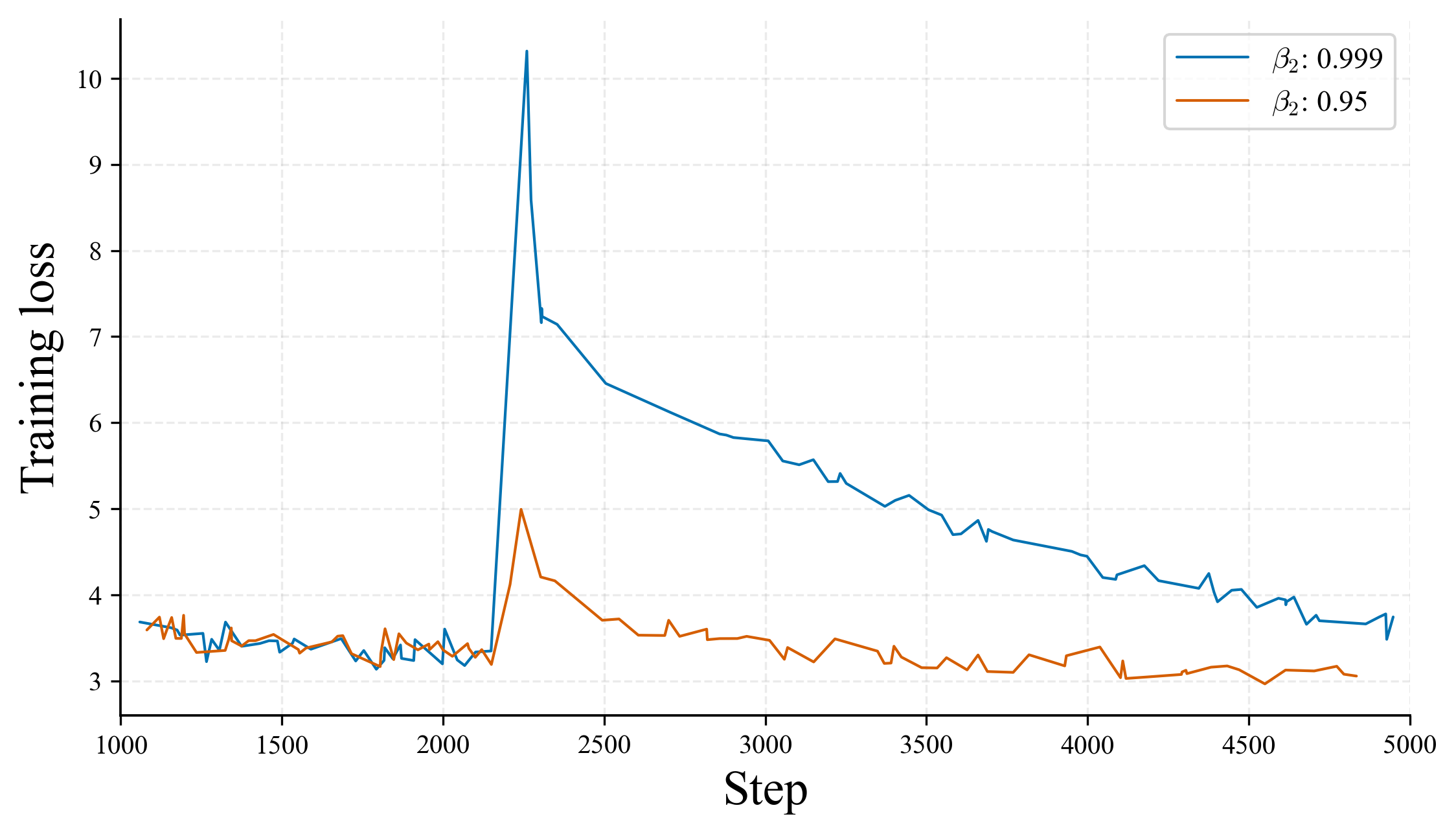}
    \caption{\textcolor{black}{The relationship between the $\beta_2$ value and the loss ``spike''.}}
    \label{fig:beta2loss}
    % \vspace{-1em}
\end{figure}

\FloatBarrier

\color{black}

\section{Additional Related Work: SAPipe and WPipe}
\label{app:sapipe_wpipe}

We also discuss two related approaches that we identified after the main phase of this work had been completed: SAPipe-style weight prediction~\citep{chen2022sapipe} and the WPipe schedule~\citep{yang2022groupbased}.

\subsection{SAPipe Weight Prediction Technique}
\label{app:sapipe}

SAPipe~\citep{chen2022sapipe} is closely related to our staleness-mitigation study, although it originates from a different systems setting.
We became aware of this connection only after the main body of this work had largely been completed, and therefore discuss it here.
In SAPipe, staleness is introduced in data-parallel training to overlap gradient synchronization with computation: the next forward-backward pass can start before the previous gradient aggregation has finished.
In contrast, our staleness comes from pipeline parallelism, specifically the asynchronous PipeDream-\texttt{2BW} schedule.
Thus, both settings lead to one-step delayed gradients, but the underlying systems motivation is different.

SAPipe proposes several staleness-compensation options, including weight prediction and delay compensation.
The variant most directly related to our setting is SAPipe-WP with the latest synchronized gradient, i.e., Option 2 in their Algorithm 3.
This option uses the stale synchronized gradient $g_{t-1}$ to construct a one-step-ahead predicted point for the next forward-backward pass.
This is exactly the same information available in our one-step delayed Async PP setting, and therefore provides the cleanest comparison to our Error-Feedback correction.

\textbf{Relation to Error Feedback.}
The relation between SAPipe-style weight prediction and our update-level Error Feedback can be seen by comparing the displacement between consecutive forward-backward compute weights.
To see this, suppress for the moment the time-dependence of the optimizer state and learning rate, and write a stale optimizer update as $u(g)$.
In the one-step delayed setting, the actual parameter update is
\[
    x_{t+1} = x_t - u(g_{t-1}).
\]
SAPipe-WP then predicts the point at which the next gradient will be evaluated by applying the same stale update once more:
\[
    \tilde{x}_{t+1}
    = \operatorname{optimizer}(x_{t+1}, g_{t-1}, \eta)
    = x_{t+1} - u(g_{t-1})
    = x_t - 2u(g_{t-1}).
\]
The previous forward--backward pass was evaluated not at $x_t$, but at the previous predicted point
\[
    \tilde{x}_t = x_t - u(g_{t-2}).
\]
Therefore, the displacement between two consecutive SAPipe prediction points is
\[
    \tilde{x}_{t+1} - \tilde{x}_t
    =
    -2u(g_{t-1}) + u(g_{t-2}),
\]
which matches the displacement used by our Error-Feedback update in~\cref{eq:ef_update}:
\[
    x_{t+1}^{\mathrm{EF}} - x_t
    =
    -2u_{t-1}(g_{t-1}) + u_{t-2}(g_{t-2}).
\]
Thus, if the sequence of gradients and optimizer updates were fixed in advance, SAPipe-WP and our EF correction would generate the same update displacement up to an index shift and a constant offset.
In actual training, the trajectories are not identical: the two methods evaluate gradients at different points during the first steps, and this changes all subsequent gradients.
Nevertheless, the algebra shows that the two methods are closely related.
It is interesting that two different viewpoints --- weight prediction in a data-parallel communication pipeline and update-level Error Feedback in asynchronous pipeline parallelism --- lead to nearly the same correction.

\textbf{Optimizer-state handling.}
One implementation detail that is not explicit in the SAPipe description is how optimizer state should be handled during the prediction step.
There are two natural choices.
The first is to compute both the real update and the prediction update using the same optimizer state, yielding two copies of the same update $u_{t-1}(g_{t-1})$.
This is the variant most directly matched by the derivation above, and empirical results in~\cref{tab:sapipe} show that it performs very similarly to our Error-Feedback correction.

The second choice is to update the optimizer state after the real parameter update and then use this updated state for the prediction step.
In our notation, this replaces the second copy of $u_{t-1}(g_{t-1})$ with $u_t(g_{t-1})$.
This variant is not mathematically identical to EF, but it can be slightly stronger in practice.
In our experiments, using the updated optimizer state improves AdamW, while the difference for Muon, SOAP, and NorMuon is within noise.
However, this variant also has an additional cost: the optimizer update must be computed a second time for the prediction step.
For modern matrix-based optimizers such as Muon or SOAP, where the update itself involves nontrivial matrix operations, this cost may be non-negligible.
Thus, SAPipe-style prediction with updated state provides a possible quality-cost trade-off rather than a strictly better replacement for EF.

\begin{table}[h]
\centering
\small
\setlength{\tabcolsep}{3.5pt}
\renewcommand{\arraystretch}{1.1}
\begin{tabular}{l|c|c|ccc}
\toprule
\textbf{Optimizer}
& \textbf{Sync}
& \textbf{Async (no-EF)}
& \textbf{EF}
& \textbf{SAPipe-WP}
& \textbf{SAPipe-WP} \\
& & & & \textbf{same state} & \textbf{updated state} \\
\midrule
Muon               & 2.842 & 2.858 & 2.846 & 2.845 & \textbf{2.843} \\
AdamW ($\beta_1=0.9$)  & 2.881 & 3.141 & 2.920 & 2.918 & \textbf{2.898} \\
AdamW ($\beta_1=0.95$) & 2.879 & 3.227 & 2.903 & 2.901 & \textbf{2.884} \\
SOAP               & 2.853 & 2.867 & 2.860 & 2.861 & \textbf{2.854} \\
NorMuon            & 2.840 & 2.863 & 2.854 & 2.856 & \textbf{2.852} \\
\bottomrule
\end{tabular}
\vspace{4.0pt}
\caption{
Comparison of update-level Error Feedback with SAPipe-style weight prediction on the 135M model.
The ``same state'' variant computes the real update and prediction update using the same optimizer state, making it closest to the EF derivation above.
The ``updated state'' variant updates the optimizer state before computing the prediction step; it can improve AdamW, but requires computing an additional optimizer update.
}
\label{tab:sapipe}
\end{table}

\textbf{Scope of the comparison.}
SAPipe and our work are complementary.
SAPipe develops a data-parallel system for hiding communication overhead and provides convergence guarantees for SGD under bounded-gradient assumptions.
Our work focuses on asynchronous pipeline-parallel LLM pre-training, where staleness is induced by the pipeline schedule, and studies modern optimizers under this delay.
In particular, our theory covers LMO-style methods such as Muon, and our experiments include optimizer benchmarking, detailed hyperparameter ablation, and validation up to 10B parameters.

\color{black}

\FloatBarrier

\subsection{WPipe Scheduling Scheme}
\label{app:wpipe}

\color{black}

The WPipe schedule~\citep{yang2022groupbased} provides an alternative way to organize asynchronous pipeline execution.
We became aware of this schedule only after the main body of this work had largely been completed, and therefore treat it here as an additional practical consideration rather than as part of the main experimental study.
WPipe is applicable in settings where multiple logical pipeline stages can be placed on the same GPU.
In this case, it can be viewed as a hybrid schedule in which the second half of the model experiences no delay, while the first half experiences a one-step delay similar to the PipeDream-\texttt{2BW} setting studied in this work.
Thus, from an optimization perspective, WPipe behaves similarly to a two-stage PipeDream-style schedule, while still allowing the model to be partitioned into an arbitrary number of physical pipeline stages.
Compared to PipeDream-\texttt{2BW}, only part of the model is optimized with delayed updates rather than all layers, suggesting a potentially more favorable optimization trade-off.

Importantly, we view WPipe as orthogonal to the main contributions of this paper.
Our work studies optimizer behavior and mitigation mechanisms under one-step delayed updates.
WPipe changes the schedule so that only part of the model experiences this delay, but the delayed part still requires robust optimizers and can still benefit from mitigation strategies such as Error Feedback.
We therefore include WPipe experiments as a demonstration that the same optimizer-level conclusions and Error-Feedback correction can be combined with a more favorable schedule.

\begin{table}[h]
\centering
\small
\setlength{\tabcolsep}{4pt}
\renewcommand{\arraystretch}{1.15}
\caption{
Validation loss for PipeDream-\texttt{2BW} and WPipe-style schedules on the 360M model.
For each optimizer, we compare synchronous training, standard one-step delayed training, delayed training with Error Feedback, and the corresponding WPipe variants.
Bold entries indicate the best result among asynchronous variants within each optimizer block.
}
\label{tab:wpipe_360m}
\begin{tabular}{ll|c|cc|cc}
\toprule
\textbf{Optimizer} & \textbf{Hyperparameter}
& \textbf{Sync}
& \multicolumn{2}{c|}{\textbf{PipeDream-\texttt{2BW}}}
& \multicolumn{2}{c}{\textbf{WPipe}} \\
\cmidrule(lr){4-5}
\cmidrule(lr){6-7}
& &
& \textbf{Async + EF} & \textbf{Async}
& \textbf{Async + EF} & \textbf{Async} \\
\midrule
Muon
& $\mu=0.99$
& 2.582
& 2.583
& 2.590
&\textbf{2.579}
& \textbf{2.579} \\
Muon
& $\mu=0.95$
& 2.578
& 2.581
& 2.582
& \textbf{2.577}
& 2.584 \\
\midrule
AdamW
& $\beta=(0.95,0.95)$
& 2.612
& 2.640
& 2.890
& \textbf{2.631}
& 2.720 \\
\midrule
SOAP
& default
& 2.581
& 2.590
& 2.608
& \textbf{2.588}
& 2.594 \\
\bottomrule
\end{tabular}
\end{table}

The results in~\cref{tab:wpipe_360m} support this view.
For Muon and SOAP, WPipe is slightly better than the corresponding PipeDream-\texttt{2BW} one-step delayed runs, and the best asynchronous variants are often nearly indistinguishable from the synchronous baselines.
For AdamW, the same qualitative conclusion as in the main text remains: the optimizer is intrinsically much more sensitive to delayed updates, although mitigation substantially reduces the degradation.
Overall, when the system setup allows this schedule, WPipe appears to be a more attractive practical choice than PipeDream-\texttt{2BW}, while still being complementary to the optimizer-level mechanisms studied in this paper.

\color{black}

\section{Delayed Stochastic Non-Euclidean Trust-Region Theory} \label{app:general_theory}

First, we theoretically formulate the general optimization problem, following the setting covered in~\citep{kovalev2025understanding-muon-theory}:
\begin{equation} \label{eq:problem}
    \min_{x\in \mathcal{X}}[F(x) = f(x) + R(x)]
\end{equation}
where $\mathcal{X}$ is a finite-dimensional vector space endowed with the inner product $\langle\cdot,\cdot\rangle\colon \mathcal{X}\times \mathcal{X} \to \mathbb{R}$, $f(\cdot)\colon \mathcal{X} \to \mathbb{R}$ is a bounded from below and differentiable objective function, and $R(\cdot)\colon \mathcal{X} \to \mathbb{R} \cup \{+\infty\}$ is a proper convex regularizer.
For further theoretical analysis, we consider the following assumptions
\begin{assumption} [Stochastic gradient estimator] \label{ass:stoch_grad}
We assume access to an unbiased stochastic gradient estimator $g(x;\xi)$ with bounded variance, for which the following holds for all $x \in \mathcal{X}$:
\begin{align*}
&\mathbb{E}_{\xi \sim \mathcal{D}}{g(x;\xi)} = \nabla f(x) \\
&\mathbb{E}_{\xi \sim \mathcal{D}}{\|g(x;\xi) - \nabla f(x)\|_2^2} \leq \sigma^2
\end{align*}    
\end{assumption}

\begin{assumption} [Smoothness] \label{ass:smoothness}
We assume that function $f(\cdot)$ has Lipschitz gradient with respect to the considered vector space $\mathcal{X}$:
\begin{align*}
 \|\nabla f(x) - \nabla f(x')\|_* \leq L \|x-x'\| \text{ for all } x,x' \in \mathcal{X}
\end{align*}
where we denote dual norm $\|x\|_* = \sup\limits_{\|x'\| \leq 1}(\langle x, x' \rangle)$.
\end{assumption}

\begin{assumption} [Norm Equivalence] \label{ass:norm_equiv}
As Norm Equivalence itself always holds in finite dimensional spaces, we denote the positive constant, which connects the norm from $\mathcal{X}$ with Euclidean one as $\rho > 0$:
\begin{align*}
 \|x\|_* &\leq \rho \|x\|_2\text{ for all } x \in \mathcal{X}
\end{align*}
\end{assumption}
Then, we introduce \textit{Stochastic Non-Euclidean Trust-Region Gradient Method with Momentum} version with gradient delay, formulated in Algorithm \ref{alg:delayed_muon}.
\begin{algorithm}[H]
  \caption{Stochastic Non-Euclidean Trust-Region Gradient Method with Momentum with Delayed Gradients}
  \label{alg:delayed_muon}
  \begin{algorithmic}[1]
    \STATE {\bf input:} $x_0, m_0\in \mathcal{X}$
    \STATE {\bf parameters:} stepsize $\eta > 0$, momentum $\alpha \in (0,1)$, number of iterations $K \in \{1,2,\ldots\}$
    \FOR{$k=0,1,\ldots, K-1$}
    \STATE Sample $\xi_k \sim \mathcal{D}$
    \STATE $m_{k+1} = (1-\alpha)m_k + \alpha g(x_\mathrm{prev(k)};\xi_{\mathrm{prev}(k)})$
    \IF{Standard Async}
        \STATE $x_{k+1} = {\argmin \limits_{\|x-x_k\| \leq \eta}}[<m_{k+1}, x> + R(x)]$
    \ELSIF{Error-Feedback (\cref{sec:ef})}
    \STATE $\begin{aligned}[t]
    x_{k+1} = x_{\mathrm{prev}(k)} - x_k + 2 \argmin_{\|x - x_k\| \leq \eta} \bigl[ \langle m_{k+1}, x \rangle + R(x) \bigr] -\argmin_{\|x - x_{\mathrm{prev(k)}}\| \leq \eta}\left[\langle m_{\mathrm{prev}(k) + 1}, x\rangle + R(x)\right]
    \end{aligned}$
    \ENDIF
    \ENDFOR
    \STATE {\bf output:} $x_K \in \mathcal{X}$
  \end{algorithmic}
\end{algorithm}

Here, we denote $\mathrm{prev}(k) = k - \tau$ for an arbitrary delay $\tau > 0$. Additionally, we make a remark that the proposed algorithm matches Muon with $R(\cdot) \equiv 0$ and $\| \cdot\| \equiv \| \cdot \|_{op}$.

\begin{theorem} [Delayed LMO Algorithm] \label{thm:delayed_muon}
Let Assumptions \ref{ass:stoch_grad} - \ref{ass:norm_equiv} hold, and let $x_0 \in \mathrm{dom} R$ and $m_0 = g(x_0,\xi_0)$. Then the iterations of Algorithm \ref{alg:delayed_muon} satisfy the following inequality:
  \begin{align*}    
  \mathbb{E}\min_{k=1,\ldots,K}\|\nabla f(x_k) + \hat{\nabla} R_k\|_*
    &\leq
    \frac{\Delta_0}{\eta K}
    +\frac{2\rho\sigma}{\alpha K} \\
    &+2\sqrt{2\alpha}\sqrt{\sigma^2\rho^2 + 2(L\eta\tau)^2} \\
    & +\frac{7L\eta}{2}
    +\frac{2 L\eta}{\alpha},
    \end{align*}
where $\hat{\nabla} R_k \in \partial R(x_k)$, $\Delta_0 = F(x_0) - \inf_x F(x)$.
\end{theorem}
\begin{proof}
    Our proof extends the convergence framework of~\citep{kovalev2025understanding-muon-theory} to accommodate arbitrary gradient delays $\tau \geq 1$, establishing delay-dependent bounds that explicitly capture how staleness propagates through the momentum accumulation process. A detailed version can be found in Appendix \ref{proof:delayed_muon}.
\end{proof}

\begin{assumption} [Star Convexity] \label{ass:star_conv}
We assume $f(x)$ to be star-convex:
\begin{align*}
 f(\beta x^* + (1 - \beta)x) \leq \beta f(x^*) + (1-\beta)f(x)
\end{align*}
for all $x \in \mathcal{X}$, where $\beta \in (0, 1).$
\end{assumption}

\begin{theorem} [Delayed LMO Algorithm with Weight Decay] \label{thm:delayed_muon_wd}
Let Assumptions \ref{ass:stoch_grad}, \ref{ass:smoothness} \ref{ass:norm_equiv} and \ref{ass:star_conv} hold, and let $x_0 \in \mathrm{dom} R$ and $m_0 = g(x_0,\xi_0)$. Then the iterations of Algorithm \ref{alg:delayed_muon} with Weight Decay $\beta > 0$ satisfy the following inequality:
\begin{align*}
    \mathbb{E}[F(x_K) - F(x^*)] &\leq (1-\beta)^K (F(x_0) - F(x^*)) \\
    &+ 2\eta\left(\frac{\rho\sigma}{\alpha} + \frac{
  \sqrt{2\alpha}\sqrt{\sigma^2\rho^2 + 8(L\eta\tau)^2}}{\beta}\right) \\
  &+ \frac{4L\eta^2}{\beta}\left(1 + \frac{1}{\alpha}\right).
\end{align*}
where $\eta$ and $\beta$ satisfy the following:
\begin{equation}\label{eq:eta_beta}
    \eta \geq \beta \max\left\{\|x_0\|, \|x^*\|\right\}
\end{equation}
\end{theorem}
\begin{proof}
    We establish this result by integrating our delay-aware convergence framework from Theorem \ref{thm:delayed_muon} with the weight decay analysis methodology of~\citep{kovalev2025understanding-muon-theory}. A detailed version can be found in Appendix \ref{proof:delayed_muon_wd}.
\end{proof}

\textbf{Discussion}. The main change comparing the obtained estimation for delayed setup with the synchronous one is in the noise bound, which previously occurred only from the stochastic oracle noise term and now it's enlarged due to gradient delay

\section{Proofs for the General Theory} \label{app:appendix_d_ef_and_additional_proofs}

% \subsection{Additional Facts and Properties}

\subsection{Proof of Theorem \ref{thm:delayed_muon}} \label{proof:delayed_muon}

We first start with the formulating of the {\em descent lemma} from \citep{kovalev2025understanding-muon-theory}. We highlight that it stays true in the delayed setup, since the delay affects the momentum terms.
\begin{lemma}\label{lem:descent}
  Let Assumption \ref{ass:smoothness} hold, and let $x_0 \in \mathrm{dom} R$. Then the iterations of Algorithm \ref{alg:delayed_muon} satisfy the following inequality:
  \begin{equation}\label{eq:descent}
    F(x_{k+1})
    \leq
    F(x_k)
    -\eta \| \nabla f(x_{k+1}) + \hat{\nabla} R_{k+1}\|_*
    +2\eta\|\nabla f(x_{k+1}) - m_{k+1}\|_*
    + \tfrac{3}{2}L\eta^2,
  \end{equation}
  where $\hat{\nabla} R_{k+1} \in \partial R(x_{k+1})$.
\end{lemma}

Next, we establish a key lemma that bounds the momentum's tracking error.
\begin{lemma}\label{lem:momentum}
  Let Assumptions \ref{ass:smoothness}, \ref{ass:stoch_grad}, \ref{ass:norm_equiv} hold, and let $x_0 \in \mathrm{dom} R$ and $m_0 = g(x_0,\xi_0)$. Then the iterations of Algorithm \ref{alg:delayed_muon} satisfy the following inequality for $k \geq 0$:
  \begin{equation}
    \mathbb{E}\|m_{k+1} - \nabla f(x_{k})\|_* \leq  (1-\alpha)^{k+1}\rho\sigma + \sqrt{2\alpha}\sqrt{\sigma^2\rho^2 + 2(L\eta\tau)^2} + \frac{L\eta}{\alpha}.
  \end{equation}
\end{lemma}

Using Lemma \ref{lem:descent}, we obtain the following inequality:
\begin{align*}
  \min_{k=1,\ldots,K}\|\nabla f(x_k) + \hat{\nabla} R_k\|_*
  &\leq
  \frac{F(x_0) - \inf_x F(x)}{\eta K} + \frac{3L\eta}{2} + \frac{2}{K}\sum_{k=1}^{K}\|\nabla f(x_k) - m_k\|_*
  \\
  & \leq \frac{F(x_0) - \inf_x F(x)}{\eta K} + \frac{7L\eta}{2} + \frac{2}{K}\sum_{k=0}^{K-1}\|\nabla f(x_k) - m_{k+1}\|_*,
\end{align*}
Using Lemma \ref{lem:momentum}, we obtain
\begin{align*}
  \mathbb{E}\min_{k=1,\ldots,K}\|\nabla f(x_k) + \hat{\nabla} R_k\|_*
  \leq
  \frac{F(x_0) - \inf_x F(x)}{\eta K}
  +\frac{7L\eta}{2}
  +\frac{2 L\eta}{\alpha}
  +\frac{2\rho\sigma}{\alpha K}
  +2\sqrt{2\alpha}\sqrt{\sigma^2\rho^2 + 2(L\eta\tau)^2}.
\end{align*}
\qed

\subsection{Proof of Lemma \ref{lem:momentum}} \label{proof:momentum}

We can express $m_{k+1} - \nabla f(x_{k})$ as follows using $m_{k + 1}$ definition in Algorithm \ref{alg:delayed_muon}:
\begin{align*}
  m_{k+1} - \nabla f(x_k)
  &=
  (1-\alpha)m_k + \alpha g(x_{\mathrm{prev}(k)};\xi_{\mathrm{prev}(k)}) - \nabla f(x_k)
  \\&=
  (1-\alpha)(m_k - \nabla f(x_{k-1})) + \alpha(g(x_{\mathrm{prev}(k)};\xi_{\mathrm{prev}(k)}) - \nabla f(x_k))
  \\&
  +(1-\alpha)(\nabla f(x_{k-1}) - \nabla f(x_k)).
\end{align*}
This implies the following for all $k \geq 0$:
\begin{align*}
  m_{k+1} - \nabla f(x_k)
  &=
  (1-\alpha)^{k+1}(m_0 - \nabla f(x_0))
  + \sum_{i=0}^{k-1}(1-\alpha)^{k-i}(\nabla f(x_{i}) - \nabla f(x_{i+1}))
  \\&
  +\sum_{i=0}^{k}\alpha(1-\alpha)^{k-i}(g(x_{\mathrm{prev}(k)};\xi_{\mathrm{prev}(k)}) - \nabla f(x_i)).
\end{align*}
Using this, we can upper-bound $\mathbb{E}{\|m_{k+1} - \nabla f(x_k)\|_*}$ for $k\geq 0$ as follows:
\begin{align*}
  \mathbb{E}{\|m_{k+1} - \nabla f(x_k)\|_*}
  &\aleq{expand the momentum update rule and apply triangle inequality}
  (1-\alpha)^{k+1}\mathbb{E}{\|m_0 - \nabla f(x_0)\|_*}
  \\&
  + \sum_{i=0}^{k-1}(1-\alpha)^{k-i}\|\nabla f(x_{i}) - \nabla f(x_{i+1})\|_*
  \\&
  +\mathbb{E}{\left\|\sum_{i=0}^{k}\alpha(1-\alpha)^{k-i}(g(x_{\mathrm{prev}(i)},\xi_{\mathrm{prev}(i)}) - \nabla f(x_i))\right\|_*}
  \\&\aleq{use $L$-smoothness of $f$ and the constraint $\|x_{i+1} - x_i\| \leq \eta$}
  (1-\alpha)^{k+1}\mathbb{E}{\|m_0 - \nabla f(x_0)\|_*}
  + \sum_{i=0}^{k-1}(1-\alpha)^{k-i}L\eta
  \\&
  +\mathbb{E}{\left\|\sum_{i=0}^{k}\alpha(1-\alpha)^{k-i}(g(x_{\mathrm{prev}(i)},\xi_{\mathrm{prev}(i)}) - \nabla f(x_i))\right\|_*}
  \\&\aleq{use the norm compatibility property $\|\cdot\|_* \leq \rho\|\cdot\|_2$, we keep the dual norm for the second term}
  (1-\alpha)^{k+1}\rho\mathbb{E}{\|m_0 - \nabla f(x_0)\|_2}
  + \sum_{i=0}^{k-1}(1-\alpha)^{k-i}L\eta
  \\&
  +\mathbb{E}{\left\|\sum_{i=0}^{k}\alpha(1-\alpha)^{k-i}(g(x_{\mathrm{prev}(i)},\xi_{\mathrm{prev}(i)}) - \nabla f(x_i))\right\|_*}
  \\&\aleq{apply Jensen's inequality $\mathbb{E}[\|X\|] \leq \sqrt{\mathbb{E}[\|X\|^2]}$}
  (1-\alpha)^{k+1}\rho\sqrt{\mathbb{E}{\|m_0 - \nabla f(x_0)\|^2_2}}
  + \sum_{i=0}^{k-1}(1-\alpha)^{k-i}L\eta
  \\&
  +\sqrt{\mathbb{E}{\left\|\sum_{i=0}^{k}\alpha(1-\alpha)^{k-i}(g(x_{\mathrm{prev}(i)},\xi_{\mathrm{prev}(i)}) - \nabla f(x_i))\right\|^2_*}},
\end{align*}
where \annotate.

Now, we focus on the delayed-gradient term estimation. We first split the error into stochastic-noise and drift components:
\begin{align}
&\sum_{i=0}^{k} \alpha(1-\alpha)^{k-i}
\bigl(g(x_{\mathrm{prev}(i)};\xi_{\mathrm{prev}(i)})-\nabla f(x_i)\bigr) \notag\\
&\quad =
\underbrace{\sum_{i=0}^{k} \alpha(1-\alpha)^{k-i}
\bigl(g(x_{\mathrm{prev}(i)};\xi_{\mathrm{prev}(i)})
-\nabla f(x_{\mathrm{prev}(i)})\bigr)}_{S_1}
+
\underbrace{\sum_{i=0}^{k} \alpha(1-\alpha)^{k-i}
\bigl(\nabla f(x_{\mathrm{prev}(i)})-\nabla f(x_i)\bigr)}_{S_2}.
\label{eq:noise_drift_split}
\end{align}
Using $\|S_1+S_2\|_*^2 \le 2\|S_1\|_*^2+2\|S_2\|_*^2$, we bound the two terms separately.
For $S_1$, Assumption~\ref{ass:norm_equiv} gives $\mathbb{E}\|S_1\|_*^2 \le \rho^2\mathbb{E}\|S_1\|_2^2$.
When expanding the Euclidean square, the cross terms vanish in expectation by conditional unbiasedness: although the iterates depend on past samples, the noise
$g(x_j;\xi_j)-\nabla f(x_j)$ has zero conditional mean given the previous randomness, while earlier noise terms are already determined.
Thus, by Assumption~\ref{ass:stoch_grad},
\begin{align}
\mathbb{E}\|S_1\|_*^2
&\le
\rho^2
\sum_{i=0}^{k}
\alpha^2(1-\alpha)^{2(k-i)}
\mathbb{E}
\left\|
g(x_{\mathrm{prev}(i)};\xi_{\mathrm{prev}(i)})
-\nabla f(x_{\mathrm{prev}(i)})
\right\|_2^2 \notag\\
&\le
\rho^2\sigma^2
\sum_{i=0}^{k}
\alpha^2(1-\alpha)^{2(k-i)}.
\label{eq:s1_bound}
\end{align}
For $S_2$, no cancellation is used. Instead, by Assumption~\ref{ass:smoothness} and the delay bound,
\begin{align}
\|\nabla f(x_{\mathrm{prev}(i)})-\nabla f(x_i)\|_*
\le
L\|x_{\mathrm{prev}(i)}-x_i\|
\le
L\eta\tau.
\label{eq:drift_delay_bound}
\end{align}
Combining these estimates, we obtain
\begin{align}
&\mathbb{E}
\left\|
\sum_{i=0}^{k}
\alpha(1-\alpha)^{k-i}
\bigl(g(x_{\mathrm{prev}(i)};\xi_{\mathrm{prev}(i)})-\nabla f(x_i)\bigr)
\right\|_*^2 \notag\\
&\quad \le
2
\sum_{i=0}^{k}
\alpha^2(1-\alpha)^{2(k-i)}
\left(\rho^2\sigma^2 + 2(L\eta\tau)^2\right).
\label{eq:delayed_noise_drift_bound}
\end{align}

Therefore, continuing our derivations, we easily obtain
\begin{align*}
  &\leq
  (1-\alpha)^{k+1}\rho\sigma
  + \sum_{i=0}^{k-1}(1-\alpha)^{k-i}L\eta
  + 
  \sqrt{2\alpha}\sqrt{\sigma^2 \rho^2+ 2(L\eta\tau)^2}
  \\&\leq
  (1-\alpha)^{k+1}\rho\sigma + \frac{L\eta}{\alpha} + \sqrt{2\alpha}\sqrt{\sigma^2\rho^2 + 2(L\eta\tau)^2} \\
\end{align*}
\qed

% \subsection{Proof of Theorem \ref{thm:ef_delayed_muon}} 
\subsection{Proof of Error-Feedback convergence} \label{proof:ef_delayed_muon}

\begin{theorem} [EF Delayed LMO Algorithm] \label{thm:ef_delayed_muon}
Let Assumptions \ref{ass:stoch_grad} - \ref{ass:norm_equiv} hold, and let $x_0 \in \mathrm{dom} R$, while $R\equiv0$ and $m_0 = g(x_0,\xi_0)$. Then the iterations of Algorithm \ref{alg:delayed_muon}
satisfy the following inequalities:
  \begin{align*}    
  \mathbb{E}\min_{k=1,\ldots,K}\|\nabla f(x_k)\|_*
    &\leq
    \frac{\Delta_0}{\eta K}
    +\frac{(2 \tau + 2)\rho\sigma}{\alpha K} \\
    &+(2\tau + 2)\sqrt{2\alpha} \sqrt{\rho^2 \sigma^2 + 2(2\tau + 1)^2 (L\eta \tau)^2} \\
    & + \frac{3(2 \tau + 1)^2L\eta}{2} + (2\tau + 2)(2\tau + 1) L\eta
  +\frac{(2\tau + 2)(2\tau + 1) L\eta}{\alpha}
    \end{align*}
% Option 2:
%   \begin{align*}    
%   \mathbb{E}\min_{k=1,\ldots,K}\|\nabla f(x_k) + \hat{\nabla} R_k\|_*
%     &\leq
%     \frac{\Delta_0}{\eta K}
%     +\frac{2\rho\sigma}{\alpha K} \\
%     &+2\rho \sigma \sqrt{\alpha} \sqrt{1 + \textcolor{blue}{2 (\frac{L\eta \tau}{\rho\sigma })^2}} \\
%     & +\frac{7L\eta}{2}
%     +\frac{2 L\eta}{\alpha},
%     \end{align*}
% Option 3:
%   \begin{align*}    
%   \mathbb{E}\min_{k=1,\ldots,K}\|\nabla f(x_k) + \hat{\nabla} R_k\|_*
%     &\leq
%     \frac{\Delta_0}{\eta K}
%     +\frac{2\rho\sigma}{\alpha K} \\
%     &+2\rho \sigma \sqrt{\alpha} \sqrt{1 + \textcolor{blue}{2 (\frac{L\eta \tau}{\rho\sigma })^2}} \\
%     & +\frac{7L\eta}{2}
%     +\frac{2 L\eta}{\alpha},
%     \end{align*}
% where $\hat{\nabla} R_k \in \partial R(x_k)$, $\Delta_0 = F(x_0) - \inf_x F(x)$.
\end{theorem}
\begin{proof}

We start the proof with an extended version of Lemma \ref{lem:descent}.
\begin{lemma}\label{lem:option1_ef_descent}
  Let Assumption \ref{ass:smoothness} hold, and let $x_0 \in \mathrm{dom} R$ and $R \equiv 0$. Then the iterations of Error-Feedback in Algorithm \ref{alg:delayed_muon} satisfy the following inequality:
  \begin{equation}\label{eq:option1_ef_descent}
    F(x_{k+1})
    \leq
    F(x_k) + \tfrac{3}{2}(2\tau + 1)^2L\eta^2
  + (2\tau + 2)\eta \|\nabla f(x_{k+1}) - m_{k+1}\|_* -\eta\|\nabla f(x_{k+1})\|_*
  \end{equation}
\end{lemma}

Using Lemma \ref{lem:option1_ef_descent} we obtain the following estimation:

\begin{align*}
  \min_{k=1,\ldots,K}\|\nabla f(x_k)\|_*
  &\leq
  \frac{F(x_0) - \inf_x F(x)}{\eta K} + \frac{3 (2\tau + 1)^2L\eta}{2} + \frac{(2\tau + 2)}{K}\sum_{k=1}^{K}\|\nabla f(x_k) - m_k\|_*
  \\
  & \leq \frac{F(x_0) - \inf_x F(x)}{\eta K} + \frac{3(2 \tau + 1)^2L\eta}{2} + (2\tau + 2)(2\tau + 1) L\eta + \frac{(2\tau + 2)}{K}\sum_{k=0}^{K-1}\|\nabla f(x_k) - m_{k+1}\|_*,
\end{align*}
Then using the same results for Delayed momentum version from Lemma \ref{lem:momentum} and combining it with $\|x_{k + 1} - x_k\| \leq (2\tau + 1)\eta$, we obtain
\begin{align*}
  \mathbb{E}\min_{k=1,\ldots,K}\|\nabla f(x_k)\|_*
  &\leq
  \frac{F(x_0) - \inf_x F(x)}{\eta K}
  +\frac{3(2 \tau + 1)^2L\eta}{2} + (2\tau + 2)(2\tau + 1) L\eta
  +\frac{(2\tau + 2)(2\tau + 1) L\eta}{\alpha}\\
  &+\frac{(2 \tau + 2)\rho\sigma}{\alpha K}
  +(2\tau + 2)\sqrt{2} \sqrt{\alpha} \sqrt{\rho^2\sigma^2 + 2(2\tau + 1)^2 (L\eta \tau)^2}.
\end{align*}
\end{proof}

\subsubsection{Proof of Lemma \ref{lem:option1_ef_descent}} \label{proof:option1_ef_descent}

We can upper-bound $F(x_{k+1})$ as follows:
\begin{align*}
  F(x_{k+1})
  &\aeq{use the definition of function $F(x)$}
  f(x_{k+1}) + R(x_{k+1})
  \\
  &\aleq{use Assumption \ref{ass:smoothness}}
  f(x_k) + \langle \nabla f(x_k), x_{k+1} - x_k \rangle + \tfrac{1}{2}L\|x_{k+1} - x_k\|_2^2 + R(x_{k+1})
  \\
  &\aeq{algebraic manipulation: add and subtract $m_{k+1}$ and $\nabla f(x_{k+1})$}
  f(x_k) + \tfrac{1}{2}L\|x_{k+1} - x_k\|_2^2 + R(x_{k+1})
  \\
  &\quad
  + \langle m_{k+1} + \nabla f(x_{k+1}) - m_{k+1} + \nabla f(x_k) - \nabla f(x_{k+1}), x_{k+1} - x_k \rangle
  \\
  &\aleq{use the definition of dual norm}
  f(x_k) + \tfrac{1}{2}L\|x_{k+1} - x_k\|_2^2 + R(x_{k+1})
  + \langle m_{k+1}, x_{k+1} - x_k \rangle
  \\
  &\quad
  + \|x_{k+1} - x_k\| \|\nabla f(x_{k+1}) - m_{k+1}\|_*
  + \|x_{k+1} - x_k\| \|\nabla f(x_k) - \nabla f(x_{k+1})\|_*
  \\
  &\aleq{use Assumption \ref{ass:smoothness}}
  f(x_k) + \tfrac{3}{2}L\|x_{k+1} - x_k\|_2^2
  + \|x_{k+1} - x_k\| \|\nabla f(x_{k+1}) - m_{k+1}\|_*
  \\
  &\quad
  + R(x_{k+1}) + \langle m_{k+1}, x_{k+1} - x_k \rangle,
\end{align*}
where \annotate.

Then we develop an estimation for $R(x_{k+1}) + \langle m_{k+1}, x_{k+1} - x_k \rangle$.

Using Error-Feedback from Algorithm \ref{alg:delayed_muon},

% $x_{k + 1} = x_{\mathrm{prev}(k)} - x_k + 2\arg\min \limits_{\|x - x_k\| \leq \eta} [\langle m_{k + 1}, x\rangle + R(x)] -\arg\min\limits_{\|x - x_p\| \leq \eta}[\langle m_{\mathrm{prev}(k) + 1}, x\rangle + R(x)]$

$x_{k + 1} = \tau x_{\mathrm{prev}(k)} - \tau x_k + (\tau + 1) \arg\min \limits_{\|x - x_k\| \leq \eta} [\langle m_{k + 1}, x\rangle + R(x)] -\tau \arg\min\limits_{\|x - x_p\| \leq \eta}[\langle m_{\mathrm{prev}(k) + 1}, x\rangle + R(x)]$

% Thus, we obtain
% \begin{align*}
%     &R(x_{k + 1}) + \left\langle m_{k + 1}, x_{\mathrm{prev}(k)} - 2x_k + 2\argmin_{\|x - x_k\| \leq \eta} \left[\langle m_{k + 1}, x\rangle + R(x)\right] - \argmin_{\|x - x_p\| \leq \eta}\left[\langle m_{\mathrm{prev}(k) + 1}, x\rangle + R(x)\right]\right\rangle \\
%     & = R(x_{k + 1}) + \left\langle m_{k + 1}, x_{\mathrm{prev}(k)} - \argmin_{\|x - x_p\| \leq \eta}\left[\langle m_{\mathrm{prev}(k) + 1}, x\rangle + R(x)\right]\right\rangle - 2\left\langle m_{k + 1}, x_k - \argmin_{\|x - x_k\| \leq \eta} \left[\langle m_{k + 1}, x\rangle + R(x)\right]\right\rangle \\
%     & = \left(R(x_{k + 1}) + \left\langle m_{k + 1}, \argmin_{\|x - x_k\| \leq \eta} \left[\langle m_{k + 1}, x\rangle + R(x)\right] - x_k\right\rangle\right) + \left(R(x_{k + 1}) + \left\langle m_{k + 1}, \argmin_{\|x - x_k\| \leq \eta} \left[\langle m_{k + 1}, x\rangle + R(x)\right] - x_k\right\rangle\right) \\
%     &- \left( R(x_{k + 1}) - \left\langle m_{k + 1}, \argmin_{\|x - x_p\| \leq \eta}\left[\langle m_{\mathrm{prev}(k) + 1}, x\rangle + R(x)\right] - x_{\mathrm{prev}(k)}\right\rangle\right)
% \end{align*}

Thus, we obtain
\begin{align*}
    &R(x_{k + 1}) + \left\langle m_{k + 1}, \tau x_{\mathrm{prev}(k)} - (\tau + 1) x_k + (\tau + 1) \arg\min \limits_{\|x - x_k\| \leq \eta} [\langle m_{k + 1}, x\rangle + R(x)] -\tau \arg\min\limits_{\|x - x_p\| \leq \eta}[\langle m_{\mathrm{prev}(k) + 1}, x\rangle + R(x)]\right\rangle \\
    & = R(x_{k + 1}) + \tau\left\langle m_{k + 1}, x_{\mathrm{prev}(k)} - \argmin_{\|x - x_p\| \leq \eta}\left[\langle m_{\mathrm{prev}(k) + 1}, x\rangle + R(x)\right]\right\rangle - (\tau + 1)\left\langle m_{k + 1}, x_k - \argmin_{\|x - x_k\| \leq \eta} \left[\langle m_{k + 1}, x\rangle + R(x)\right]\right\rangle
\end{align*}

% \textcolor{red}{==== old below ====}

% Then we apply Lemma 3 from \citep{kovalev2025understanding-muon-theory} for first two terms:
% \begin{align*}
%     &\leq 2\cdot\left(R(x_k) - \eta \| m_{k + 1} + \hat{\nabla} R_{k + 1}\|_*\right) \\
%     & - \left( R(x_{k + 1}) - \langle m_{k + 1}, \arg \min_{\|x - x_p\| \leq \eta}(\langle m_{\mathrm{prev}(k) + 1}, x\rangle + R(x)) - x_{\mathrm{prev}(k)}\rangle\right) 
% \end{align*}

Then we apply Lemma 3 from \citep{kovalev2025understanding-muon-theory} to the first two terms:
\begin{align*}
    &\leq (\tau + 1)\cdot\left(R(x_k) - \eta \| m_{k + 1} + \hat{\nabla} R_{k + 1}\|_*\right) \\
    & - \tau\left( R(x_{k + 1}) - \langle m_{k + 1}, \arg \min_{\|x - x_p\| \leq \eta}(\langle m_{\mathrm{prev}(k) + 1}, x\rangle + R(x)) - x_{\mathrm{prev}(k)}\rangle\right) 
\end{align*}

% Now we apply the $R(x) \equiv 0$, so we now can apply Lemma 3 for the third term, since $0 + \langle m_{k + 1}, \arg\min\limits_{\|x-x_k\| \leq \eta} (\langle m_{k + 1}, x\rangle + 0) - x_k\rangle \leq 0 - \eta \|m_{k + 1} + 0\|_*$, which is equivalent to $\langle m_{k + 1}, \arg\min\limits_{\|x-x_k\| \leq \eta} (\langle m_{k + 1}, x\rangle) - x_k\rangle \leq -\eta \|m_{k + 1}\|_*$

% Next, we apply $R \equiv 0$ and Cauchy-Schwarz inequality to the third term and obtain
% \begin{align*}
%     \leq -2\eta \|m_{k + 1}\|_* + \eta \|m_{k + 1}\|_* = -\eta \|m_{k + 1}\|_*
% \end{align*}

Next, we apply $R \equiv 0$ and Cauchy-Schwarz inequality to the third term and obtain:
\begin{align*}
    \leq -(\tau + 1)\eta \|m_{k + 1}\|_* + \tau\eta \|m_{k + 1}\|_* = -\eta \|m_{k + 1}\|_*
\end{align*}

Then we estimate $\|x_{k+1} - x_k\|$ using the EF update from Algorithm \ref{alg:delayed_muon}.
\begin{align*}
    \|x_{k + 1} - x_k\| &= \left\|(\tau + 1)x_k - \tau x_{\mathrm{prev}(k)} - (\tau + 1)\argmin_{\|x - x_k\| \leq \eta}\left[\langle m_{k+1}, x\rangle + R(x)\right] + \tau\argmin_{\|x - x_{\mathrm{prev}(k)}\| \leq \eta}\left[\langle m_{\mathrm{prev}(k)+1}, x\rangle + R(x)\right]\right\| \\
    &\leq (\tau + 1) \left\|x_k - \argmin_{\|x - x_k\| \leq \eta}\left[\langle m_{k+1}, x\rangle + R(x)\right]\right\| + \tau\left\|x_{\mathrm{prev}(k)} - \argmin_{\|x - x_{\mathrm{prev}(k)}\| \leq \eta}\left[\langle m_{\mathrm{prev}(k)+1}, x\rangle + R(x)\right]\right\| \\
    & \leq (2\tau + 1)\eta 
\end{align*}

Continuing estimation we obtain
\begin{align*}
    &\leq f(x_k) + \tfrac{3}{2}L\|x_{k+1} - x_k\|_2^2
  + \|x_{k+1} - x_k\| \|\nabla f(x_{k+1}) - m_{k+1}\|_* \\
  &\quad
  -\eta\|m_{k + 1}\|_* \\
  & \leq f(x_k) + \tfrac{3}{2}(2\tau + 1)^2L\eta^2
  + (2\tau + 1)\eta \|\nabla f(x_{k+1}) - m_{k+1}\|_* \\
  &\quad
  -\eta\|m_{k + 1}\|_* \\
  &= F(x_k) + \tfrac{3}{2}(2\tau + 1)^2L\eta^2
  + (2\tau + 1)\eta \|\nabla f(x_{k+1}) - m_{k+1}\|_* -\eta\|m_{k + 1}\|_* \\
  & \leq F(x_k) + \tfrac{3}{2}(2\tau + 1)^2L\eta^2
  + (2\tau + 2)\eta \|\nabla f(x_{k+1}) - m_{k+1}\|_* -\eta\|\nabla f(x_{k+1})\|_*
\end{align*}
\qed

\subsection{Proof of Theorem \ref{thm:delayed_muon_wd}} \label{proof:delayed_muon_wd}

In this proof, we are going to use Lemma \ref{lem:x} from \citep{kovalev2025understanding-muon-theory}.

\begin{lemma}\label{lem:x}
  Under the conditions of Equation \ref{eq:eta_beta}, let $x \in \mathcal{X}$ be defined as follows:
  \begin{equation}\label{eq:x_def}
    x = \beta x^* + (1-\beta)x_k.
  \end{equation}
  Then, the following inequalities hold:
  \begin{equation}
    \|x-(1-\beta)x_k\| \leq \eta,
    \quad
    \|x-x_k\| \leq 2\eta,
    \quad
    \|x-x_{k+1}\| \leq 2\eta,
    \quad
    \|x_{k+1}-x_k\| \leq 2\eta.
  \end{equation}
\end{lemma}

Additionally, we obtain the following Lemma \ref{lem:momentum_decay}. 

\begin{lemma}\label{lem:momentum_decay}
  Let Assumptions \ref{ass:stoch_grad} - \ref{ass:norm_equiv} hold, and let $x_0 \in \mathrm{dom} R$ and $m_0 = g(x_0,\xi_0)$. Then the iterations of Algorithm \ref{alg:delayed_muon} with Weight Decay satisfy the following inequality for $k \geq 0$:
  \begin{equation}
    \mathbb{E}[\|m_{k+1} - \nabla f(x_{k})\|_*] \leq  (1-\alpha)^{k+1}\rho\sigma +
  \sqrt{2\alpha} \sqrt{\rho^2 \sigma^2 + 8 (L\eta \tau)^2} + \frac{2L\eta}{\alpha}.
  \end{equation}
\end{lemma}
The proof is similar to Section \ref{proof:momentum}, with  $\|x_{k+1} - x_k\| \leq 2\eta$.

The proof for Theorem \ref{thm:delayed_muon_wd} is similar to Theorem 4 from \citep{kovalev2025understanding-muon-theory}; we obtain
\begin{align*}
  \mathbb{E}[F(x_{k+1}) - F(x^*)]
  &\leq
  (1-\beta) \mathbb{E}[F(x_k) - F(x^*)]
  +2\eta\rho\sigma(1-\alpha)^k + 2\eta \sqrt{2\alpha} \sqrt{\rho^2\sigma^2 + 8 (L\eta \tau)^2}
  \\&
  +4L\eta^2 + \frac{4L\eta^2}{\alpha},
\end{align*}
which implies the following inequality:
\begin{align*}
  \mathbb{E}[F(x_K) - F(x^*)] \leq (1-\beta)^K (F(x_0) - F(x^*)) + 2\eta\left(\frac{\rho\sigma}{\alpha} + \frac{\sqrt{2\alpha}\sqrt{\rho^2\sigma^2 + 8 (L\eta \tau)^2} }{\beta}\right) + \frac{4L\eta^2}{\beta}\left(1 + \frac{1}{\alpha}\right).
\end{align*}
\qed

\section{Experimental Setup} \label{app:experimental_setup}

We focus on the standard next-token prediction task, training decoder-only models based on the SmolLM2 architecture~\citep{allal2025smollm2} with parameter counts of 135M and 360M using the Fineweb-Edu dataset~\citep{penedo2024fineweb}.
Unless explicitly stated otherwise, all training runs adhere to a Chinchilla compute-optimal token-to-parameter ratio of 20:1~\citep{hoffmann2022training-chinchilla}.

\subsection{Hyperparameters and Training Details}\label{app:optimizers_and_hyperparameters}

To ensure a rigorous and fair comparison across different optimization algorithms, we adopted a systematic approach to hyperparameter tuning, prioritizing the stability and optimality of the synchronous baselines.

\textbf{Optimizer Tuning.}
We began by establishing strong baselines for the SmolLM-2 models~\citep{allal2025smollm2} (135M and 360M) using AdamW~\citep{loshchilov2017decoupled-adamw}.
We performed a grid search over learning rates and weight decay values, using a multiplicative step of 2 (uniform grid in log scale) for the learning rate and testing four distinct weight decay values.
After verifying that a weight decay of $0.1$ consistently yielded optimal results, we fixed this value for the remainder of the study.
The resulting optimal learning rates were found to be $4\text{e-}3$ for the 135M model and $2\text{e-}3$ for the 360M model.

For all other optimizers (with the exception of Lion~\citep{chen2023symbolic-lion}, which operates on a distinct scale), we tuned the learning rate within a narrow range surrounding the optimal AdamW values.
This approach leverages prior findings~\cite{wen2025fantastic-optimizers-benchmarking-1, semenov2025benchmarking-optimizers-benchmarking-1} indicating that optimal hyperparameters for many modern optimizers tend to cluster in similar regions.
Optimal weight decay for Lion was 0.5 and learning rate was approximately $5\text{e-}4$, which aligns with results from~\citet{wen2025fantastic-optimizers-benchmarking-1}.
Crucially, we always compared synchronous and asynchronous runs using \textit{identical} hyperparameter configurations.

\textbf{Batch Size Selection.}
We aimed to approximate optimal batch sizes for valid scaling laws.
For the 135M and 360M models, we selected global batch sizes based on the average of predictions derived from~\citet{li2025predictable-scale-part1} and~\citet{bi2024deepseek1}.
Although the context length was set to 1024, the use of padding for the FineWeb dataset resulted in an average sequence length of approximately $\sim$700 tokens.
Consequently, a global batch size of 256 for the 135M model and 512 for the 360M model resulted in effective batch sizes of approximately 180K and 360K tokens, respectively.

For the larger 2B and 10B MoE models, we utilized batch sizes slightly larger than theoretical optima to maximize GPU utilization. For 2B training we used 1M, 1.5M and 2.25M tokens batch sizes for 50B, 100B, and 200B respectively, and for 10B model training on 200B tokens we used 4M batch size.
It is important to note that this regime theoretically disadvantages Async PP: larger batch sizes imply fewer total optimization steps for the same token budget, leaving the model with fewer opportunities to recover from the initial errors caused by gradient delays~\cref{fig:10b}.
Thus, the robustness observed in our large-scale experiments is likely a conservative estimate.

\textbf{Other Settings.}
We utilized a cosine decay learning rate schedule with a minimum learning rate of $0.1 \times \text{max\_lr}$.
We used 10\% of Chinchilla tokens for learning rate warmup.
We also used gradient clipping with the standard clipping value of $1.0$.

\subsection{Model architectures} \label{app:models_architectures}

In our experiments, we utilize four distinct model architectures: two dense models from the SmolLM-2 family (135M and 360M parameters) and two custom sparse Mixture-of-Experts (MoE) models with 2B and 10B total parameters.

\textbf{SmolLM-2 Models.}
We employ the SmolLM-2~\citep{allal2025smollm2} 135M and 360M architectures as our dense baselines. Built upon the standard Llama architecture, these models incorporate Grouped Query Attention~\citep{ainslie2023gqa}, RMSNorm~\citep{rmsnorm} and SwiGLU~\citep{shazeer2020glu}.
Both models are trained with a context length of 1,024 tokens and a vocabulary size of 49,152.

\textbf{Custom MoE Models.}
To rigorously validate our hypotheses at scale, we trained two custom MoE models. Both models utilize a tokenizer with a vocabulary size of 128k and support a context length of 8,192 tokens.
\begin{itemize}
    \item \textbf{2B MoE (0.5B Active):} This model features 16 layers with a hidden size of 1024. It uses 16 query heads and 4 key-value heads. The routing mechanism involves 64 experts with top-8 gating. Despite the total parameter count of $\approx$2B, the active parameter count per token is approximately 500M.
    \item \textbf{10B MoE (0.65B Active):} This model employs a hybrid architecture inspired by~\citet{qwen3next}, incorporating Gated DeltaNet~\citep{yang2024gated-delta-net} layers. It consists of 24 layers, configured such that every 4th layer uses Full Attention while the remaining layers utilize linear attention. The model scales to 512 experts with top-10 gating and utilizes a Shared Expert and Shared Expert Trainable Weight mechanism. While the total parameter count is $\approx$10B, the highly sparse architecture maintains an efficient active parameter count of only $\approx$0.65B during inference.
\end{itemize}

\FloatBarrier
\section{Memory and Runtime Overhead}

\subsection{Memory overhead}\label{app:memory_overhead}
PipeDream-\texttt{2BW} and Error Feedback each introduce one additional parameter-sized state.
For PipeDream-\texttt{2BW}, this state is the extra parameter version required to preserve forward--backward consistency under delayed updates.
For Error Feedback, it is the residual buffer used to accumulate the correction.
In modern large-scale training setups, however, this overhead is applied to the local model shard stored on each GPU, not to the full model.
As a result, the per-GPU cost is typically small because model states are distributed across pipeline, tensor, expert, and data-parallel dimensions.

We first consider DeepSeek-V3~\citep{deepseek3}, a 681B-parameter MoE model trained on 2048 GPUs with 16-way Pipeline Parallelism (PP) and 64-way Expert Parallelism (EP).
The remaining data-parallel degree is therefore $2048 / (16 \cdot 64) = 2$.
DeepSeek-V3 has 61 hidden layers in total, so each pipeline stage stores at most four layers.
For a MoE layer, each GPU stores all non-expert components assigned to its pipeline stage, the shared expert, and only $256 / 64 = 4$ routed experts due to expert parallelism.
This gives approximately $0.409$B parameters per MoE layer per GPU, or about $4 \cdot 0.409\text{B} \approx 1.6$B parameters per GPU for four MoE layers.
Since DeepSeek-V3 uses ZeRO-1 with data-parallel degree $2$, an additional FP32 master-weight copy is sharded across two data-parallel ranks, giving an estimated cost of
\[
    1.6\text{B} \cdot 4\text{ bytes} / 2 \approx 3.2\text{ GB}
\]
per GPU.
On 80GB GPUs, this overhead is not prohibitive.

As a second example, consider LLaMA~3 405B~\citep{grattafiori2024llama3}, which uses 8-way tensor parallelism and 16-way pipeline parallelism.
Before FSDP sharding, the resident parameter count per GPU is approximately
\[
    405\text{B} / (8 \cdot 16) \approx 3.16\text{B}.
\]
With an effective FSDP sharding factor of about $128$ for optimizer and master-weight states, one additional FP32 sharded state costs
\[
    3.16\text{B} \cdot 4\text{ bytes} / 128 \approx 0.10\text{ GB}
\]
per GPU.
In this setting, the additional memory overhead is therefore negligible.

The same conclusion holds in our largest experiment: a 10B-parameter MoE model trained on 64 GPUs.
Since our setup partitions all hidden layers across devices and does not replicate sublayers across GPUs, each GPU stores at most about 200M master-weight parameters, allowing a small margin for embeddings and the language-model head.
Thus, one additional FP32 parameter-sized state costs approximately
\[
    200\text{M} \cdot 4\text{ bytes} = 800\text{ MB}.
\]
In practice, the total additional memory cost of Async PP with Error Feedback was below 1.5GB per GPU, which is less than $2\%$ of an 80GB GPU.
These estimates suggest that as long as the number of additional parameter-sized states remains a small constant and does not grow with pipeline depth, the memory overhead of PipeDream-\texttt{2BW} and Error Feedback is not a major obstacle in realistic LLM training scenarios.

\subsection{Runtime overhead}\label{app:runtime_overhead}
The main runtime advantage of Async PP comes from eliminating pipeline bubbles.
We estimate this effect using the bubble model from the DeepSeek-V3 technical report~\citep{deepseek3}.
Let $P$ denote the pipeline depth, $M$ the number of micro-batches, $F$ the forward time for one micro-batch chunk, $B$ the backward time, $W$ the backward-for-weights component, and $F\&B$ the execution time of an overlapped forward/backward pair.
We define the bubble-to-compute ratio as
\[
    \rho = \frac{T_{\mathrm{bubble}}}{T_{\mathrm{compute}}},
    \qquad
    T_{\mathrm{compute}} = M(F+B).
\]
Async PP has no pipeline bubbles in this schedule-level model, so $\rho_{\mathrm{async}} = 0$.
Thus, $1+\rho$ can be interpreted as the slowdown of a synchronous PP schedule relative to the async ideal.

For standard synchronous schedules, the DeepSeek-V3 technical report gives the following bubble ratios:
\[
    \rho_{\mathrm{1F1B}} = \frac{P-1}{M},
\]
\[
    \rho_{\mathrm{ZB1P}} = \frac{(P-1)(F+B-2W)}{M(F+B)},
\]
\[
    \rho_{\mathrm{DualPipe}} =
    \frac{\left(\frac{P}{2}-1\right)(F\&B+B-3W)}{M(F+B)}.
\]
To keep the analysis simple, we assume that communication is fully overlapped, so $F\&B \approx F+B$.
We consider two standard zero-order compute models.
The first assumes $B=2F$ and $W=F$, corresponding to matmul-only accounting.
The second assumes $B=3F$ and $W=F$, roughly accounting for activation recomputation on the input-gradient path.
We report results for $P=16$, a pipeline depth used in modern large-scale training setups such as DeepSeek-V3~\citep{deepseek3} and LLaMA~3~\citep{grattafiori2024llama3}.

\begin{table}[h]
\centering
\caption{Slowdown factors $1+\rho$ relative to the async ideal under the DeepSeek-V3 bubble model with $P=16$.}
\label{tab:runtime_overhead}
\begin{tabular}{lccc}
\toprule
Schedule & $M=16$ & $M=32$ & $M=64$ \\
\midrule
Async PP / PipeDream-\texttt{2BW} & 1.000 & 1.000 & 1.000 \\
1F1B, $B=2F$ & 1.938 & 1.469 & 1.234 \\
ZB1P, $B=2F, W=F$ & 1.313 & 1.156 & 1.078 \\
DualPipe, $B=2F, W=F$ & 1.292 & 1.146 & 1.073 \\
1F1B, $B=3F$ & 1.938 & 1.469 & 1.234 \\
ZB1P, $B=3F, W=F$ & 1.469 & 1.234 & 1.117 \\
DualPipe, $B=3F, W=F$ & 1.438 & 1.219 & 1.109 \\
\bottomrule
\end{tabular}
\end{table}

Under these standard bubble models, synchronous PP can still incur substantial schedule-level overhead at practical micro-batch counts, whereas Async PP removes this bubble term entirely.
This analysis is not a substitute for end-to-end wall-clock measurements, which depend on implementation details, communication overlap, and hardware.
Nevertheless, it provides a simple estimate of the runtime advantage that can be expected from eliminating synchronous pipeline bubbles.

\end{document}